\documentclass[twoside,11pt]{article}

\usepackage{blindtext}

\usepackage{jmlr2e}

\usepackage{mathrsfs} 
\usepackage{amsfonts}
\usepackage{subfigure}
\usepackage{amsmath}
\usepackage{amssymb}
\usepackage{mathtools}

\usepackage{multirow}
\usepackage{arydshln}
\usepackage{makecell}
\usepackage{bm}
\usepackage{graphicx}

\usepackage{hyperref}       
\usepackage{url}            
\usepackage{booktabs}       
\usepackage{amsfonts}       
\usepackage{nicefrac}       
\usepackage{microtype}      
\usepackage{xcolor}         

\usepackage{graphicx}

\usepackage{lastpage}

\ShortHeadings{A Unified Information-Theoretic Framework for Meta-Learning Generalization}{Wen Wen, Tieliang Gong, Yuxin Dong, Yunjiao Zhang, Zeyu Gao, and Yong-Jin Liu}
\firstpageno{1}

\begin{document}

\newtheorem{restatetheoreminner}{\bf Theorem}
\newenvironment{restatetheorem}[1]{%
  \renewcommand\therestatetheoreminner{#1}%
  \restatetheoreminner
}{\endrestatetheoreminner}

\newtheorem{restatepropositioninner}{\bf Proposition}
\newenvironment{restateproposition}[1]{%
  \renewcommand\therestatepropositioninner{#1}%
  \restatepropositioninner
}{\endrestatepropositioninner}


\title{
A Unified Information-Theoretic Framework for Meta-Learning Generalization}

\author{\name Wen Wen \email wen190329@gmail.com \\
       \addr School of Computer Science of Technology\\
         Xi'an Jiaotong University\\
       Xi'an, Shanxi, China
       \AND
       \name Tieliang Gong \thanks{Corresponding author}  \email adidasgtl@gmail.com \\
       \addr School of Computer Science of Technology\\
         Xi'an Jiaotong University\\
       Xi'an, Shanxi, China
       \AND
       \name Yuxin Dong \email yxdong9805@gmail.com \\
       \addr School of Biomedical Informatics\\
         Ohio State University\\
       Columbus, OH 43210, USA
       \AND
         \name Zeyu Gao \email betpotti@gmail.com \\
       \addr Department of Oncology\\
         University of Cambridge\\
       Cambridge, CB2 1TN, UK
       \AND
       \name Yong-Jin Liu \email liuyongjin@tsinghua.edu.cn \\
       \addr Department of Computer Science and Technology\\
         Tsinghua University\\
       Beijing, China
       }

\editor{My editor}

\maketitle

\begin{abstract}
In recent years, information-theoretic generalization bounds have gained increasing attention for analyzing the generalization capabilities of meta-learning algorithms. However, existing results are confined to two-step bounds, failing to provide a sharper characterization of the meta-generalization gap that simultaneously accounts for environment-level and task-level dependencies. This paper addresses this fundamental limitation by developing a unified information-theoretic framework using a single-step derivation. The resulting meta-generalization bounds, expressed in terms of diverse information measures, exhibit substantial advantages over previous work, particularly in terms of tightness, scaling behavior associated with sampled tasks and samples per task, and computational tractability.  Furthermore, through gradient covariance analysis, we provide new theoretical insights into the generalization properties of two classes of noisy and iterative meta-learning algorithms, where the meta-learner uses either the entire meta-training data (e.g., Reptile), or separate training and test data within the task (e.g., model agnostic meta-learning (MAML)). Numerical results validate the effectiveness of the derived bounds in capturing the generalization dynamics of meta-learning.  
\end{abstract}


\section{Introduction}

\label{intro}

Meta-learning, also known as learning to learn, has emerged as a prevalent paradigm for rapidly adapting to new tasks by leveraging prior knowledge extracted from multiple inherently relevant tasks \citep{hospedales2021meta,hu2023learning,lake2023human}. Concretely, the meta-learner has access to training data from tasks observed within a common task environment to learn meta-hyperparameters, which can then fine-tune task-specific parameters for improving performance on an unseen task. 
Significant efforts have recently been dedicated to optimizing their empirical behavior on the meta-training data, with the goal of achieving minimal meta-training error \citep{zhang2022distributed,abbas2022sharp,wang2023improving}. However, lower meta-training errors do not necessarily ensure stronger performance on previously unseen tasks. It is therefore essential to establish upper bounds on the meta-generalization gap—the difference between the population and empirical risks of the meta-hypothesis—to ensure reliable generalization performance of meta-learning.

Information-theoretic generalization bounds \citep{xu2017information,steinke2020reasoning,hellstrom2022new,wang2023tighter}, expressed by distribution- and algorithm-dependent information measures, have attracted widespread attention for their capability of precisely characterizing the generalization properties of learning algorithms. However, the multi-level dependencies among meta-parameters, task-specific parameters, and training samples hinder the direct extension of these bounds from single-task learning to meta-learning, and render classical theoretical tools ineffective in analyzing the interactions among multiple distributions. This has led to a proliferation of work aimed at upper-bounding the meta-generalization gap by decomposing it into task-level and environment-level error terms via the auxiliary loss, and then separately controlling each term within the conventional learning framework \citep{jose2021information,rezazadeh2021conditional,hellstrom2022evaluated}. Unfortunately, the resulting `two-step' bounds typically yield an undesirable scaling rate of $\mathcal{O}(\max\{\frac{1}{\sqrt{n}}, \frac{1}{\sqrt{m}}\})$ in terms of the number $n$ of tasks and the number $m$ of samples per task, while relying on a strict assumption that the meta-parameters and task-specific parameters are mutually independent. Recent substantial advancements have been made by \citep{chen2021generalization, bu2023generalization}, who immediately bound the original meta-generalization gap through the mutual information (MI) between the meta-training data and the task-specific as well as the meta parameters. While such `single-step' bounds achieve an improved scaling rate of $\mathcal{O}(\frac{1}{\sqrt{nm}})$, they also implicitly assume that the meta-parameter distribution is independent of the sample distribution, which is relatively stringent and unrealistic in practice. Building on the conditional mutual information (CMI) framework \citep{steinke2020reasoning}, \cite{hellstrom2022evaluated} establish improved upper bounds by introducing meta-supersample variables, thereby circumventing the assumption of distributional independence.  Notably, all the aforementioned work can be further enhanced and refined within the proposed information-theoretical framework.

In this paper, we provide a unified information-theoretic generalization analysis for meta-learning by leveraging the multi-distribution extension of the Donsker-Varadhan formula and introducing a novel meta-supersample regime with fewer supersample variables, along with a comparison to existing results presented in Table \ref{table1}. The main contributions are summarized as follows:

\begin{table*}[t]
  \centering
  \vskip 0.15in
  \linespread{1.2}
  \begin{small}
      \resizebox{\textwidth}{!}{
  \begin{tabular}{ccccr}
  \toprule
  Information Measure & Related Work   & Analysis Path & Convergence Rate\\
  \midrule
  \multirow{3}{*}{Input-output MI} & \citet{jose2021information} &  Two-step derivation & $\mathcal{O}\big(\frac{1}{\sqrt{n}}+\frac{1}{\sqrt{m}}\big) $ \\
  & \citet{chen2021generalization} & \multirow{2}{*}{Single-step derivation} & \multirow{2}{*}{$\mathcal{O}\big(\frac{1}{\sqrt{nm}}\big) $} \\
  & \citet{bu2023generalization} &  &  \\
 \textbf{Samplewise Input-output MI} & \textbf{Ours (Thm. \ref{theorem3.1})} & \textbf{\makecell[c]{Single-step derivation\\ Random subset}} & $\mathcal{O}\big(\mathbf{\frac{1}{\sqrt{nm}}}\big) $ \\
  \cdashline{1-4} \\
  \multirow{4}{*}{CMI} & \citet{rezazadeh2021conditional} & \makecell[c]{Two-step derivation\\ $(n+nm)$ sample marks} & $\mathcal{O}(\frac{1}{\sqrt{n}}+ \frac{1}{\sqrt{nm}}) $ \\
  & \textbf{Ours (Thm. \ref{theorem3.5})} & \multirow{3}{*}{\textbf{\makecell[c]{Single-step derivation\\ Random subset\\ $\mathbf{(n+m)}$ sample marks}}}  & $\mathcal{O}(\mathbf{\frac{1}{\sqrt{nm}}}) $ \\
  & \textbf{Ours (Thm. \ref{theorem3.6})} &  & $\ast\mathcal{O}(\mathbf{\frac{1}{nm}}) $ \\
  & \textbf{Ours (Thm. \ref{theorem3.8})} &  & $\star\mathcal{O}(\mathbf{\frac{1}{nm}}) $ \\
   \cdashline{1-4} \\
  \multirow{3}{*}{e-CMI} &  \citet{hellstrom2022evaluated} & \makecell[c]{Single-step derivation\\ $(n+nm)$ sample marks} &  $\mathcal{O}(\frac{1}{\sqrt{nm}}) $  \\
  & \textbf{Ours (Thm. \ref{theorem4})} & \multirow{2}{*}{\textbf{\makecell[c]{Single-step derivation\\ $\mathbf{(n+m)}$ sample marks}}}  & $\mathcal{O}(\frac{1}{\sqrt{nm}}) $  \\
  & \textbf{Ours (Thm. \ref{theorem6})} &  & $\ast\mathcal{O}(\mathbf{\frac{1}{nm}}) $  \\
  \cdashline{1-4} \\
 \textbf{Loss difference-based MI/CMI} & \textbf{Ours (Thms. \ref{theorem7}, \ref{theorem8})}   &  \multirow{2}{*}{\textbf{\makecell[c]{Single-step derivation\\ $\mathbf{(n+m)}$ sample marks}}} & $\mathcal{O}(\frac{1}{\sqrt{nm}}) $ \\
 \textbf{Single-loss MI} &  \textbf{Ours (Thms. \ref{theorem10}, \ref{theorem12})}  & & $ \star\mathcal{O}(\mathbf{\frac{1}{nm}})$ \\
  \bottomrule
  \end{tabular}}
  \label{table1}
  \caption{Existing information-theoretic generalization bounds for meta-learning (MI-Mutual Information; CMI-Conditional Mutual Information; $\ast$-binary KL divergence bound; $\star$-fast-rate bound).}
\end{small}
  \vskip -0.1in
  \end{table*}

\begin{itemize}
    \item We establish refined meta-generalization bounds through a single-step derivation, without dependence on distributional independence. Our bounds, expressed by the samplewise MI that simultaneously integrate a meta-parameter, task-specific parameter subsets, and sample subsets, not only tighten existing bounds but also yield a favorable scaling rate of $\mathcal{O}(\frac{1}{\sqrt{nm}})$, where $n$ and $m$ represent the number of tasks and the number of samples per task, respectively. Additionally, we derive more precise CMI bounds that involve fewer supersample variables compared to previous work under the novel meta-supersample setting.
    
    \item We develop novel loss difference-based meta-generalization bounds through both unconditional and conditional information measures. The derived bounds exclusively involve two one-dimensional variables, making them computationally tractable and tighter than previous MI- and CMI-based meta-generalization bounds. Furthermore, we derive a novel fast-rate bound by employing the weighted generalization error, achieving a faster convergence order of $\mathcal{O}(\frac{1}{nm})$ in the interpolating regime. This fast-rate result is further generalized to non-interpolating settings through the development of variance-based bounds. 
    \item Our theoretical framework generally applies to a broad range of meta-learning paradigms, including those using the entire meta-training data as well as those employing within-task train-test partitions. In particular, we derive enhanced algorithm-dependent bounds for both learning paradigms using the conditional gradient variance, exhibiting substantial advancement over bounds based on the gradient norm and gradient variance. Our results provide novel theoretical insights into the learning dynamics of noisy, iterative-based meta-learning algorithms.
    \item Numerical results on synthetic and real-world datasets demonstrate the closeness between the generalization error and the derived bounds.
\end{itemize}

\section{Related Work}
\subsection{Learning Theory for Meta-learning} Theoretical analysis of meta-learning traces back to the seminal work of \citep{baxter2000model}, which formally introduces the notion of the task environment and establishes uniform convergence bounds via the lens of covering numbers. Subsequent research has enriched the generalization guarantees of meta-learning through diverse learning theoretic techniques, including uniform convergence analysis associated with hypothesis space capacity \citep{tripuraneni2021provable,guan2022task,aliakbarpour2024metalearning}, algorithmic stability analysis \citep{farid2021generalization,fallah2021generalization,chen2024stability}, PAC-Bayes framework \citep{rothfuss2021pacoh,rezazadeh2022unified,zakerinia2024more}, information-theoretic analysis \citep{jose2021information,chen2021generalization,rezazadeh2021conditional,hellstrom2022evaluated,bu2023generalization}, etc. Another line of work has investigated the generalization properties from the perspective of task similarity, without relying on task distribution assumptions \citep{tripuraneni2020theory,guan2022task}.

\subsection{Information-theoretic Bounds}  Information-theoretic metrics are first introduced in \citep{xu2017information,russo2019much} to characterize the average generalization error of learning algorithms in terms of the mutual information between the output hypothesis and the input data. This approach has shown a significant advantage in depicting the dynamics of iterative and noisy learning algorithms, exemplified by its application in SGLD \citep{negrea2019information,wang2021analyzing} and SGD \citep{neu2021information,wang2021generalization}. This framework has been subsequently expanded and enhanced through diverse techniques, including conditioning \citep{hafez2020conditioning}, the random subsets \citep{bu2020tightening,rodriguez2021random}, and conditional information measures \citep{steinke2020reasoning,haghifam2020sharpened}. A remarkable advancement is made by \citep{harutyunyan2021information}, who establish generalization bounds in terms of the CMI between the output hypothesis and supersample variables, leading to a substantial reduction in the dimensionality of random variables involved in the bounds and achieving the computational feasibility. Follow-up work \citep{hellstrom2022new,wang2023tighter} further improves this methodology by integrating the loss pairs and loss differences, thereby yielding tighter generalization bounds.

\section{Preliminaries}
\paragraph{Basic Notations.}  We denote random variables by capital letters (e.g., $X$), their specific values by lowercase letters (e.g., $x$), and the corresponding domains by calligraphic letters (e.g., $\mathcal{X}$). Let $P_{X,Y}$ denote the joint distribution of variable $(X,Y)$, $P_{X}$ denote the marginal probability distribution of $X$, and $P_{X|Y}$ be the conditional distribution of $X$ given $Y$, where $P_{X|y}$ denotes the one conditioning on a sepcific value $Y=y$. Similarly, denote by $\mathbb{E}_X$, $\mathrm{Var}_X$, and $\mathrm{Cov}_X$ the expectation, variance, and covariance matrix taken over $X\sim P_X$. Given probability measures $P$ and $Q$, we define the Kullback-Leibler (KL) divergence of $P$ w.r.t $Q$ as $D(P\Vert Q) = \int \log\frac{dP}{dQ} dP$. For two Bernoulli distributions with parameters $p$ and $q$, we refer to $d(p\Vert q)=p\log(\frac{p}{q}) + (1-p)\log(\frac{1-p}{1-q})$ as the binary KL divergence. Let $I(X;Y)$ be the MI between variables $X$ and $Y$, and $I(X;Y|Z) = \mathbb{E}_Z[I^Z(X;Y)] $ be the CMI conditioned on $Z$, where $I^Z(X;Y)=D(P_{X,Y|Z}\Vert P_{X|Z}P_{Y|Z})$ denotes the disintegrated MI. Let $\log$ be the logarithmic function with base $e$ and $\mathbf{1}_d$ denote a $d$-dimensional vector of ones.

\paragraph{Meta-learning.} Meta-learning aims to automatically infer an output hypothesis $U\in\mathcal{U}$ from data across multiple related tasks, enabling rapid and efficient adaptation to novel, previously unseen tasks \citep{chen2021generalization,rezazadeh2021conditional,hellstrom2022evaluated}, where $\mathcal{U}$ represents the parametrized hypothesis space. Consider a common task environment $\mathcal{T}$ defined by a probability distribution $P_{\mathcal{T}}$. Assume that we have $n$ different observation tasks, denoted as $\tau_{\mathbb{N}}=(\tau_1,\ldots,\tau_n)$, independently drawn from the distribution $P_{\mathcal{T}}$. For each task $\tau_i$, where $i\in[n]$, further draw $m$ i.i.d. training samples from a data-generation distribution $P_{Z|\tau_i}$ over the sample space $\mathcal{Z}$ associated with task $\tau_i$, denoted as $T_{\mathbb{M}}^i=\{Z^i_j\}^m_{j=1}$. The complete meta-training dataset consisting of $n$ tasks with $m$ in-task samples is then expressed by $T_{\mathbb{M}}^{\mathbb{N}}=\{T_{\mathbb{M}}^i\}^n_{i=1}$. A widely adopted approach for learning an ideal meta-parameter $U$ is to minimize the overall empirical losses of the task-specific parameters $W_{\mathbb{N}}=(W_1,\ldots,W_n)$ adjusted by $U$ on the dataset $T_{\mathbb{M}}^{\mathbb{N}}$, known as the empirical meta-risk:
\begin{equation*}
    \mathcal{R}(U, T_{\mathbb{M}}^{\mathbb{N}}) = \frac{1}{n}\sum_{i=1}^n \mathbb{E}_{W_i\sim P_{W_i|U, T_{\mathbb{M}}^i}} \Big[  \frac{1}{m} \sum_{j=1}^m \ell(U, W_i,Z^i_{j}) \Big],
\end{equation*}
where the expectation is taken over each task-specific parameter $W_i\in \mathcal{W}$ tuned via $U$ and $T_{\mathbb{M}}^i$ in the parameter space $\mathcal{W}$, and $\ell:\mathcal{U}\times\mathcal{W}\times \mathcal{Z}\rightarrow \mathbb{R}^+$ is a given loss function. The corresponding population meta-risk is defined by 
\begin{equation*}
    \mathcal{R}(U,\mathcal{T}) = \mathbb{E}_{\tau \sim P_{\mathcal{T}}}\mathbb{E}_{T_{\mathbb{M}} \sim P_{Z^{\otimes m}|\tau}}\mathbb{E}_{P_{W|T_{\mathbb{M}},U}} \big[ \mathbb{E}_{Z\sim P_{Z|\tau}} \ell(U,W,Z) \big],
\end{equation*}
which measures the prediction performance on unseen samples $Z \sim P_{Z|\tau}$ after leveraging the meta-parameters $U$ to fine-tune the task-specific parameters $W$ for adapting to a new task $\tau$, thereby evaluating the generalization capability of $U$ for novel tasks.

We define the meta-generalization gap by $\overline{\mathrm{gen}} = \mathbb{E}_{U, T_{\mathbb{M}}^{\mathbb{N}}} [\mathcal{R}(U,\mathcal{T}) - \mathcal{R}(U, T_{\mathbb{M}}^{\mathbb{N}})]$, 
which quantifies the discrepancy between the empirical and population meta-risks. To simplify the notation, let  $\mathcal{R}_{\mathcal{T}} :=\mathbb{E}_{U, T_{\mathbb{M}}^{\mathbb{N}}} \big[ \mathcal{R}(U,\mathcal{T})]$ and $\hat{\mathcal{R}} := \mathbb{E}_{U, T_{\mathbb{M}}^{\mathbb{N}}} \big[ \mathcal{R}(U, T_{\mathbb{M}}^{\mathbb{N}})]$.

\begin{figure*}[t]
    \vskip 0.2in
    \centering
      \includegraphics[width=1\linewidth]{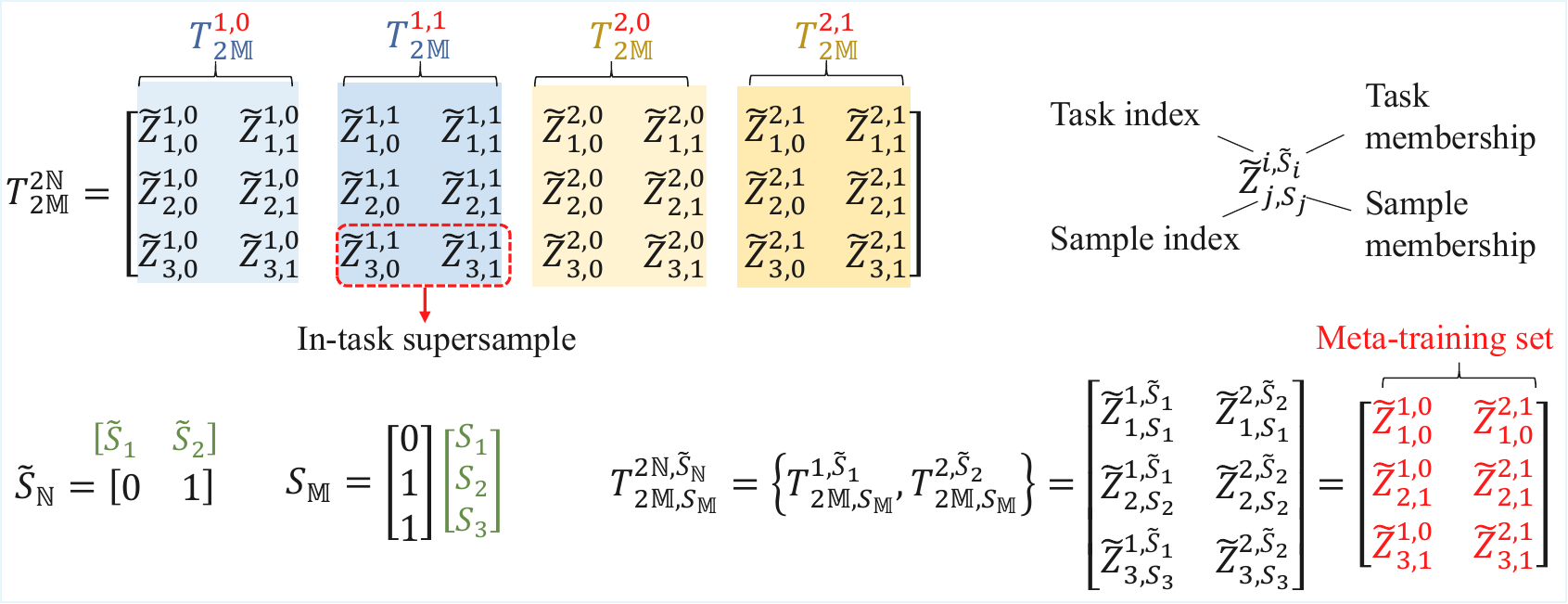}
    \caption{A graphical representation of the notation system. In this example, we chose $n=2$, i.e., $4$ meta tasks, and $m=3$, i.e., $6$ data samples per task. Here, the meta-supersample $T^{2\mathbb{N}}_{2\mathbb{M}}$ involves task pairs $T^{2\mathbb{N}}_{2\mathbb{M}}=\{T_{2\mathbb{M}}^{i,0},T_{2\mathbb{M}}^{i,1} \}_{i=1}^2$, where each task pair is marked in blue and yellow respectively. Each individual task then contains three sample pairs, e.g., the third sample pair $(\tilde{Z}^{1,1}_{3,0},\tilde{Z}^{1,1}_{3,1})$ in the task $T^{1,1}_{2\mathbb{M}}$ marked by the red dotted box. Consequently, each sample $\tilde{Z}^{i,\tilde{S}_i}_{j,S_j}$ can be identified by task index $i$, task membership variable $\tilde{S}_i$, sample index $j$, and sample membership variable $S_j$, for $i=1,2$, $j=1,2,3$, and $\tilde{S}_i,S_j\sim \mathrm{Unif}(0,1)$. Once $\tilde{S}_{\mathbb{N}}=\{\tilde{S}_i\}_{i=1}^2$ and $S_\mathbb{M}=\{S_j\}_{j=1}^3$ are determined, samples can be enumerated from the meta-supersample $T^{2\mathbb{N}}_{2\mathbb{M}}$ according to these vectors to construct meta-training and meta-test datasets. To illustrate this, let $\tilde{S}_{\mathbb{N}}=\{\tilde{S}_1,\tilde{S}_2\}=\{0,1\}$ and $S_{\mathbb{M}}=\{S_1,S_2,S_3\}=\{0,1,1\}$. From the first task-pair, $\tilde{S}_1=0$ indicates that we should select task $0$, i.e., $T^{1,0}_{2\mathbb{M}}$. Then, from the first sample pair in the task, $S_1 = 0$ indicates that we should select the sample $0$, i.e., $\tilde{Z}^{1,0}_{1,0}$. Repeat this process until all samples of the meta-training set $T^{2\mathbb{N},\tilde{S}_{\mathbb{N}}}_{2\mathbb{M},S_{\mathbb{M}}}$ are identified, which are marked in red. The meta-test dateset $T^{2\mathbb{N}, \bar{\tilde{S}}_{\mathbb{N}}}_{2\mathbb{M},\bar{S}_{\mathbb{M}}} =\{\tilde{Z}^{i,\bar{\tilde{S}}_i}_{j,\bar{S}_j}\}_{i,j=1}^{n,m}$ is constructed by an analogous procedure, which is based on $\bar{\tilde{S}}_{\mathbb{N}}=\{1-\tilde{S}_i\}_{i=1}^2$ and $\bar{S}_{\mathbb{M}}=\{1-S_j\}_{j=1}^3$.
  }
  \vskip -0.2in
    \label{fig11} 
  \end{figure*}

\paragraph{CMI-based Framework.} The CMI framework is originally investigated in \citep{steinke2020reasoning} by introducing the supersample setting for generalization analysis. Let $T^{2\mathbb{N}}_{2\mathbb{M}}=\{T^{i,0}_{2\mathbb{M}},T^{i,1}_{2\mathbb{M}}\}_{i=1}^n$ be the meta-supersample dataset across $n\times 2$ different tasks, where $T^{i,0}_{2\mathbb{M}} = \{(\tilde{Z}^{i,0}_{j,0}, \tilde{Z}^{i,0}_{j,1})\}_{j=1}^m, T^{i,1}_{2\mathbb{M}} = \{(\tilde{Z}^{i,1}_{j,0}, \tilde{Z}^{i,1}_{j,1})\}_{j=1}^m \in\mathcal{Z}^{m\times 2}$ consist of $m\times 2$ i.i.d. samples. Further let $\tilde{S}_{\mathbb{N}}=\{\tilde{S}_i\}_{i=1}^n \sim \mathrm{Unif}(\{0,1\}^n)$ and $S_{\mathbb{M}}=\{S_j\}_{j=1}^m \sim \mathrm{Unif}(\{0,1\}^m)$ be the membership vectors independent of $T^{2\mathbb{N}}_{2\mathbb{M}}$, and $\bar{\tilde{S}}_{\mathbb{N}}=\{1-\tilde{S}_i\}_{i=1}^n$ and $\bar{S}_{\mathbb{M}}=\{1-S_j\}_{j=1}^m$ be the modulo-2 complement of $\tilde{S}_{\mathbb{N}}$ and $S_{\mathbb{M}}$, respectively. We utilize the variables $(\tilde{S}_{\mathbb{N}},S_{\mathbb{M}})$ and $(\bar{\tilde{S}}_{\mathbb{N}},\bar{S}_{\mathbb{M}})$ to separate $T^{2\mathbb{N}}_{2\mathbb{M}}$ into the meta-training dataset $T^{2\mathbb{N},\tilde{S}_{\mathbb{N}}}_{2\mathbb{M},S_{\mathbb{M}}}=\{\tilde{Z}^{i,\tilde{S}_i}_{j,S_j}\}_{i,j=1}^{n,m}$ and the meta-test dataset $T^{2\mathbb{N}, \bar{\tilde{S}}_{\mathbb{N}}}_{2\mathbb{M},\bar{S}_{\mathbb{M}}} =\{\tilde{Z}^{i,\bar{\tilde{S}}_i}_{j,\bar{S}_j}\}_{i,j=1}^{n,m}$. Analogously, denote by $T^{2\mathbb{N},\bar{\tilde{S}}_{\mathbb{N}}}_{2\mathbb{M},S_{\mathbb{M}}}$ the training dataset for the meta-test task and $T^{2\mathbb{N},\tilde{S}_{\mathbb{N}}}_{2\mathbb{M},\bar{S}_{\mathbb{M}}}$ the test dataset for the meta-training tasks. In Figure \ref{fig11}, we illustrate $T^{2\mathbb{N}}_{2\mathbb{M}}$ as an $m \times 4n$ matrix for the case of $n=2$ task pairs and $m=3$ sample pairs for each task, and show the process of how to select the corresponding meta-training set $T^{2\mathbb{N},\tilde{S}_{\mathbb{N}}}_{2\mathbb{M},S_{\mathbb{M}}}$ from $T^{2\mathbb{N}}_{2\mathbb{M}}$ using the membership vectors $\tilde{S}_{\mathbb{N}} $ and $S_{\mathbb{M}}$.

Let $L^{i,\tilde{S}_i}_{j,S_j} \triangleq  \ell(U,W_i,\tilde{Z}^{i,\tilde{S}_i}_{j,S_j})$ and $L^{i,\bar{\tilde{S}}_i}_{j,\bar{S}_j} \triangleq  \ell(U,W_i,\tilde{Z}^{i,\bar{\tilde{S}}_i}_{j,\bar{S}_j})$ represent the samplewise losses on the meta-training and meta-test data, respectively. Define $\Psi = \{\Psi_{i,j} = \tilde{S}_i\oplus S_j\}_{i,j=1}^{n,m}$ and $\tilde{\Psi} = \{\tilde{\Psi}_{i,j} = 1 \oplus \Psi_{i,j}\}_{i,j=1}^{n,m}$, where $\oplus$ is the XOR operation. To simplify the notations, we denote $L_{j,1}^{i,\tilde{\Psi}_{i,j}}$ and $L_{j,0}^{i,\tilde{\Psi}_{i,j}}$ as $L_{i,j}^{\Psi_+}$ and $L_{i,j}^{\Psi_-}$, respectively. $L^{\Psi}_{i,j}=\{L_{i,j}^{\Psi_+} , L_{i,j}^{\Psi_-} \}$ then represents a pair of losses, and $\Delta^{\Psi}_{i,j} = L_{i,j}^{\Psi_+} - L_{i,j}^{\Psi_-}$ is their difference. We define the loss pairs on the full meta-supersample set as $L_{2\mathbb{M}}^{2\mathbb{N}}=\{L^{i}_{j}= (L^{i,0}_{j,0},L^{i,1}_{j,1},L^{i,1}_{j,0},L^{i,0}_{j,1})\}_{i,j=1}^{n,m}$.

\section{Information-theoretic Generalization Bounds for Meta-learning} \label{Section3}
In this section, we provide a unified analytical framework for meta-learning by establishing multiple information-theoretic bounds, with the aim of enhancing tightness, improving convergence rates, and ensuring computational tractability. We start by enhancing MI bounds presented in \citep{chen2021generalization,bu2023generalization}, incorporating the random subsets of tasks and within-task samples. Subsequently, we establish a series of novel CMI-based upper bounds under the improved meta-supersample setting, thereby tightening existing results and achieving computational tractability \citep{rezazadeh2021conditional,hellstrom2022evaluated}. Finally, we provide fast-rate bounds through the lens of the weighted generalization error and loss variance. The proofs of all theoretical results are provided in the Appendices.

\subsection{Generalization Bounds with Input-output MI} \label{Section3.1}
Theorem \ref{theorem3.1} below provides improved input-output MI bounds through the information contained in randomly selected subsets of tasks and within-task samples:

\begin{theorem} \label{theorem3.1}
    Let $\mathbb{K}$ and $\mathbb{J}$ be random subsets of $[n]$ and $[m]$ with sizes $\zeta$ and $\xi$, respectively, independent of $T^{\mathbb{N}}_{\mathbb{M}}$ and $(U,W_{\mathbb{N}})$. Assume that $\ell(u,w,Z)$ is $\sigma$-sub-gaussian with respect to $Z\sim P_{Z|\tau}$ and $\tau\sim P_{\mathcal{T}}$ for all $u,w$, then
    \begin{equation*}
       \vert \overline{\mathrm{gen}} \vert \leq \mathbb{E}_{K\sim \mathbb{K}, J\sim\mathbb{J}} \sqrt{\frac{2\sigma^2}{\zeta \xi} I(U,W_K;T^{K}_{J})}.  
    \end{equation*}
\end{theorem}
Theorem \ref{theorem3.1} corroborates a key insight from \citep{xu2017information}: the less the output hypotheses depend on the input data, the better the learning algorithm generalizes. Notice that when $\zeta = n$ and $\xi = m$, the derived bound regarding the entire meta-training dataset includes the result of \citep{chen2021generalization,bu2023generalization} as a special case: $\vert \overline{\mathrm{gen}} \vert \leq \sqrt{2\sigma^2 I(U,W_{\mathbb{N}}; T^{\mathbb{N}}_{\mathbb{M}})/nm}$. However, such an upper bound reflects an `on-average' stability property \citep{shalev2010learnability}, potentially leading to meaningless results when the output algorithm is deterministic, namely, $I(U,W_{\mathbb{N}}; T^{\mathbb{N}}_{\mathbb{M}}) = \infty$ for unique minimizer $(U,W_{\mathbb{N}})$ of the empirical meta-risk. Substituting $\zeta = \xi= 1$ yields the following samplewise MI bound: $$\vert \overline{\mathrm{gen}} \vert \leq \frac{1}{nm}\sum_{i,j=1}^{n,m}\sqrt{2\sigma^2 I(U, W_i; Z^i_{j})},$$ which leads to the `point-wise' stability \citep{raginsky2016information} and thus tighten existing bound \citep{jose2021information,chen2021generalization}.


It is worth noting that the optimal choice of $\zeta$ and $\xi $ values for rigid bounds is not straightforward, as smaller values reduce both the denominator and the MI term. The following proposition demonstrates that the upper bound in Theorem \ref{theorem3.1} is non-decreasing with respect to both $\zeta $ and $\xi$, implying that the smallest $\zeta $ and $\xi$, namely $\zeta =\xi=1$, yield the tightest bound.

\begin{proposition}\label{proposition1}
    Let $\zeta\in[n-1]$, $\xi \in[m-1]$, and $\mathbb{K}$ and $\mathbb{J}$ be random subsets of $[n]$ and $[m]$ with sizes $\zeta$ and $\xi$, respectively. Further, let $\mathbb{K}'$ and $\mathbb{J}'$ be random subsets with sizes $\zeta+1$ and $\xi+1$, respectively. If $g:\mathbb{R}\rightarrow\mathbb{R}$ is any non-decreasing concave function, then 
    \begin{align*}
         \mathbb{E}_{K\sim \mathbb{K}, J \sim \mathbb{J}} g\left(\frac{1}{\zeta \xi} I(U,W_{K}; T^{K}_{J}) \right) 
        \leq \mathbb{E}_{K' \sim \mathbb{K}', J' \sim \mathbb{J}'} g\left(\frac{1}{(\zeta+1)(\xi+1)} I(U,W_{K'}; T^{K'}_{J'}) \right).
    \end{align*}
    
\end{proposition}
Applying $g(x)=\sqrt{x}$ does indeed justify that $\zeta =1 $ and $ \xi=1$ are the optimal values for minimizing the generalization bound in Theorem \ref{theorem3.1}.

\subsection{Generalization Bounds with CMI} \label{Section3.2}
A recent advance by \cite{rezazadeh2021conditional} extends the conditional mutual information (CMI) methodology, originally introduced in \citep{steinke2020reasoning}, to meta-learning, obtaining the following sharper bound:
\begin{lemma}[\cite{rezazadeh2021conditional}]\label{lemma3.2}
    Assume that the loss function takes values in $[0,1]$, then 
    \begin{align}
        \vert \overline{\mathrm{gen}} \vert \leq  & \sqrt{ \frac{2}{n} I(U;\tilde{S}_{\mathbb{N}},S^{\mathbb{N}}_{\mathbb{M}}|T^{2\mathbb{N}}_{2\mathbb{M}})}
         + \frac{1}{n} \sum_{i=1}^n \sqrt{\frac{2}{m} I(W_i;S_{\mathbb{M}}^{i}|T^{2\mathbb{N}}_{2\mathbb{M}}, \tilde{S}_i)}, 
    \end{align}
where $\tilde{S}_{\mathbb{N}}=\{\tilde{S}_i\}_{i=1}^n$,  $S^{\mathbb{N}}_{\mathbb{M}}=\{S_{\mathbb{M}}^{i}\}_{i=1}^n$, and $ S_{\mathbb{M}}^{i}=\{S_j^{i}\}_{j=1}^m\sim \mathrm{Unif}(\{0,1\}^m)$ are the membership vectors.
\end{lemma}
While the improvement is achieved by leveraging the CMI term that integrates low-dimensional membership vectors $\tilde{S}_{\mathbb{N}}$ and $S^{\mathbb{N}}_{\mathbb{M}}$, Lemma \ref{lemma3.2} is confined to a two-step derivation, which neglects the dependencies between the meta-hypothesis $U$ and the task-specific hypotheses $W_i$, $i=1\ldots,n$, leading to an undesirable scaling rate w.r.t the task size $n$ and the sample size $m$ per task. 



To provide a more refined analysis, we improve the aforementioned bounds by establishing the following single-step CMI bound:

\begin{theorem}\label{theorem3.5}
  Let $\mathbb{K}$ and $\mathbb{J}$ be random subsets of $[n]$ and $[m]$ with sizes $\zeta$ and $\xi$, respectively. If the loss function $\ell(\cdot,\cdot,\cdot)$ is bounded within $[0,1]$, then
    \begin{equation*}
        \vert \overline{\mathrm{gen}} \vert \leq  \mathbb{E}_{T_{2\mathbb{M}}^{2\mathbb{N}}, K\sim\mathbb{K}, J\sim\mathbb{J} }\sqrt{\frac{2}{\zeta\xi}I^{T_{2\mathbb{M}}^{2\mathbb{N}}}(U, W_K;  \tilde{S}_K,S_J)},
    \end{equation*}
where $\tilde{S}_K=\{\tilde{S}_i\}_{i=1}^{\zeta}$ and $S_J=\{S_j\}_{j=1}^{\xi}$ are the membership vectors.
\end{theorem}
Notice that when $\zeta=n$ and $\xi=m$, Theorem \ref{theorem3.5} bounds the meta-generalization gap through the CMI between the hypotheses $(U, W_{\mathbb{N}})$ and binary variables $\tilde{S}_{\mathbb{N}},S_{\mathbb{M}}$, conditioning on the meta-supersample $T_{2\mathbb{M}}^{2\mathbb{N}}$. This result substantially tightens the MI-based bound while achieving a more favorable scaling rate than the existing CMI-based bound. On the one hand, it can be shown that the CMI term $I(U, W_\mathbb{N}; \tilde{S}_{\mathbb{N}},S_{\mathbb{M}} | T_{2\mathbb{M}}^{2\mathbb{N}})$ consistently provides a tighter upper bound than the MI term $I(U, W_\mathbb{N};  T_{2\mathbb{M},S_{\mathbb{M}}}^{2\mathbb{N},\tilde{S}_{\mathbb{N}}})$ in \citep{chen2021generalization,bu2023generalization}: by using the Markov chain $(T_{2\mathbb{M}}^{2\mathbb{N}} ,  \tilde{S}_{\mathbb{N}},S_{\mathbb{M}})- T_{2\mathbb{M},S_{\mathbb{M}}}^{2\mathbb{N},\tilde{S}_{\mathbb{N}}} - (U,W_\mathbb{N})$, it is obvious that $I(U,W_\mathbb{N} ;  T_{2\mathbb{M},S_{\mathbb{M}}}^{2\mathbb{N},\tilde{S}_{\mathbb{N}}})=I(U,W_\mathbb{N} ; T_{2\mathbb{M}}^{2\mathbb{N}},\tilde{S}_{\mathbb{N}},S_{\mathbb{M}})=I(U,W_\mathbb{N} ; \tilde{S}_{\mathbb{N}},S_{\mathbb{M}} |T_{2\mathbb{M}}^{2\mathbb{N}}) + I(U,W_\mathbb{N} ; T_{2\mathbb{M}}^{2\mathbb{N}})\geq I(U,W_\mathbb{N} ; \tilde{S}_{\mathbb{N}},S_{\mathbb{M}} |T_{2\mathbb{M}}^{2\mathbb{N}})$. On the other hand, the resulting bound exhibit a desirable convergence rate of $\mathcal{O}(\frac{1}{\sqrt{nm}})$ instead of $\mathcal{O}(\frac{1}{\sqrt{n}}+\frac{1}{\sqrt{nm}})$ in Lemma \ref{lemma3.2}.  Furthermore, by setting $\zeta=\xi=1$, we obtain the following samplewise CMI bound for meta-learning: $$ \vert \overline{\mathrm{gen}} \vert \leq \mathbb{E}_{T_{2\mathbb{M}}^{2\mathbb{N}}} \frac{1}{nm}\sum_{i,j=1}^{n,m}\sqrt{2 I^{T_{2\mathbb{M}}^{2\mathbb{N}}}(U, W_i; \tilde{S}_i,S_j)}.$$ 
Applying the chain rule and Jensen's inequality, the above bound is further upper-bounded by $\frac{1}{nm}\sum_{i,j=1}^{n,m} \Big(\sqrt{2 I(U;  \tilde{S}_i,S_j|T_{2\mathbb{M}}^{2\mathbb{N}}) }  + \sqrt{2  I( W_i;  \tilde{S}_i,S_j|T_{2\mathbb{M}}^{2\mathbb{N}},U)}  \Big)$, which strengthens the result in Lemma \ref{lemma3.2}.

By extending the analysis of Proposition \ref{proposition1} to the disintegrated MI $I^{T_{2\mathbb{M}}^{2\mathbb{N}}}(U, W_K; \tilde{S}_K, S_J)$, we obtain the following monotonic property related to Theorem \ref{theorem3.5}:
\begin{proposition}\label{proposition2}
    Let $\zeta\in[n-1]$, $\xi \in[m-1]$, and $\mathbb{K}$ and $\mathbb{J}$ be random subsets of $[n]$ and $[m]$ with sizes $\zeta$ and $\xi$, respectively. Further, let $\mathbb{K}'$ and $\mathbb{J}'$ be random subsets with sizes $\zeta+1$ and $\xi+1$, respectively. If $g:\mathbb{R}\rightarrow\mathbb{R}$ is any non-decreasing concave function, then for $T_{2\mathbb{M}}^{2\mathbb{N}} \sim \{P_{Z|\tau_i}\}_{i=1}^n$ over $\mathcal{Z}^{2n\times 2m}$,
    \begin{align*}
         & \mathbb{E}_{K\sim \mathbb{K}, J \sim \mathbb{J}} g\left(\frac{1}{\zeta \xi} I^{T_{2\mathbb{M}}^{2\mathbb{N}}}(U, W_K;  \tilde{S}_K,S_J) \right) \\
         \leq & \mathbb{E}_{K' \sim \mathbb{K}', J' \sim \mathbb{J}'} g\left(\frac{1}{(\zeta+1)(\xi+1)} I^{T_{2\mathbb{M}}^{2\mathbb{N}}}(U, W_{K'};  \tilde{S}_{K'},S_{J'}) \right).
    \end{align*}
\end{proposition}
By leveraging Proposition \ref{proposition2} with $g(x)=\sqrt{x}$ and subsequently taking an expectation over $T_{2\mathbb{M}}^{2\mathbb{N}}$, one can observe that Theorem \ref{theorem3.5} is non-decreasing with respect to both variables $\zeta$ and $\xi$. Accordingly, $\zeta=\xi=1$ emerges as the optimal choice for obtaining the tightest meta-generalization bounds.

In parallel with the development in Theorem \ref{theorem3.5}, we proceed to derive an improved upper bound on the binary KL divergence between the expected empirical meta-risk $\hat{\mathcal{R}}$ and the mean of the expected empirical and population meta-risks $(\hat{\mathcal{R}} + \mathcal{R}_{\mathcal{T}}) /2$, as follows:
\begin{theorem}\label{theorem3.6}
    Assume that the loss function $\ell(\cdot,\cdot,\cdot)$ is bounded within $[0,1]$, then
        \begin{equation*}
            d\left(\hat{\mathcal{R}} \Big\Vert \frac{\hat{\mathcal{R}} + \mathcal{R}_{\mathcal{T}}}{2} \right) \leq \frac{1}{nm}\sum_{i,j=1}^{n,m} I(U,W_i; \tilde{S}_i,S_j|T_{2\mathbb{M}}^{2\mathbb{N}}).
        \end{equation*}
\end{theorem}
The intrinsic properties of CMI guarantee the finiteness of these samplewise CMI bounds in Theorems \ref{theorem3.5} and \ref{theorem3.6}, as $I(U,W_i; \tilde{S}_i,S_j|T_{2\mathbb{M}}^{2\mathbb{N}}) \leq H(\tilde{S}_i,S_j) = 2\log 2$.

We further establish a fast-rate generalization bound for meta-learning by leveraging the weighted generalization error: $\overline{\mathrm{gen}}_{C_1}\triangleq \mathcal{R}_{\mathcal{T}} - (1+C_1)\hat{\mathcal{R}}$, where $C_1$ is prescribed constant. This methodology has been rapidly developed within the information-theoretic framework \cite{hellstrom2021fast, wang2023tighter, dongtowards}, facilitating the attainment of fast scaling rates of the generalization bounds.
\begin{theorem}[\textbf{Fast-rate Bound}]\label{theorem3.8}
    Assume that the loss function $\ell(\cdot,\cdot,\cdot)$ is bounded within $[0,1]$, then for any $0 \leq C_2\leq\log 2$ and $C_1\geq -\frac{\log(2-e^{C_2})}{C_2}-1$,
  \begin{equation*}
    \overline{\mathrm{gen}} \leq C_1 \hat{\mathcal{R}} + \frac{1}{nm}\sum_{i,j=1}^{n,m} \frac{I(U,W_i;\tilde{S}_i,S_j|T_{2\mathbb{M}}^{2\mathbb{N}})}{C_2}.
  \end{equation*}
  In the interpolating setting, i.e., $\hat{\mathcal{R}} = 0$, we have
  \begin{equation*}
    \mathcal{R}_{\mathcal{T}} \leq  \frac{1}{nm}\sum_{i,j=1}^{n,m} \frac{I(U,W_i;\tilde{S}_i,S_j|T_{2\mathbb{M}}^{2\mathbb{N}})}{\log 2}.
  \end{equation*}
\end{theorem}
Theorem \ref{theorem3.8} gives a fast-decaying meta-generalization bound, in the sense that the bound benefits from a small empirical error. In the interpolation regime, this upper bound exhibits a faster convergence rate at the order of $\frac{1}{nm}$ as opposed to the conventional order of $\frac{1}{\sqrt{nm}}$.

\subsection{Generalization Bounds with e-CMI} \label{Section3.3}
The evaluated CMI (e-CMI) bounds are originally investigated in \citep{hellstrom2022new}, focusing on the CMI between the evaluated loss pairs and the binary variables, conditioned on the supersamples. This methodology has been extended in a follow-up work \citep{hellstrom2022evaluated} to establish tighter CMI bounds for meta-learning:

\begin{lemma} [Theorem 2 in \citep{hellstrom2022evaluated}] \label{lemma3.9}
    With notations in Lemma \ref{lemma3.2}. Let $L^{i}_{j}= \{L^{i,0}_{j,0}, L^{i,0}_{j,1}, L^{i,1}_{j,0}, L^{i,1}_{j,1}\}$ be the evaluated loss over $T^{2\mathbb{N}}_{2\mathbb{M}}$. Assume that the loss function $\ell(\cdot,\cdot,\cdot) \in [0,1]$, then
    \begin{align*}
        \vert \overline{\mathrm{gen}} \vert \leq \frac{1}{nm}\sum_{i,j=1}^{n,m} \mathbb{E}_{T^{2\mathbb{N}}_{2\mathbb{M}}}\sqrt{2I^{T^{2\mathbb{N}}_{2\mathbb{M}}} (L^i_j; \tilde{S}_i,S^i_j)}.
    \end{align*}
 
\end{lemma}
It is noteworthy that under the same setting as \citep{rezazadeh2021conditional}, Lemma \ref{lemma3.9} enhances the upper bound in Lemma \ref{lemma3.2} by deriving a refined single-step bound, which incorporates the CMI between the evaluated losses $L^{i}_{j}$ and both the environment-level and task-level binary variables $\tilde{S}_i$ and $S^i_j$. However, this bound involves $n+nm$ membership variables $\tilde{S}_i$ and $S^i_j$, potentially exhibiting a substantial cumulative effect of $I^{T^{2\mathbb{N}}_{2\mathbb{M}}} (L^i_j; \tilde{S}_i,S^i_j)$ as $n$ and $m$ increase.

Building upon the proposed meta-supersample setting, we improve the aforementioned bounds by leveraging loss-based CMI with fewer $n+m$ membership variables $\tilde{S}_i$ and $S_j$:

\begin{theorem}\label{theorem4}
    Assume that the loss function $\ell(\cdot,\cdot,\cdot) \in [0,1]$, then
    \begin{equation*}
        \vert \overline{\mathrm{gen}} \vert \leq \frac{1}{nm}\sum_{i,j=1}^{n,m} \mathbb{E}_{T^{2\mathbb{N}}_{2\mathbb{M}}}\sqrt{2I^{T^{2\mathbb{N}}_{2\mathbb{M}}} (L^i_j; \tilde{S}_i,S_j)}.
    \end{equation*}
\end{theorem}
We further derive the following upper bound based on the MI between the loss $L^i_j$ and the membership variables $ \tilde{S}_i,S_j$, without conditioning on the meta-supersample dataset $T^{2\mathbb{N}}_{2\mathbb{M}}$:
\begin{theorem}\label{theorem5}
    Assume that the loss function $\ell(\cdot,\cdot,\cdot) \in [0,1]$, then
        \begin{equation*}
            \vert \overline{\mathrm{gen}} \vert \leq \frac{1}{nm}\sum_{i,j=1}^{n,m}  \sqrt{2I (L^i_j; \tilde{S}_i,S_j)}.
        \end{equation*}
\end{theorem}
Notably, the above bound is strictly tighter than its conditional counterpart in Theorem \ref{theorem4}, as $I(L^i_j;  \tilde{S}_i,S_j)\leq I(L^i_j;  \tilde{S}_i,S_j) + I(T^{2\mathbb{N}}_{2\mathbb{M}}; \tilde{S}_i,S_j| L^i_j) = I(L^i_j ;  \tilde{S}_i,S_j|T^{2\mathbb{N}}_{2\mathbb{M}})$ by the independence between $ \tilde{S}_i,S_j $, and $T^{2\mathbb{N}}_{2\mathbb{M}}$. Theorem \ref{theorem5} thus provides a stronger theoretical guarantee. In a parallel development, we establish the binary KL divergence bound based on the loss-based MI:
\begin{theorem}\label{theorem6}
    Assume that the loss function $\ell(\cdot,\cdot,\cdot) \in [0,1]$, then
        \begin{equation*}
            d\left(\hat{\mathcal{R}} \Big\Vert \frac{\hat{\mathcal{R}} + \mathcal{R}_{\mathcal{T}}}{2} \right) \leq \frac{1}{nm}\sum_{i,j=1}^{n,m} I(L^i_j ; \tilde{S}_i, S_j).
        \end{equation*}
\end{theorem}
Theorem \ref{theorem6} exhibits a comparable convergence rate of $\mathcal{O}(1/nm)$ to previous work \citep{hellstrom2022evaluated}, while improving upon their results by employing tighter MI terms $I(L^i_j ; \tilde{S}_i, S_j)$ and reducing the cumulative effect associated with $n+m$ membership variables. In the special case $n=1$, our bounds recover classical single-task generalization results \citep{harutyunyan2021information,hellstrom2022new}. Additionally, we notice that these loss-based upper bounds, originally formulated based on the dependence between loss pairs and two 1-dimensional variables, can be tightened by exploring the dependence between only two one-dimensional variables (loss difference and sample mask), as will be detailed in the subsequent section.

\subsection{Generalization Bounds via Loss Difference}\label{section3.3}
Building on the loss-difference (ld) methodology \citep{wang2023tighter}, we integrate the member variables $\tilde{S}_i$ and $S_j$ into a single variable $\Psi$ and establish the following improved bounds based on loss difference $\Delta_{i,j}^{\Psi}$, where $\Psi$ determines a pair of losses from $L^i_j$ on the meta-training and test data, from which $\Delta_{i,j}^{\Psi}$ is calculated:

\begin{theorem}\label{theorem7}
    Assume that the loss function $\ell(\cdot,\cdot,\cdot) \in [0,1]$, then
    \begin{equation*}
        \vert \overline{\mathrm{gen}} \vert\leq    \frac{1}{nm}\sum_{i,j=1}^{n,m} \sqrt{2I(\Delta_{i,j}^{\Psi};S_j|T_{2\mathbb{M}}^{2\mathbb{N}})}.
    \end{equation*}
\end{theorem}
This bound depends only on the information contained in the scalar $\Delta_{i,j}^{\Psi}$ and a single variable $S_j$, thereby leading to a significant improvement over existing results \citep{chen2021generalization,rezazadeh2021conditional,hellstrom2022evaluated} and enhancing computational tractability. By applying the data-processing inequality on the Markov chain $(\tilde{S}_i,S_{j})-(U,W_i)-(L^j_i,\Psi_{i,j})- \Delta_{i,j}^{\Psi}$ (conditioned on $T^{2\mathbb{N}}_{2\mathbb{M}}$), we have $I(\Delta_{i,j}^{\Psi};S_j|T^{2\mathbb{N}}_{2\mathbb{M}}) \leq I(L^i_j ; \tilde{S}_i,S_{j}|T^{2\mathbb{N}}_{2\mathbb{M}}) \leq I(U,W_i ; \tilde{S}_i,S_{j}|T^{2\mathbb{N}}_{2\mathbb{M}})$, demonstrating the advantage of Theorem \ref{theorem7}.

Similarly, we derive the following unconditional upper bound based on the loss difference:

\begin{theorem}\label{theorem8}
    Assume that the loss function $\ell(\cdot,\cdot,\cdot) \in [0,1]$, then
        \begin{equation*}
            \vert \overline{\mathrm{gen}} \vert \leq \frac{1}{nm}\sum_{i,j=1}^{n,m} \sqrt{2I(\Delta_{i,j}^{\Psi};S_j)}.
        \end{equation*}
\end{theorem}

Analogous to the analysis in Theorem \ref{theorem5}, it is evidence that the MI between $\Delta_{i,j}^{\Psi}$ and $S_j$ is strictly tighter than its conditional counterpart in Theorem \ref{theorem7} by utilizing the independence between $S_j $ and $T^{2\mathbb{N}}_{2\mathbb{M}}$. In the special case where the loss function $\ell(\cdot,\cdot,\cdot)$ is the zero-one loss, $I(\Delta_{i,j}^{\Psi} ; S_j)$ can be interpreted as the rate of reliable communication over a memoryless channel with input $S_j $ and output $\Delta_{i,j}^{\Psi}$, as discussed in \citep{wang2023tighter}. Consequently, this gives rise to a precise meta-generalization bound under the interpolating regime:
\begin{theorem}\label{theorem9}
    Assume that the loss function $\ell(\cdot,\cdot,\cdot) \in \{0, 1 \}$. In the interpolating setting, i.e., $\hat{\mathcal{R}} = 0$, then
    \begin{equation*}
        \mathcal{R}_{\mathcal{T}} = \frac{1}{nm}\sum_{i,j=1}^{n,m} \frac{I(\Delta_{i,j}^{\Psi};S_j) }{\log 2} = \frac{1}{nm}\sum_{i,j=1}^{n,m} \frac{I(L_{i,j}^{\Psi}; S_j)}{\log 2}.
    \end{equation*}
\end{theorem}
Theorem \ref{theorem9} yields fast-decaying generalization bounds of order $\mathcal{O}(1/nm)$, in the sense that the bound benefits from a small empirical error. In such an interpolating setting, the population meta-risk can be precisely determined via the samplewise MI between $S_j$ and either the loss difference $\Delta_{i,j}^{\Psi}$ or the selected loss pair $L_{i,j}^{\Psi}$. This leads to an essentially tight bound 
 on the meta-generalization gap, which can be further refined in the following theorem to accommodate a broad class of bounded loss functions by utilizing the weighted generalization error: $\overline{\mathrm{gen}}_{C_1}\triangleq \mathcal{R}_{\mathcal{T}} - (1+C_1)\hat{\mathcal{R}}$, where $C_1$ is prescribed constant:

\begin{theorem}[\textbf{Fast-rate Bound}]\label{theorem10}
    Assume that the loss function $\ell(\cdot,\cdot,\cdot)\in [0,1]$. For any $0 < C_2 < \frac{\log 2}{2}$ and $C_1\geq -\frac{\log(2-e^{C_2})}{2C_2}-1$, 
    \begin{equation*}
        \overline{\mathrm{gen}} \leq C_1 \hat{\mathcal{R}} + \frac{1}{nm}\sum_{i,j=1}^{n,m} \frac{\min\{I(L_{i,j}^{\Psi};S_j ), 2 I(L_{i,j}^{\Psi_+};S_j)\}}{C_2}.
    \end{equation*}
    In the interpolating setting, i.e., $\hat{\mathcal{R}} = 0$, we further have 
    \begin{equation*}
        \mathcal{R}_{\mathcal{T}} \leq \frac{1}{nm}\sum_{i,j=1}^{n,m} \frac{2\min\{I(L_{i,j}^{\Psi};S_j ), 2 I(L_{i,j}^{\Psi_+};S_j)\}}{\log 2}.
    \end{equation*}
\end{theorem}
Theorem \ref{theorem10} achieves a convergence rate of $\mathcal{O}(1/nm)$, the same as Theorem \ref{theorem9}, while improving upon this bound by simultaneously taking the minimum between paired-loss MI $I(L_{i,j}^{\Psi};S_j )$ and single-loss MI $2I(L_{i,j}^{\Psi_+};S_j)$. The dependency between these MI terms is depicted by the interaction information $I(L_{i,j}^{\Psi_+};L_{i,j}^{\Psi_-};S_j)$, where $I(L_{i,j}^{\Psi_+};L_{i,j}^{\Psi_-};S_j) = I(L_{i,j}^{\Psi_+};S_j) - I(L_{i,j}^{\Psi_+};S_j|L_{i,j}^{\Psi_-})=2 I(L_{i,j}^{\Psi_+};S_j) - I(L_{i,j}^{\Psi};S_j)$ can be either positive or negative. For the special case of single-task learning, Theorem \ref{theorem10}, which does not impose a definitive ordering on the MI terms, improves upon existing interpolating bounds \citep{hellstrom2022new,wang2023tighter}. 

Given that interpolating bounds typically achieve a fast rate under small or even zero empirical risk, we generalize them to the non-interpolating regime by leveraging the empirical loss variance, defined as follows.

\begin{definition}
    For any $\gamma\in(0,1)$, $\gamma$-variance of the meta-learning, $V(\gamma)$, is defined by 
    \begin{equation*}
      V(\gamma) := \mathbb{E}_{U,T_{\mathbb{M}}^{\mathbb{N}}}\Big[\frac{\sum_{i,j=1}^{n,m} \mathbb{E}_{W_i|U,T_\mathbb{M}^i} (\ell(U,W_i,Z^i_j) - (1+\gamma)  \mathcal{R}(U, T_{\mathbb{M}}^{\mathbb{N}}) )^2}{nm}\Big].
\end{equation*}
\end{definition}
Accordingly, the refined fast-rate bound is provided as follows:
\begin{theorem}[\textbf{Fast-rate Bound}]\label{theorem12}
    Assume that the loss function $\ell(\cdot,\cdot,\cdot)\in\{0,1\}$ and $\gamma\in(0,1)$. Then, for any $0 < C_2 < \frac{\log 2}{2}$ and $C_1\geq -\frac{\log(2-e^{C_2})}{2C_2\gamma^2}-\frac{1}{\gamma^2}$, we have 
    \begin{equation*}
        \overline{\mathrm{gen}} \leq  C_1 V(\gamma) + \frac{1}{nm}\sum_{i,j=1}^{n,m}\frac{\min\{I(L_{i,j}^{\Psi};S_j ), 2 I(L_{i,j}^{\Psi_+};S_j)\}}{C_2}.
    \end{equation*}
\end{theorem}
Comparing Theorems \ref{theorem10} and \ref{theorem12} under the same constants $C_1,C_2$ illustrates that this loss variance bound is more stringent than the interpolating bound by at least $C_1(1-\gamma^2)\mathbb{E}_{U,T_{\mathbb{M}}^{\mathbb{N}}}[\mathcal{R}^2(U,T_{\mathbb{M}}^{\mathbb{N}})]$.  Hence, Theorem \ref{theorem12} could reach zero $\gamma$-variance even for non-zero empirical losses, thereby making it applicable to a wide range of application scenarios.
\section{Applications}
In this section, we extend our analysis to two widely used meta-training strategies: one that uses meta-training data jointly, and another that employs separate in-task training and test data.  Our results are associated with the mini-batched noisy iterative meta-learning algorithms, with a particular focus on stochastic gradient Langevin dynamics (SGLD) \citep{welling2011bayesian}.

\begin{figure}[t]
  \vskip 0.2in
  \centering
  \subfigure[Joint Meta-learning Paradigm]{
    \label{aa} 
    \includegraphics[width=0.65\linewidth]{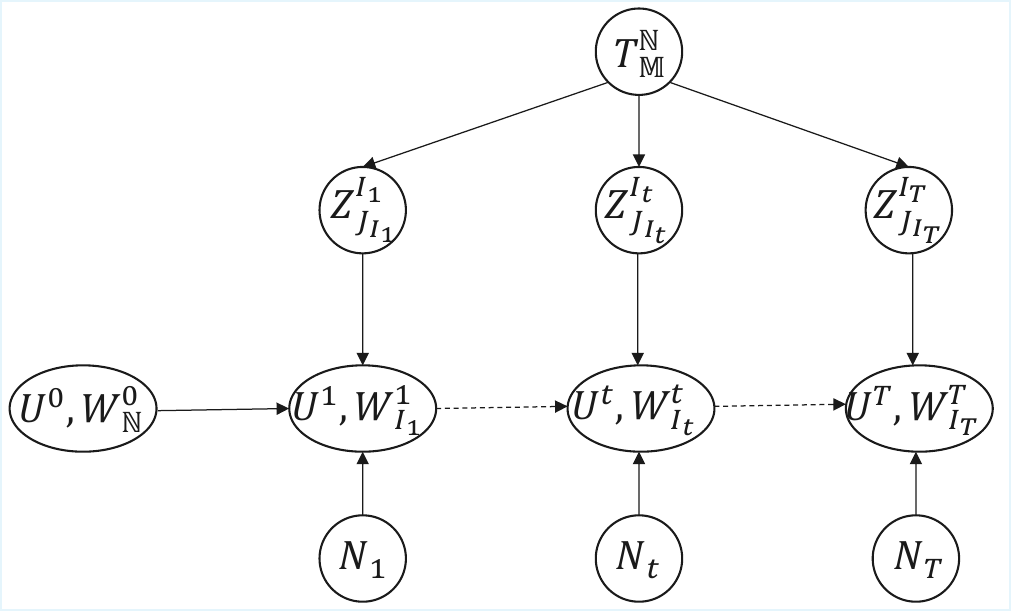}}
  \subfigure[Meta-learning with Separate In-task Training and Test Sets]{
    \label{bb} 
    \includegraphics[width=1.\linewidth]{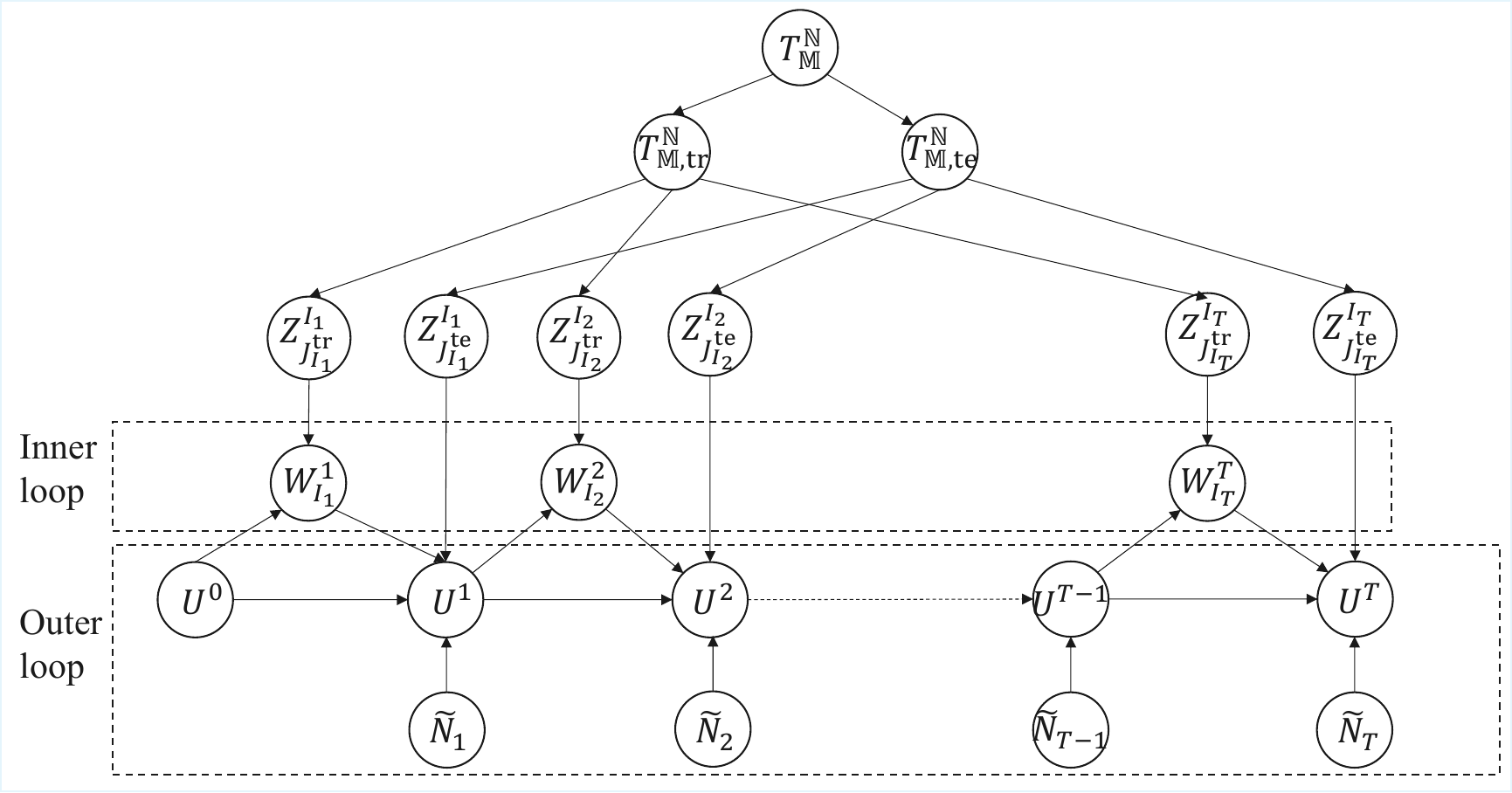}}
  \caption{A graphical representation of the meta-learning strategy through the noisy iterative approach.
}
\vskip -0.2in
  \label{fig_al} 
\end{figure}

\subsection{Algorithm-based Bound for Meta-learning with Joint In-task Training and Test Sets}
In joint meta-learning paradigm \citep{amit2018meta,chen2021generalization}, the meta-parameters $U$ and the task-specific parameters $W_{\mathbb{N}}=\{W_i\}_{i=1}^n$ are jointly updated within the entire meta-training dataset. We denote the training trajectory of SGLD-based meta-learning algorithms across $T$ iterations by $\{(U^t, W^t_{\mathbb{N}})\}_{t=0}^T$, where $U^{0}\in\mathbb{R}^d$ denotes the randomly initialized meta-parameter used for the initial iteration of the task-specific parameters, i.e., $W_1^0 =\cdots W_n^0 = U^0$. At $t$-th iteration, a batch of tasks is selected independently and randomly, indexed by $I_t\subseteq [n]$. For each task $i\in I_t$, we further randomly choose a batch of samples with indices $J_i\subseteq [m]$, and then use $J_{I_t}=\{J_i\}_{i\in I_t}$ to denote the collection of sample indices across the selected tasks. The updating rule with respect to meta and task-specific parameters is formulated by
\begin{equation*}
    (U^t, W^t_{I_t}) = (U^{t-1}, W^{t-1}_{I_t}) + \eta_t G^t + N_{t},
\end{equation*}
where $\eta_t$ is the learning rate, $N_{t}\sim N(0, \sigma^2_t \mathbf{1}_{(\vert I_t \vert +1 ) d   }) $ is the isotropic Gaussian noise injected in $t$-th iteration, $G^t = (G^t_U, G^t_{W_{ I_t}})$ is the average gradient for the batch computed by $G^t_{W_{ I_t}} = \{G^t_{W_{i}}\}_{i\in I_t}$,  $G^t_{W_{i}} = -\frac{1}{\vert J_{i} \vert} \sum_{j\in J_{i}} \nabla_W\ell(U^{t-1}, W^{t-1}_i,Z^i_j)$, and $G^t_U =  \frac{1}{\vert I_t \vert} \sum_{i\in I_t} G^t_{W_{i}}$.

The following theorem elucidates that the generalization error of iterative and noisy meta-learning algorithms can be characterized by the gradient covariance matrices of the determinant trajectory:
\begin{theorem}\label{theorem4.1}
  Assume that $\ell(\cdot,\cdot,\cdot)\in [0,1]$. For the algorithm output $(U, W_{\mathbb{N}})$ after $T$ iterations, the following bound holds:
    \begin{equation*}
      \vert \overline{\mathrm{gen}} \vert \leq \frac{1}{\sqrt{n m}} \sqrt{ \sum_{t=1}^T \frac{1}{2} \log \Big\vert \frac{\eta_t^2}{\sigma^2_t} \mathbb{E}_{U^{t-1},W^{t-1}_{\mathbb{N}}} [\Sigma_t]  + \mathbf{1}_{(\vert I_t \vert +1 ) d   } \Big\vert},
    \end{equation*}
where $\Sigma_t = \mathrm{Cov}_{T_{\mathbb{M}}^{\mathbb{N}},J_{I_t}}[G^t]$.
\end{theorem}
Theoren \ref{theorem4.1} provides a more precise characterization of the relationship between the meta-generalization error and the learning trajectory through the conditional gradient covariance matrix $\Sigma_t$, in contrast to the bounded gradient assumption $\sup_{U,W_{\mathbb{N}}, T_{\mathbb{M}}^{\mathbb{N}}} \Vert G^t \Vert_2\leq L$ with $L>0$ exploited in \citep{chen2021generalization}. This amount of gradient variance quantifies a particular ``sharpness'' of the loss landscape, highly associating with the true generalize dynamics of meta-learning, as supported by empirical evidence in \citep{jiang2019fantastic}. 


\subsection{Algorithm-based Bound for Meta-learning with Separate In-task Training and Test Sets}
Model-Agnostic Meta-Learning (MAML) \citep{finn2017model} is a well-known meta-learning approach that uses separate in-task training and test sets to update meta and task-specific parameters within a nested loop structure. Let $\{(U^t, W^t_{\mathbb{N}})\}_{t=0}^T$ denote the training trajectory of the noisy iterative MAML algorithm over $T$ iterations, where $U^0\in\mathbb{R}^d$ is the initial meta-parameter.  At $t$-th iteration, we randomly select a batch of task indices $I_t\subseteq [n]$ and a batch of sample indices $J_{I_t}=(J^{\mathrm{tr}}_{I_t}, J^{\mathrm{te}}_{I_t}) \subseteq [m]^{\vert I_t\vert}$ across the chosen tasks, where $J^{\mathrm{tr}}_{I_t}$ and $J^{\mathrm{te}}_{I_t}$ represent the index sets of the separate in-task training and test samples for all selected tasks, respectively. In the inner loop, the task-specific parameters are updated by
\begin{equation*}
    W_i^{t-1} = U^{t-1}, \quad W_i^t = W_i^{t-1} + \beta_t \tilde{G}^{\mathrm{tr},t}_{W_i} + \tilde{N}_t,\quad i\in I_t,
\end{equation*}
where $\tilde{G}^{\mathrm{tr},t}_{W_i} = -\frac{1}{\vert J^{\mathrm{tr}}_{i} \vert} \sum_{j\in J^{\mathrm{tr}}_{i}} \nabla_W\ell(U^{t-1}, W^{t-1}_i,Z^i_j)$ is the average gradient for the batch over the separate in-task training samples, $\beta_t$ is the learning rate for task parameter, $\tilde{N}_t\sim N(0, \sigma^2_t \mathbf{1}_{d} )$ is an isotropic Gaussian noise. In the outer loop, the meta parameters are updated by 
\begin{equation*}
    U^{t}  = U^{t-1} + \eta_t \tilde{G}^t_{U} + \tilde{N}_t,
\end{equation*}
where $\tilde{G}^t_{U}= \frac{1}{\vert I_t \vert} \sum_{i\in I_t} \tilde{G}^{\mathrm{te},t}_{W_i}$ and $\tilde{G}^{\mathrm{te},t}_{W_i} = -\frac{1}{\vert J^{\mathrm{te}}_{i} \vert} \sum_{j\in J^{\mathrm{te}}_{i}}\nabla_W \ell(U^{t-1}, W^{t}_i,Z^i_j)$ is the average gradient for the batch over the separate in-task test samples, and $\eta_t$ is the meta learning rate.

Let $\tilde{G}^{\mathrm{tr},t}_{W_{I_t}}=\{\tilde{G}^{\mathrm{tr},t}_{W_i}\}_{i\in I_t}$ and $\tilde{G}^{\mathrm{te},t}_{W_{I_t}}=\{\tilde{G}^{\mathrm{te},t}_{W_i}\}_{i\in I_t}$. In the following theorem, we establish the generalization bound for the noisy, iterative MAML algorithm, offering a theoretical insight into the generalization properties of meta-learning with in-task sample partitioning:
\begin{theorem}\label{theorem4.2}
    Let $T_{\mathbb{M}}^{\mathbb{N}}=\{T_{\mathbb{M},\mathrm{tr}}^{\mathbb{N}}, T_{\mathbb{M},\mathrm{te}}^{\mathbb{N}}\}$ consist of separate in-task training and test datasets, where $\vert T_{\mathbb{M},\mathrm{tr}}^{\mathbb{N}}\vert = m^{\mathrm{tr}}$, $\vert T_{\mathbb{M},\mathrm{te}}^{\mathbb{N}}\vert = m^{\mathrm{te}}$, and $m^{\mathrm{tr}} + m^{\mathrm{te}} =m$. Assume that $\ell(\cdot,\cdot,\cdot)\in [0,1]$, then 
    \begin{align*}
        \vert \overline{\mathrm{gen}} \vert \leq  \frac{1}{\sqrt{nm^{\mathrm{te}}}}\sqrt{\sum_{t=1}^T \log \Big\vert \frac{\beta_t^2}{\sigma^2_t} \mathbb{E}_{U^{t-1}} [\Sigma^{\mathrm{tr}}_t]  + \mathbf{1}_{\vert I_t \vert d   } \Big\vert  +  \sum_{t=1}^T \log \Big\vert \frac{\eta_t^2}{\sigma^2_t} \mathbb{E}_{U^{t-1}, W^t_{I_t}} [\Sigma^{\mathrm{te}}_t]  + \mathbf{1}_{d} \Big\vert },
     \end{align*}
where  $\Sigma^{\mathrm{te}}_t = \mathrm{Cov}_{T_{\mathbb{M},\mathrm{te}}^{\mathbb{N}}, J_{I_t}^{\mathrm{te}}}[\frac{1}{\vert I_t\vert} \sum_{i\in I_t} \tilde{G}^{\mathrm{te},t}_{W_{i}}]$ and $\Sigma^{\mathrm{tr}}_t = \mathrm{Cov}_{T_{\mathbb{M},\mathrm{tr}}^{\mathbb{N}}, J_{I_t}^{\mathrm{tr}}}[\tilde{G}^{\mathrm{tr},t}_{W_{I_t}}]$.
\end{theorem}

The generalization bound above decreases monotonically at the rate of $\mathcal{O}(1/nm^{\mathrm{te}})$ as task size $n$ and separate in-task test data size $m^{\mathrm{te}}$ increase. This result underscores the importance of the in-task train-test split for achieving a smaller generalization error of the MAML algorithm. Furthermore, it is essential to simultaneously ensure small gradient covariances $\mathbb{E}_{U^{t-1}} [\Sigma^{\mathrm{tr}}_t]$ and $\mathbb{E}_{U^{t-1}, W^t_{I_t}} [\Sigma^{\mathrm{te}}_t]$ of task-specific parameters and meta-parameters associated with the separate in-task training and test samples, thereby achieving good generalization performance.

\section{Experiments}\label{AppendixA}

\begin{figure}[t]
  \vskip 0.2in
  \centering
  \subfigure[ Task size $n=2$]{
    \label{aa} 
    \includegraphics[width=65mm]{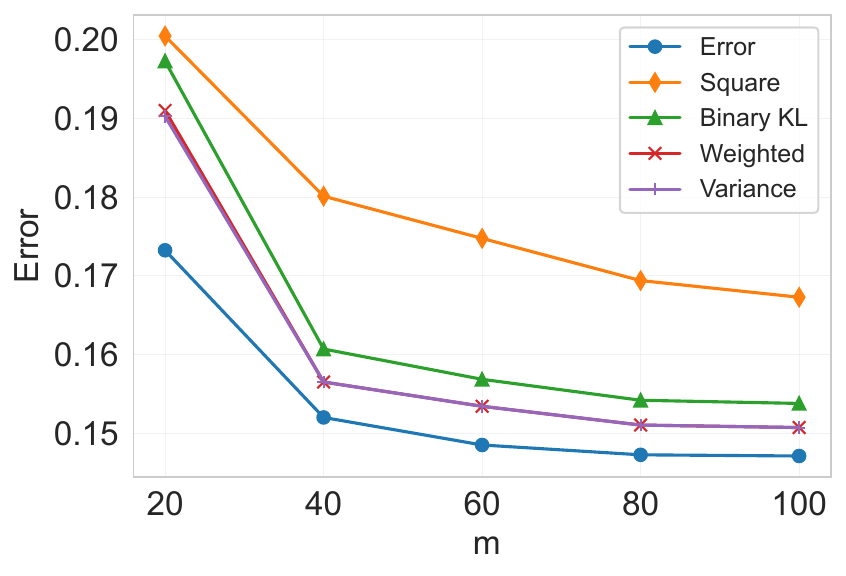}}
  \subfigure[Task size $n=3$]{
    \label{bb} 
    \includegraphics[width=65mm]{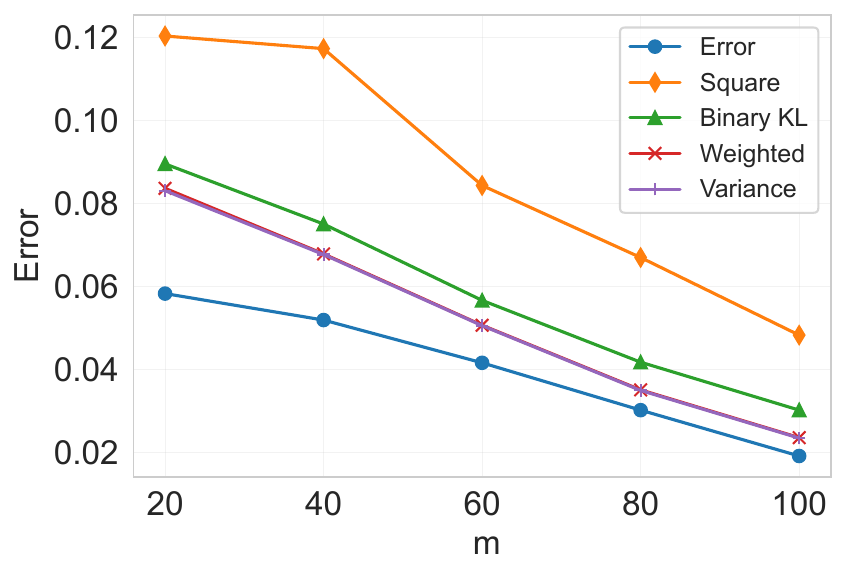}}
    \subfigure[Task size $n=4$]{
    \label{cc} 
    \includegraphics[width=65mm]{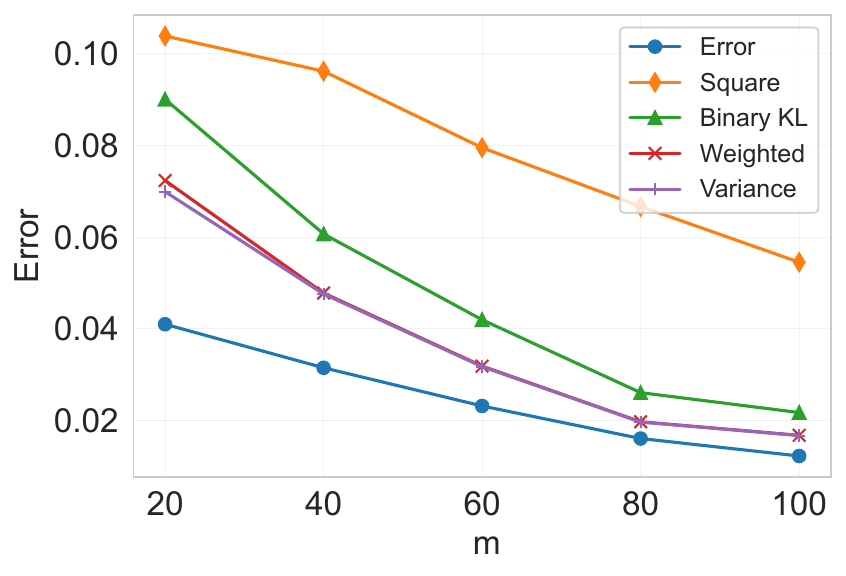}}
  \subfigure[Task size $n=5$]{
    \label{dd} 
    \includegraphics[width=65mm]{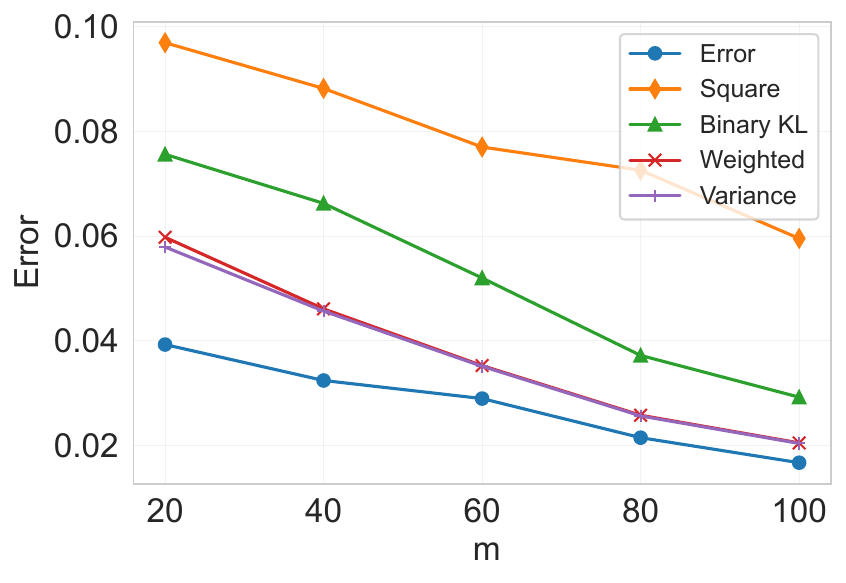}}
  \caption{Comparison of the meta-generalization bounds on synthetic Gaussian datasets.
}
\vskip -0.2in
  \label{fig1} 
\end{figure}

\begin{figure*}[t]
  \vskip 0.2in
  \centering
  \subfigure[MNIST, Adam ($n=10$)]{
    \label{aa} 
    \includegraphics[width=65mm]{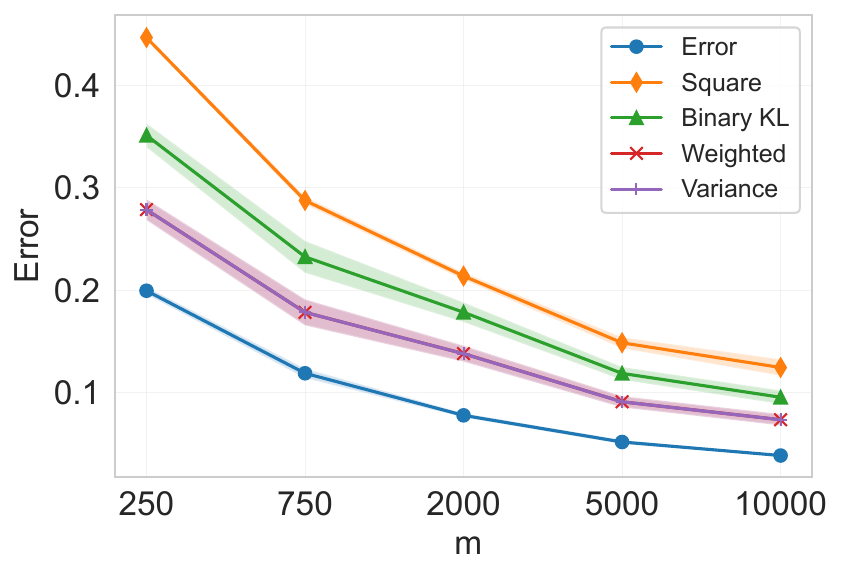}}
  \subfigure[MNIST, SGLD ($n=10$)]{
    \label{bb} 
    \includegraphics[width=65mm]{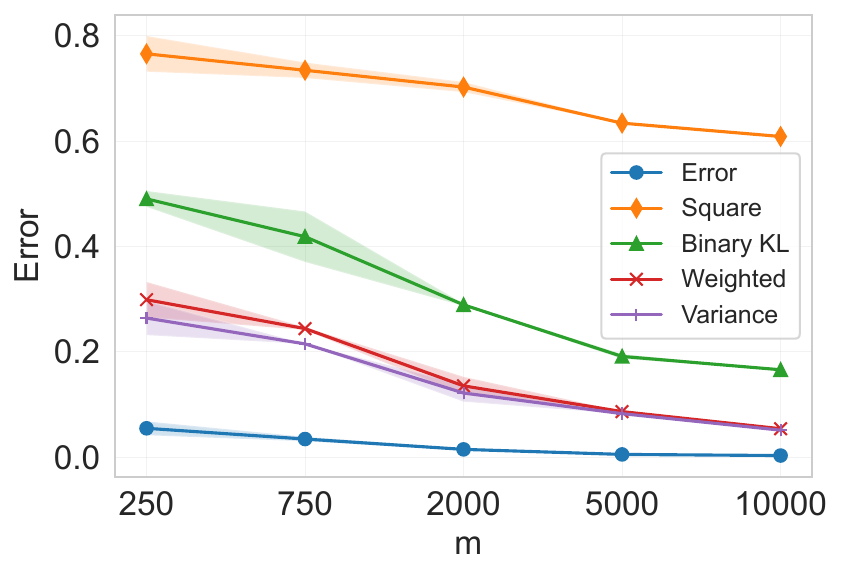}}
    \subfigure[CIFAR-10, SGD ($n=10$)]{
    \label{cc} 
    \includegraphics[width=65mm]{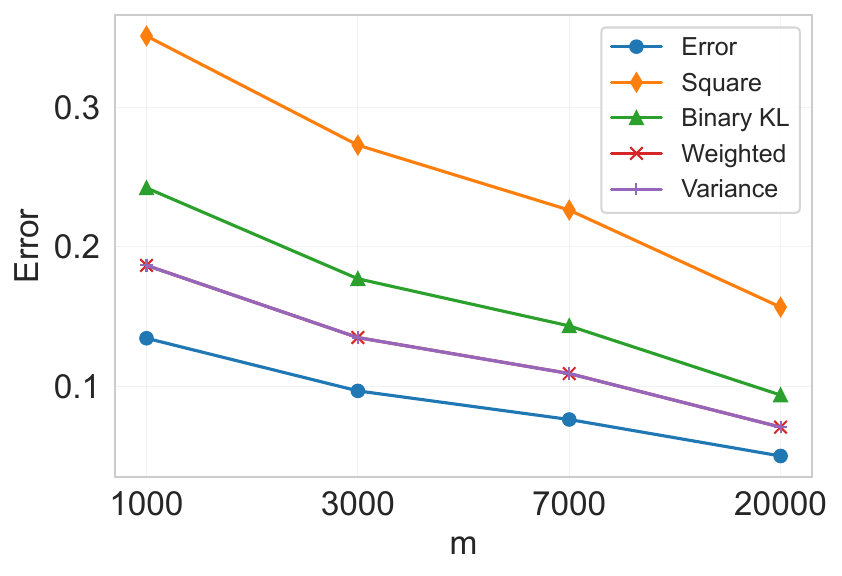}}
  \subfigure[CIFAR-10, SGLD ($n=10$)]{
    \label{dd} 
    \includegraphics[width=65mm]{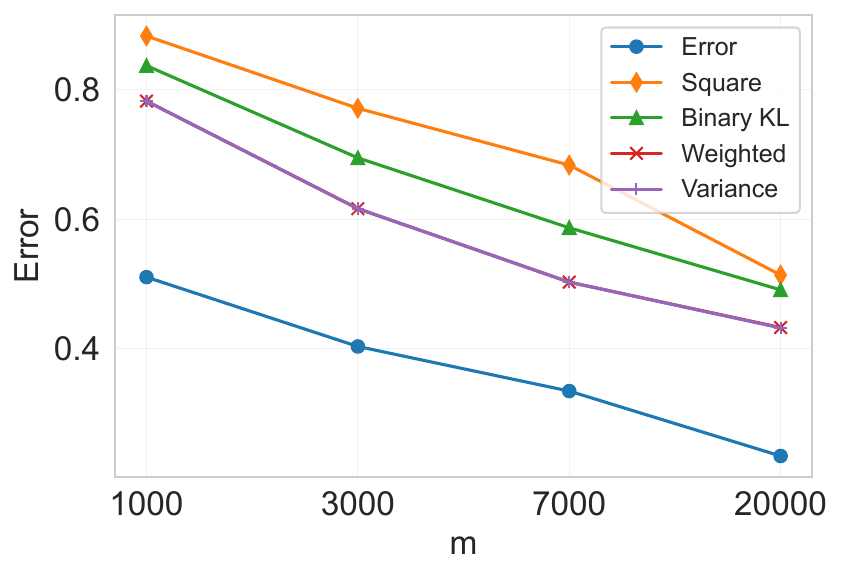}}
  \caption{Comparison of the meta-generalization bounds on real-world datasets with different optimizers.
}
\vskip -0.2in
  \label{fig2} 
\end{figure*}

\begin{figure*}[t]
    \vskip 0.2in
    \centering
    \subfigure[MNIST ($n=10$)]{
      \label{aa} 
      \includegraphics[width=65mm]{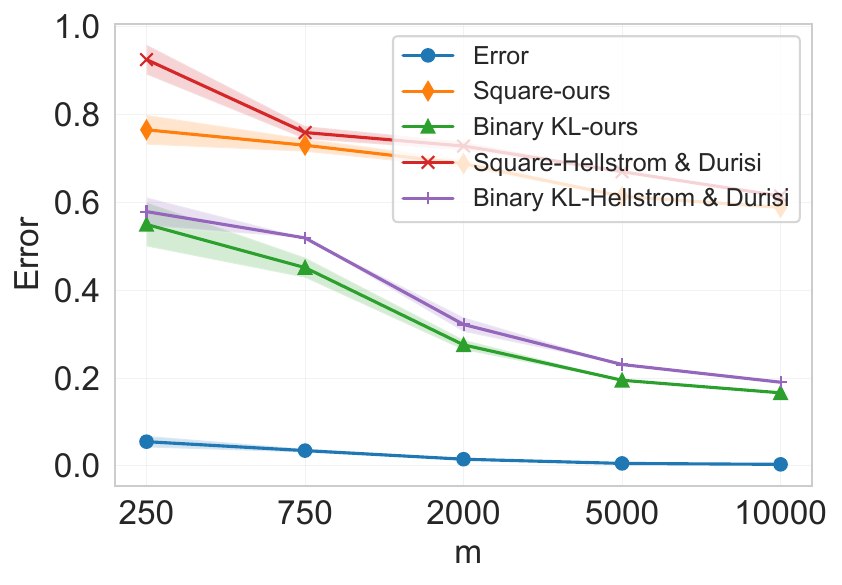}}
      \subfigure[CIFAR-10  ($n=10$)]{
      \label{cc} 
      \includegraphics[width=65mm]{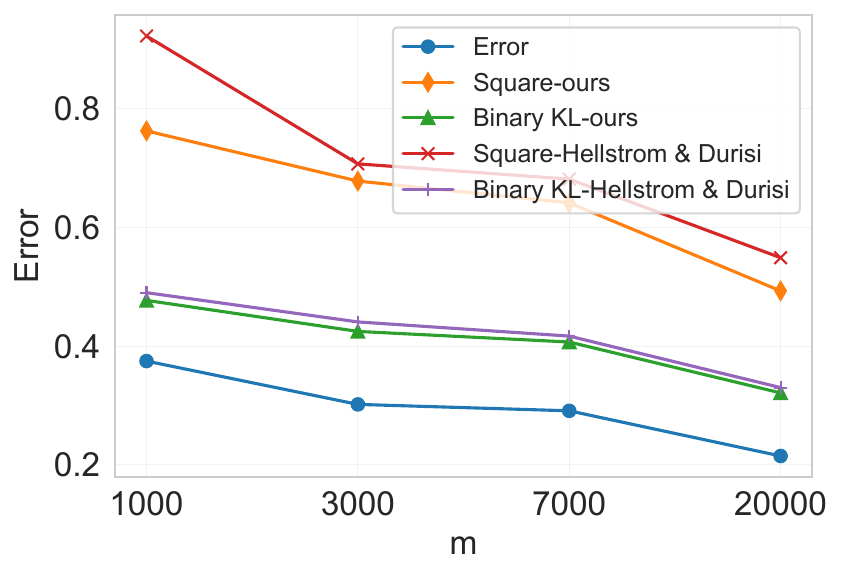}}
   
    \caption{Comparison of the meta-generalization bounds on real-world datasets.
  }
  \vskip -0.2in
    \label{fig3} 
  \end{figure*}

This section presents an empirical evaluation of our theoretical results by comparing the generalization gap with various bounds derived in Section \ref{Section3}. Considering computational feasibility, we focus primarily on the loss-based bounds, including the binary KL bound (Theorem \ref{theorem6}), the square-root bound (Theorem \ref{theorem8}), and the fast-rate bounds (Theorems \ref{theorem10} and \ref{theorem12}). In our experiments, we adopt a binary loss function to measure both empirical and population risks. The experimental setup follows the configuration used in prior work \citep{wang2023tighter} to ensure a consistent and precise quantification of the bounds. All experiments were conducted on a server equipped with an Intel Xeon CPU (2.10GHz, 48 cores), 256GB of memory, and 4 Nvidia Tesla V100 GPUs (32GB each).

\subsection{Experimental Setup}\label{section6.1}

\paragraph{Synthetic Experiments.} We adopt the experimental settings outlined in \citep{wang2023tighter} to generate synthetic Gaussian datasets using the scikit-learn Python package. Class centers for each dataset are positioned at the vertices of a hypercube. Data points are then independently drawn from isotropic Gaussian distributions with a standard deviation of 0.25. A four-layer MLP neural network with ReLU activation serves as the base model, and the generalization performance is evaluated using the binary 0–1 loss.
\paragraph{Real-world Experiments.} Following the experiment settings in \citep{harutyunyan2021information,hellstrom2022new}, we consider four distinct real-world scenarios to evaluate the derived meta-generalization bounds: 1) MNIST classification using Adam, 2) MNIST classification using SGLD, 3) CIFAR-10 classification with fine-tuned ResNet-50 using SGD, 4)  CIFAR-10 classification with fine-tuned ResNet-50 using SGLD. In meta-learning, we consider each category in the dataset as a distinct learning task, where the samples belonging to that category are regarded as in-task data used for task-specific training and evaluation. For each learning task, we sample $t_1$ instances of $T^{2\mathbb{N}}_{2\mathbb{M}}$, by randomly selecting $4nm$ samples from the individual datasets. For each $T^{2\mathbb{N}}_{2\mathbb{M}}$, we further draw $t_2$ samples of the supersample variables $\tilde{S}_{\mathbb{N}},S_{\mathbb{M}}$, culminating in $t_1\times t_2$ independent runs in total.

\subsection{Numerical Results}

\paragraph{Synthetic Datasets.} We begin with $n$-class classification tasks within meta-learning settings, where each class containing $m$ training samples is regarded as a meta-learning training task. Synthetic Gaussian datasets are generated as described in Section \ref{section6.1}. Figure \ref{fig1} shows the empirical generalization gap and the corresponding theoretical bounds on synthetic Gaussian datasets using a simple MLP network. As shown in Figure \ref{fig1}, our bounds adapt effectively to varying numbers $n$ of meta-training tasks and sizes $m$ of per-task samples, coinciding well with the convergence trend of meta-generalization gap: the bounds decrease as $n\times m$ increases. A comparison among these bounds illustrates that the fast-rate bounds serve as the tightest estimates of the generalization gap.

\paragraph{Real-world Datasets} We employ the real-world configurations outlined in Section \ref{section6.1}, covering four learning scenarios with different datasets and optimization methods. Figure \ref{fig2} illustrates the closeness of the generalization gap to various bounds, exhibiting the same trend regarding the size of the training dataset. Notably, the weighted and loss variance bounds have been shown to be the tightest among the upper bounds, aligning with the analysis of their stringency. Additionally, we empirically compare the derived bounds with those established by \citep{hellstrom2022evaluated} on real-world datasets. As shown in Figure \ref{fig3}, both the square root bound and the binary KL bound demonstrate a closer alignment with the true generalization error compared to previous work, thereby underscoring the advantage of using fewer meta-supersample variables within our meta-supersample framework.

\section{Conclusion}
This paper has presented a unified generalization analysis for meta-learning by developing various information-theoretic bounds. The resulting bounds improve upon existing results in terms of tightness, scaling rate, and computational feasibility. Our analysis is modular and broadly applicable to a wide range of meta-learning algorithms. In particular, we establish algorithm-dependent bounds for two widely adopted meta-learning methods, with a focus on MAML and Reptile. Numerical results on both synthetic and real-world datasets demonstrate that our theoretical bounds closely track the true generalization gap, validating their practical relevance.


\newpage

\appendix

\section{Preparatory Definitions and Lemmas}
\begin{definition}[$\sigma$-sub-gaussian]
    A random variable $X$ is $\sigma$-sub-gaussian if for any $\lambda$, $\ln \mathbb{E}[e^{\lambda (X-\mathbb{E}X) }]\leq \lambda^2 \sigma^2/2$.
\end{definition}
\begin{definition}[Binary Relative Entropy]
    Let $p,q\in[0,1]$. Then $d(p\Vert q)$ denotes the relative entropy between two Bernoulli random variables with parameters $p$ and $q$ respectively, defined as $d(p\Vert q)=p\log(\frac{p}{q})+(1-p)\log(\frac{1-p}{1-q})$. Given $\gamma\in\mathbb{R}$, $d_\gamma(p\Vert q)=\gamma p-\log(1-q+qe^{\gamma})$ is the relaxed version of binary relative entropy. One can prove that $\sup_\gamma d_\gamma(p\Vert q)=d(p\Vert q)$.
\end{definition}
\begin{definition}[Kullback-Leibler Divergence]
 Let $P$ and $Q$ be probability distributions defined on the same measurable space such that $P$ is absolutely continuous with respect to $Q$. The Kullback-Leibler (KL) divergence between $P$ and $Q$ is defined as $D(P\Vert Q)\triangleq \int_{\mathcal{X}}p(x)\log (\frac{p(x)}{q(x)})dx$.
\end{definition}
\begin{definition}[Mutual Information]
    For random variables $X$ and $Y$ with joint distribution $P_{X,Y}$ and product of their marginals $P_XP_Y$, the mutual information between $X$ and $Y$ is defined as $I(X;Y) = D(P_{X,Y} \Vert P_XP_Y)$.
\end{definition}
\begin{lemma}[Donsker-Varadhan Formula (Theorem 5.2.1 in \cite{gray2011entropy})] \label{lemmaA.5}
    Let $P$ and $Q$ be probability measures over the same space $\mathcal{X}$ such that $P$ is absolutely continuous with respect to $Q$. For any bounded function $f:\mathcal{X}\rightarrow \mathbb{R}$, 
    \begin{equation*}
        D(P\Vert Q) = \sup_{f} \Big\{ \mathbb{E}_{X\sim P}[f(X)] - \log \mathbb{E}_{X\sim Q} [e^{f(X)}] \Big\},
    \end{equation*}
where $X$ is any random variable such that both $\mathbb{E}_{X\sim P}[f(X)]$ and $\mathbb{E}_{X\sim Q} [e^{f(X)}]$ exist.
\end{lemma}
\begin{lemma}[Lemma 1 in \cite{harutyunyan2021information}]\label{lemmaA.6}
    Let $(X,Y)$ be a pair of random variables with joint distribution $P_{X,Y}$, and $\bar{Y}$ be an independent copy of $Y$. If $f(x,y)$ be a measurable function such that $\mathbb{E}_{X,Y}[f(X,Y)]$ exists and $\mathbb{E}_{X,\bar{Y}}[f(X,\bar{Y})]$ is $\sigma$-sub-gaussian, then 
    \begin{equation*}
        \big\vert \mathbb{E}_{X,Y} [f(X,Y)] - \mathbb{E}_{X,\bar{Y}}[f(X,\bar{Y})] \big\vert \leq \sqrt{2\sigma^2 I(X,Y)}.
    \end{equation*}
\end{lemma}
\begin{lemma}[Lemma 3 in \cite{harutyunyan2021information}] \label{lemmaA.7}
    Let $X$ and $Y$ be independent random variables. If $f$ is a measurable function such that $f(x,Y)$ is $\sigma$-sub-gaussian and $\mathbb{E}f(x,Y)=0$ for all $x\in\mathcal{X}$, then $f(X,Y)$ is also $\sigma$-sub-gaussian.
\end{lemma}
\begin{lemma}[Lemma 2 in \cite{hellstrom2022evaluated}] \label{lemmaA.8}
    Let $X_1,\ldots,X_n$ be $n$ independent random variables that for $i\in[n]$, $X_i\sim P_{X_i}$, $\mathbb{E}[X_i]=\mu_i$, $\bar{\mu}=\frac{1}{n}\sum_{i=1}^{n}X_i$, and $\mu = \frac{1}{n}\sum_{i=1}^{n}\mu_i$. Assume that $X_i\in[0,1]$ almost surely. Then, for any $\gamma>0$, $\mathbb{E}[e^{(n d_\gamma(\bar{\mu} \Vert \mu))}]\leq 1$.
\end{lemma}
\begin{lemma}[Lemma A.11 in \cite{dongtowards}] \label{lemmaA.9}
    Let $X\sim N(0,\Sigma)$ and $Y$ be any zero-mean random vector satisfying $\mathrm{Cov}_Y[Y] = \Sigma$, then $H(Y)\leq H(X)$.
\end{lemma}

\begin{lemma}[Lemma 9 in \cite{dong2023understanding}] \label{lemmaA.10}
    For any symmetric positive-definite matrix $A$, let $A = \Big[\begin{matrix}  B & D^T \\ D & C \end{matrix}\Big] $ be a partition of $A$, where $B$ and $C$ are square matrices, then $\vert A \vert \leq \vert B \vert \vert C \vert$.
\end{lemma}

\section{Omitted Proofs [Input-Output MI Bounds]}

\subsection{Proof of Theorem \ref{theorem3.1}}
\begin{restatetheorem}{\ref{theorem3.1}}[Restate]
    Let $\mathbb{K}$ and $\mathbb{J}$ be random subsets of $[n]$ and $[m]$ with sizes $\zeta$ and $\xi$, respectively, independent of $T^{\mathbb{N}}_{\mathbb{M}}$ and $(U,W_{\mathbb{N}})$. Assume that $\ell(u,w,Z)$ is $\sigma$-sub-gaussian with respect to $Z\sim P_{Z|\tau}$ and $\tau\sim P_{\mathcal{T}}$ for all $u,w$, then
    \begin{equation*}
       \vert \overline{\mathrm{gen}} \vert \leq \mathbb{E}_{K\sim \mathbb{K}, J\sim\mathbb{J}} \sqrt{\frac{2\sigma^2}{\zeta \xi} I(U,W_K;T^{K}_{J})}.  
    \end{equation*}
\end{restatetheorem}
\begin{proof}
    Let random subsets $\mathbb{K},\mathbb{J}$ be fixed to $\mathbb{K}=K,\mathbb{J}=J$ with size $\zeta$ and $\xi$, and independent of $(U,W_{\mathbb{N}})$ and $T^{\mathbb{N}}_{\mathbb{M}}$. Let $\tau_K=\{\tau_i\}_{i\in K}$, $T^{K}_{J}=\{Z^i_j\}_{i\in K,j\in J} \subseteq T^{\mathbb{N}}_{\mathbb{M}}$, $W_K = \{W_i\}_{i\in K}$, and 
        \begin{equation*}
        f(U,W_K, T^{K}_{J}) = \frac{1}{\zeta} \sum_{i\in K} \frac{1}{\xi}\sum_{j\in J}\ell(U, W_i, Z^i_j).
    \end{equation*}
Let $\bar{U}$ and $\bar{T}^{\mathbb{N}}_{\mathbb{M}}$ be independent copy of $U$ and $T^{\mathbb{N}}_{\mathbb{M}}$.  Applying Lemma \ref{lemmaA.5} with $P= P_{\tau_K,T^{K}_{J}}P_{U|T^{K}_{J}}P_{W_{K}|U,T^{K}_{J}}$, $Q=  P_{\tau_K,T^{K}_{J}} P_{\bar{U}}P_{W_K|U,T^{K}_{J}}P_{\bar{T}^{K}_{J}|\tau_K}$, and $f = f(U,W_K, T^{K}_{J})$, we get that 
\begin{align}
    &  D(P_{\tau_K,T^{K}_{J}}P_{U|T^{K}_{J}}P_{W_{K}|U,T^{K}_{J}} \Vert P_{\tau_K,T^{K}_{J}} P_{\bar{U}}P_{W_K|U,T^{K}_{J}}P_{\bar{T}^{K}_{J}|\tau_K})\nonumber\\
    \geq & \sup_{\lambda}\Big\{\lambda\Big(\mathbb{E}_{\tau_K,T^{K}_{J},U,W_{K}} [f(U,W_K, T^{K}_{J})] - \mathbb{E}_{\tau_K,T^{K}_{J}}\mathbb{E}_{\bar{U}, \bar{T}^{K}_{J}|\tau_K}\mathbb{E}_{W_K|U,T^{K}_{J}} [f(\bar{U},W_K, \bar{T}^{K}_{J})] \Big)\nonumber\\
    & -  \log\mathbb{E}_{\tau_K,T^{K}_{J}}\mathbb{E}_{\bar{U}, \bar{T}^{K}_{J}|\tau_K}\mathbb{E}_{W_K|U,T^{K}_{J}} \Big[e^{\lambda \big(f(\bar{U},W_K, \bar{T}^{K}_{J}) - \mathbb{E}[f(\bar{U},W_K, \bar{T}^{K}_{J})]\big)} \Big] \Big\}, \label{equ4}
\end{align}
where 
\begin{align}
    &D(P_{\tau_K,T^{K}_{J}}P_{U|T^{K}_{J}}P_{W_{K}|U,T^{K}_{J}} \Vert P_{\tau_K,T^{K}_{J}} P_{\bar{U}}P_{W_K|U,T^{K}_{J}}P_{\bar{T}^{K}_{J}|\tau_K}) \nonumber\\
     =& D(P_{\tau_K,T^{K}_{J}}P_{U|T^{K}_{J}} \Vert P_{\bar{U}}P_{\bar{T}^{K}_{J}|\tau_K} ) + D(P_{W_K|U,T^{K}_{J}}\Vert P_{\bar{W}_K|U,\tau_K} | P_{\tau_K,T^{K}_{J}}P_{U}) \nonumber\\
    =& I(U;T^{K}_{J}) + I(W_K;T^{K}_{J}|\tau_K,U) \label{equ5}
\end{align}
By the sub-gaussian property of the loss function, it is clear that the random variable $f(\bar{U},W_K, \bar{T}^{K}_{J})$ is $\frac{\sigma}{\sqrt{\zeta\xi}}$-sub-gaussian, which implies that 
\begin{align*}
    \log\mathbb{E}_{\tau_K,T^{K}_{J}}\mathbb{E}_{\bar{U}, \bar{T}^{K}_{J}|\tau_K}\mathbb{E}_{W|U,T^{K}_{J}} \Big[e^{\lambda \big(f(\bar{U},W_K, \bar{T}^{K}_{J}) - \mathbb{E}[f(\bar{U},W_K, \bar{T}^{K}_{J})]\big)} \Big] \leq \frac{\lambda^2\sigma^2}{2\xi\zeta}.
\end{align*}
Putting the above back into (\ref{equ4}) and combining with (\ref{equ5}), we have
\begin{align*}
    &  D(P_{\tau_K,T^{K}_{J}}P_{U|T^{K}_{J}}P_{W_{K}|U,T^{K}_{J}} \Vert P_{\tau_K,T^{K}_{J}} P_{\bar{U}}P_{W|U,T^{K}_{J}}P_{\bar{T}^{K}_{J}|\tau_K}) \\
    = & I(U;T^{K}_{J}) + I(W_K;T^{K}_{J}|\tau_K,U) \\
    \geq& \sup_{\lambda}\Big\{\lambda\Big(\mathbb{E}_{\tau_K,T^{K}_{J},U,W_{K}} [f(U,W_K, T^{K}_{J})] - \mathbb{E}_{\tau_K,T^{K}_{J}}\mathbb{E}_{\bar{U}, \bar{T}^{K}_{J}|\tau_K}\mathbb{E}_{W|U,T^{K}_{J}} [f(\bar{U},W_K, \bar{T}^{K}_{J})] \Big) - \frac{\lambda^2\sigma^2}{2\xi\zeta} \Big\}. 
\end{align*}
Solving $\lambda$ to maximize the RHS of the above inequality, we get that 
\begin{align*}
    &\Big\vert \mathbb{E}_{\tau_K,T^{K}_{J},U,W_{K}} [f(U,W_K, T^{K}_{J})] - \mathbb{E}_{\tau_K,T^{K}_{J}}\mathbb{E}_{\bar{U}, \bar{T}^{K}_{J}|\tau_K}\mathbb{E}_{W|U,T^{K}_{J}} [f(\bar{U},W_K, \bar{T}^{K}_{J})]   \Big\vert \\
    = & \bigg\vert \mathbb{E}_{\tau_K,T^{K}_{J},U,W_{K}} \Big[ \frac{1}{\zeta} \sum_{i\in K} \frac{1}{\xi}\sum_{j\in J}\ell(U, W_i, Z^i_j)\Big] - \mathbb{E}_{\tau_K,T^{K}_{J}}\mathbb{E}_{\bar{U}, \bar{T}^{K}_{J}|\tau_K}\mathbb{E}_{W|U,T^{K}_{J}} \Big[\frac{1}{\zeta} \sum_{i\in K} \frac{1}{\xi}\sum_{j\in J}\ell(\bar{U}, W_i, \bar{Z}^i_j) \Big]   \bigg\vert \\
    \leq & \sqrt{\frac{2\sigma^2}{\zeta \xi} \Big(I(U;T^{K}_{J}) + I(W_K;T^{K}_{J}|\tau_K,U)\Big)}  . 
\end{align*}
Taking expectation over $K,J$ on both sides and applying Jensen's inequality on the absolute value function, we have 
\begin{align}
    & \bigg\vert \mathbb{E}_{K\sim \mathbb{K}, J\sim\mathbb{J}}\bigg[ \mathbb{E}_{\tau_K,T^{K}_{J},U,W_{K}} \Big[ \frac{1}{\zeta} \sum_{i\in K} \frac{1}{\xi}\sum_{j\in J}\ell(U, W_i, Z^i_j)\Big] - \mathbb{E}_{\tau_K,T^{K}_{J}}\mathbb{E}_{\bar{U}, \bar{T}^{K}_{J}|\tau_K}\mathbb{E}_{W|U,T^{K}_{J}} \Big[\frac{1}{\zeta} \sum_{i\in K} \frac{1}{\xi}\sum_{j\in J}\ell(\bar{U}, W_i, \bar{Z}^i_j) \Big]  \bigg] \bigg\vert \nonumber\\
    = & \bigg\vert  \mathbb{E}_{\tau_\mathbb{N},T^{\mathbb{N}}_{\mathbb{M}},U,W_\mathbb{N}} \Big[ \frac{1}{nm}\sum_{i,j=1}^{n,m}\ell(U, W_i, Z^i_j)\Big] - \mathbb{E}_{\tau_\mathbb{N},T^{\mathbb{N}}_{\mathbb{M}}}\mathbb{E}_{\bar{U}, \bar{T}^{\mathbb{N}}_{\mathbb{M}}|\tau_\mathbb{N}}\mathbb{E}_{W|U,T^{\mathbb{N}}_{\mathbb{M}}} \Big[\frac{1}{nm}\sum_{i,j=1}^{n,m}\ell(\bar{U}, W_i, \bar{Z}^i_j) \Big]  \bigg\vert \nonumber\\
     \leq &\mathbb{E}_{K\sim \mathbb{K}, J\sim\mathbb{J}} \sqrt{\frac{2\sigma^2}{\zeta \xi} \Big(I(U;T^{K}_{J}) + I(W_K;T^{K}_{J}|\tau_K,U)\Big)}  . \label{equ6}
 \end{align}
Notice that 
\begin{equation*}
    I(W_K;T^{K}_{J}|\tau_K,U)\leq I(W_K;T^{K}_{J}|\tau_K,U) + I(W_K;\tau_K|U) =I(W_K;T^{K}_{J},\tau_K|U) = I(W_K;T^{K}_{J}|U).
\end{equation*}
Putting the above inequality back into (\ref{equ6}), we obtain that 
\begin{align}
    & \bigg\vert  \mathbb{E}_{\tau_\mathbb{N},T^{\mathbb{N}}_{\mathbb{M}},U,W_\mathbb{N}} \Big[ \frac{1}{nm}\sum_{i,j=1}^{n,m}\ell(U, W_i, Z^i_j)\Big] - \mathbb{E}_{\tau_\mathbb{N},T^{\mathbb{N}}_{\mathbb{M}}}\mathbb{E}_{\bar{U}, \bar{T}^{\mathbb{N}}_{\mathbb{M}}|\tau_\mathbb{N}}\mathbb{E}_{W|U,T^{\mathbb{N}}_{\mathbb{M}}} \Big[\frac{1}{nm}\sum_{i,j=1}^{n,m}\ell(\bar{U}, W_i, \bar{Z}^i_j) \Big]  \bigg\vert \nonumber\\
    \leq &\mathbb{E}_{K\sim \mathbb{K}, J\sim\mathbb{J}} \sqrt{\frac{2\sigma^2}{\zeta \xi} \Big(I(U;T^{K}_{J}) + I(W_K;T^{K}_{J}|\tau_K,U)\Big)}  \nonumber\\
    \leq &  \mathbb{E}_{K\sim \mathbb{K}, J\sim\mathbb{J}} \sqrt{\frac{2\sigma^2}{\zeta \xi} \Big(I(U;T^{K}_{J}) + I(W_K;T^{K}_{J}|U)\Big)} \nonumber\\
    = &  \mathbb{E}_{K\sim \mathbb{K}, J\sim\mathbb{J}} \sqrt{\frac{2\sigma^2}{\zeta \xi} I(U,W_K;T^{K}_{J})}. \label{equ7}
\end{align}
We proceed to prove that the LSH of the inequality (\ref{equ7}) is equivalent to the absolute value of the meta-generalization gap. Notice that $W_i$, $i=1,\ldots,n$ are mutually independent given $U$ and $T^{\mathbb{N}}_{\mathbb{M}}$,  $P_{W_{\mathbb{N}}|U,T^{\mathbb{N}}_{\mathbb{M}}}=\prod_{i=1}^n P_{W_i|U,T^{i}_{\mathbb{M}}}$, we have
\begin{align*}
      \mathbb{E}_{\tau_\mathbb{N},T^{\mathbb{N}}_{\mathbb{M}},U,W_\mathbb{N}}  \bigg[\frac{1}{n}\sum_{i=1}^{n} \frac{1}{m}\sum_{j=1}^{m} \ell(U,W_{i},Z^i_j)\bigg]  = & \mathbb{E}_{U,T^{\mathbb{N}}_{\mathbb{M}}}\bigg[\frac{1}{n}\sum_{i=1}^{n} \frac{1}{m}\sum_{j=1}^{m} \mathbb{E}_{W_\mathbb{N}|U,T^{\mathbb{N}}_{\mathbb{M}}} [\ell(U,W_{i},Z^i_j)] \bigg] \\
    = & \mathbb{E}_{U,T^{\mathbb{N}}_{\mathbb{M}}} \bigg[\frac{1}{n}\sum_{i=1}^{n} \frac{1}{m}\sum_{j=1}^{m} \mathbb{E}_{W_{i}|U,T^{i}_{\mathbb{M}}} \ell(U,W_{i},Z^i_j) \bigg] \\
    = & \mathbb{E}_{U,T^{\mathbb{N}}_{\mathbb{M}}}[\mathcal{R}(U, T^{\mathbb{N}}_{\mathbb{M}})].
\end{align*}
For the second term on LSH of the inequality (\ref{equ7}), we have
\begin{align*}
    & \mathbb{E}_{\tau_\mathbb{N},T^{\mathbb{N}}_{\mathbb{M}}}\mathbb{E}_{\bar{U}}\mathbb{E}_{W|U,T^{\mathbb{N}}_{\mathbb{M}}}\mathbb{E}_{\bar{T}^{\mathbb{N}}_{\mathbb{M}}|\tau_\mathbb{N}}\Big[\frac{1}{n}\sum_{i=1}^{n} \frac{1}{m}\sum_{j=1}^{m} \ell(\bar{U},W_{i},\bar{Z}^i_j)\Big] \\
    = & \mathbb{E}_{\bar{U},T^{\mathbb{N}}_{\mathbb{M}}}\bigg[\frac{1}{n}\sum_{i=1}^{n} \frac{1}{m}\sum_{j=1}^{m} \mathbb{E}_{W_{\mathbb{N}}|U,T^{\mathbb{N}}_{\mathbb{M}}}\mathbb{E}_{\tau_\mathbb{N}}\mathbb{E}_{\bar{T}^{\mathbb{N}}_{\mathbb{M}}|\tau_\mathbb{N}}  \ell(\bar{U},W_{i},\bar{Z}^i_j) \bigg] \\
    = & \mathbb{E}_{\bar{U},T^{\mathbb{N}}_{\mathbb{M}}}\bigg[\frac{1}{n}\sum_{i=1}^{n}   \mathbb{E}_{W_i|U,T^{i}_{\mathbb{M}}} \mathbb{E}_{\tau_i\sim P_{\mathcal{T}}}  \mathbb{E}_{Z'\sim P_{Z|\tau_i}} \ell(\bar{U},W_{i},Z') \bigg] \\
    = & \mathbb{E}_{U,T^{\mathbb{N}}_{\mathbb{M}}}\bigg[\frac{1}{n}\sum_{i=1}^{n}   \mathbb{E}_{W_i|\bar{U},T^{i}_{\mathbb{M}}} \mathbb{E}_{\tau_i\sim P_{\mathcal{T}}} \mathbb{E}_{Z'\sim P_{Z|\tau_i}} \ell(U,W_{i},Z') \bigg] \\
    = & \mathbb{E}_{U,T^{\mathbb{N}}_{\mathbb{M}}}\bigg[\frac{1}{n}\sum_{i=1}^{n}    \mathbb{E}_{\tau_i\sim P_{\mathcal{T}}} \mathbb{E}_{\bar{T}^{i}_{\mathbb{M}}\sim P_{Z|\tau_i}}  \mathbb{E}_{\bar{W}_i|U,\bar{T}^{i}_{\mathbb{M}}}  \mathbb{E}_{Z'\sim P_{Z|\tau_i}} \ell(U,\bar{W}_{i},Z') \bigg] \\
    = &\mathbb{E}_{U,T^{\mathbb{N}}_{\mathbb{M}}}\bigg[ \mathbb{E}_{\tau \sim P_{\mathcal{T}}}\mathbb{E}_{T_{\mathbb{M}}|\tau} \mathbb{E}_{W|U,T_{\mathbb{M}}}\mathbb{E}_{Z'\sim P'_{Z|\tau}} \ell(U,W,Z') \bigg] \\
    = &  \mathbb{E}_{U,T^{\mathbb{N}}_{\mathbb{M}}}[\mathcal{R}(U, \mathcal{T})].
\end{align*} 
Plugging the above estimations into (\ref{equ7}), we obtain 
\begin{align*}
    \Big\vert \mathbb{E}_{U,T^{\mathbb{N}}_{\mathbb{M}}}[\mathcal{R}(U, T^{\mathbb{N}}_{\mathbb{M}})] - \mathbb{E}_{U,T^{\mathbb{N}}_{\mathbb{M}}}[\mathcal{R}(U, \mathcal{T})] \Big\vert \leq &  \mathbb{E}_{K\sim \mathbb{K}, J\sim\mathbb{J}} \sqrt{\frac{2\sigma^2}{\zeta \xi} I(U,W_K;T^{K}_{J})}.
\end{align*}
This completes the proof.

\end{proof}

\subsection{Proof of Proposition \ref{proposition1}}
\begin{restateproposition}{\ref{proposition1}}[Restate] 
    Let $\zeta\in[n-1]$, $\xi \in[m-1]$, and $\mathbb{K}$ and $\mathbb{J}$ be random subsets of $[n]$ and $[m]$ with sizes $\zeta$ and $\xi$, respectively. Further, let $\mathbb{K}'$ and $\mathbb{J}'$ be random subsets of $[n]$ and $[m]$ with sizes $\zeta+1$ and $\xi+1$, respectively. If $g:\mathbb{R}\rightarrow\mathbb{R}$ is any non-decreasing concave function, then 
    \begin{align*}
         \mathbb{E}_{K\sim \mathbb{K}, J \sim \mathbb{J}} g\left(\frac{1}{\zeta \xi}I(U,W_K;T^{K}_{J}) \right) \leq \mathbb{E}_{K' \sim \mathbb{K}', J' \sim \mathbb{J}'} g\left(\frac{1}{(\zeta+1)(\xi+1)} I(U,W_{K'};T^{K'}_{J'}) \right).
    \end{align*}
\end{restateproposition}
\begin{proof}
Let subsets $K' = (K'_1,\ldots, K'_{\zeta+1})\sim \mathbb{K}'\subseteq [n]$ and $J' = (J'_1,\ldots, J'_{\xi+1})\sim \mathbb{J}'\subseteq [m]$. By applying the chain rule of MI, we have
    \begin{equation*}
        I(U,W_{K'}; T^{K'}_{J'}) = I(U; T^{K'}_{J'})+I(W_{K'}; T^{K'}_{J'}|U).
    \end{equation*}
For the first term on the RHS, we have 
\begin{align} 
    I(U; T^{K'}_{J'}) = & \sum_{i=1}^{\zeta+1} I(U; T^{K'_i}_{J'}| T^{K'_{1:i-1}}_{J'}) \nonumber\\
     = & \sum_{i=1}^{\zeta+1} I(U, T^{K'_{i+1:\zeta+1}}_{J'} ;T^{K'_i}_{J'}|T^{K'_{1:i-1}}_{J'})  -  I(T^{K'_i}_{J'};T^{K'_{i+1:\zeta+1}}_{J'} |U,T^{K'_{1:i-1}}_{J'}) \nonumber\\
     \leq& \sum_{i=1}^{\zeta+1} I(U, T^{K'_{i+1:\zeta+1}}_{J'} ;T^{K'_i}_{J'}|T^{K'_{1:i-1}}_{J'}) \nonumber\\
     = & \sum_{i=1}^{\zeta+1} I(U;T^{K'_i}_{J'}|T^{K'_{i+1:\zeta+1}}_{J'} , T^{K'_{1:i-1}}_{J'}) +   I(T^{K'_i}_{J'} ; T^{K'_{i+1:\zeta+1}}_{J'} |T^{K'_{1:i-1}}_{J'}) \nonumber\\
     = & \sum_{i=1}^{\zeta+1} I(U;T^{K'_i}_{J'}|T^{K'_{i+1:\zeta+1}}_{J'} , T^{K'_{1:i-1}}_{J'}). \label{equ8}
\end{align}
Leveraging the inequality (\ref{equ8}), we have 
\begin{align}
    I(U; T^{K'}_{J'}) = & \frac{1}{\zeta+1 } \sum_{K'_i \in K'} \big(I(U; T^{K'\backslash K'_i}_{J'}) + I(U;T^{ K'_i}_{J'}| T^{K'\backslash K'_i}_{J'}) \big) \nonumber\\
    \geq &   \frac{1}{\zeta + 1} \Big(\sum_{K'_i \in K'} I(U; T^{K'\backslash K'_i}_{J'}) + I(U; T^{K'}_{J'}) \Big) \nonumber\\
    \geq &  \frac{1}{\zeta} \sum_{K'_i \in K'} I(U; T^{K'\backslash K'_i}_{J'}) . \label{equ9}
\end{align}
Let $T^{K'\backslash K'_i}_{J'}=\{Z^{K'\backslash K'_i}_{J'_1} ,\ldots,Z^{K'\backslash K'_i}_{J'_{\xi+1}}\}$, where $Z^{K'\backslash K'_i}_{J'_j} = \{Z^{K'_1}_{J'_j},\ldots, Z^{K'_{i-1}}_{J'_j}, Z^{K'_{i+1}}_{J'_j},\ldots, Z^{K'_{\zeta+1}}_{J'_j}\}$ for $j\in[\xi+1]$. Again using the chain rule of MI on $I(U; T^{K'\backslash K'_i}_{J'})$, we have 
\begin{align}
    I(U; T^{K'\backslash K'_i}_{J'}) = &  \sum_{j=1}^{\xi+1} I(U;Z^{K'\backslash K'_i}_{J'_j}| Z^{K'\backslash K'_i}_{J'_{1:j-1}}) \nonumber \\
    = &   \sum_{j=1}^{\xi+1}  I(U, Z^{K'\backslash K'_i}_{J'_{j+1:\xi+1}};Z^{K'\backslash K'_i}_{J'_j}| Z^{K'\backslash K'_i}_{J'_{1:j-1}} )  -  \sum_{j=1}^{\xi+1} I( Z^{K'\backslash K'_i}_{J'_{j+1:\xi+1}};Z^{K'\backslash K'_i}_{J'_j}|U, Z^{K'\backslash K'_i}_{J'_{1:j-1}})  \nonumber \\
    \leq &   \sum_{j=1}^{\xi+1} I(U, Z^{K'\backslash K'_i}_{J'_{j+1:\xi+1}};Z^{K'\backslash K'_i}_{J'_j}| Z^{K'\backslash K'_i}_{J'_{1:j-1}}) \nonumber \\
    = &  \sum_{j=1}^{\xi+1} I(U ;Z^{K'\backslash K'_i}_{J'_j}| Z^{K'\backslash K'_i}_{J'_{1:j-1}}, Z^{K'\backslash K'_i}_{J'_{j+1:\xi+1}})  
    +    \sum_{j=1}^{\xi+1} I(Z^{K'\backslash K'_i}_{J'_{j+1:\xi+1}};Z^{K'\backslash K'_i}_{J'_j}|Z^{K'\backslash K'_i}_{J'_{1:j-1}})  \nonumber \\
    = & \sum_{j=1}^{\xi+1} I(U ;Z^{K'\backslash K'_i}_{J'_j}| Z^{K'\backslash K'_i}_{J'_{1:j-1}}, Z^{K'\backslash K'_i}_{J'_{j+1:\xi+1}}) . \label{equ10}
\end{align}
Similarly, we obtain that 
\begin{align}
    I(U; T^{K'\backslash K'_i}_{J'})  =&\frac{1}{\xi+1} \sum_{J'_j \in J'} \big( I(U; T^{K'\backslash K'_i}_{J'\backslash J'_j}) + I(U; Z^{K'\backslash K'_i}_{J'_j}| T^{K'\backslash K'_i}_{J'\backslash J'_j})\big) \nonumber\\
    =& \frac{1}{\xi+1} \sum_{J'_j \in J'} \big( I(U; T^{K'\backslash K'_i}_{J'\backslash J'_j}) + I(U; Z^{K'\backslash K'_i}_{J'_j}| Z^{K'\backslash K'_i}_{J'_{1:j-1}}, Z^{K'\backslash K'_i}_{J'_{j+1:\xi+1}})\big)  \nonumber\\
    \geq &   \frac{1}{\xi+1} \Big(\sum_{J'_j \in J'} I(U; T^{K'\backslash K'_i}_{J'\backslash J'_j}) + I(U; T^{K'\backslash K'_i}_{J'})  \Big)  \nonumber\\
    \geq & \frac{1}{\xi}  \sum_{J'_j \in J'} I(U; T^{K'\backslash K'_i}_{J'\backslash J'_j}). \label{equ11}
\end{align}
Putting (\ref{equ11}) back into (\ref{equ9}) yields
\begin{equation} \label{equ12}
    I(U; T^{K'}_{J'}) \geq \frac{1}{\zeta \xi} \sum_{K'_i \in K', J'_j \in J'} I(U; T^{K'\backslash K'_i}_{J'\backslash J'_j}).
\end{equation}
Analogously analyzing $I(W_{K'}; T^{K'}_{J'}|U)$. Since $W_i$, $i=1,\ldots,n$ are mutually independent given $T^{\mathbb{N}}_{\mathbb{M}}$ and $U$, we get that
\begin{align} 
    I(W_{K'}; T^{K'}_{J'}|U) = & \sum_{i=1}^{\zeta+1} I(W_{K'};T^{K_i'}_{J'}|U, T^{K'_{1:i-1}}_{J'}) \nonumber\\
    = & \sum_{i=1}^{\zeta+1} I(W_{K'}, T^{K'_{i+1:\zeta+1}}_{J'} ;T^{K_i'}_{J'}|U, T^{K'_{1:i-1}}_{J'}) - I(T^{K'_{i+1:\zeta+1}}_{J'} ;T^{K_i'}_{J'}|U, W_{K'}, T^{K'_{1:i-1}}_{J'}) \nonumber\\
    = & \sum_{i=1}^{\zeta+1} I(W_{K'}, T^{K'_{i+1:\zeta+1}}_{J'} ;T^{K_i'}_{J'}|U , T^{K'_{1:i-1}}_{J'})  \nonumber\\
    =  & \sum_{i=1}^{\zeta+1} I(W_{K'_i} ;T^{K_i'}_{J'}|U, T^{K'_{1:i-1}}_{J'}, T^{K'_{i+1:\zeta+1}}_{J'}) + I(T^{K'_{i+1:\zeta+1}}_{J'} ;T^{K_i'}_{J'}|U, T^{K'_{1:i-1}}_{J'}) \nonumber\\
    = & \sum_{i=1}^{\zeta+1} I(W_{K'_i} ;T^{K_i'}_{J'}|U, T^{K'_{1:i-1}}_{J'}, T^{K'_{i+1:\zeta+1}}_{J'}). \label{equ13}
\end{align}
Using the inequality (\ref{equ13}), we have 
\begin{align} 
    I(W_{K'}; T^{K'}_{J'}|U) = &  \frac{1}{\zeta+1} \sum_{K'_i\in K'} \big(I(W_{K' \backslash K'_i}; T^{K' \backslash K'_i}_{J'}|U) + I(W_{K'_i}; T^{K'_i}_{J'}|T^{K' \backslash K'_i}_{J'},U) \big) \nonumber\\
    \geq & \frac{1}{\zeta+1} \Big(\sum_{K'_i\in K'} I(W_{K'\backslash K'_i}; T^{K' \backslash K'_i}_{J'}|U) + I(W_{K'}; T^{K'}_{J'}|U) \Big)  \nonumber\\
    \geq & \frac{1}{\zeta} \sum_{K'_i\in K'} I(W_{K'\backslash K'_i}; T^{K' \backslash K'_i}_{J'}|U) . \label{equ14}
\end{align}
Similar to the proof of the inequality (\ref{equ10}), for the RHS of (\ref{equ14}), we further get that 
\begin{align*}
    I(W_{K'\backslash K'_i}; T^{K' \backslash K'_i}_{J'}|U) = &  \sum_{j=1}^{\xi+1} I(W_{K'\backslash K'_i} ;Z^{K' \backslash K'_i}_{J'_j}|U, Z^{K' \backslash K'_i}_{J'_{1:j-1}})  \\
    \leq &    \sum_{j=1}^{\xi+1} I(W_{K'\backslash K'_i} ;Z^{K' \backslash K'_i}_{J'_j}|U, Z^{K' \backslash K'_i}_{J'_{1:j-1}}, Z^{K' \backslash K'_i}_{J'_{j+1:\xi+1}}),
\end{align*}
which implies that
\begin{align}
   & I(W_{K'\backslash K'_i}; T^{K' \backslash K'_i}_{J'}|U) \\
    = & \frac{1}{\xi +1 } \sum_{J'_j\in J'} \big(I(W_{K'\backslash K'_i}; T^{K' \backslash K'_i}_{J'\backslash J'_j}|U) + I(W_{K'\backslash K'_i} ;Z^{K' \backslash K'_i}_{J'_j}|U, T^{K' \backslash K'_i}_{J'\backslash J'_j}) \big) \nonumber\\
    = & \frac{1}{\xi +1 } \sum_{J'_j\in J'} \big(I(W_{K'\backslash K'_i}; T^{K' \backslash K'_i}_{J'\backslash J'_j}|U) + I(W_{K'\backslash K'_i} ;Z^{K' \backslash K'_i}_{J'_j}|U,  Z^{K' \backslash K'_i}_{J'_{1:j-1}}, Z^{K' \backslash K'_i}_{J'_{j+1:\xi+1}}) \big) \nonumber\\
    \geq &  \frac{1}{\xi +1 } \Big(\sum_{J'_j\in J'} I(W_{K'\backslash K'_i}; T^{K' \backslash K'_i}_{J'\backslash J'_j}|U) + I(W_{K'\backslash K'_i}; T^{K' \backslash K'_i}_{J'}|U) \Big) \nonumber\\
      \geq  &  \frac{1}{\xi} \sum_{J'_j\in J'} I(W_{K'\backslash K'_i}; T^{K' \backslash K'_i}_{J'\backslash J'_j}|U). \label{equ15}
\end{align}
Combining (\ref{equ14}) and (\ref{equ15}), we get that 
\begin{equation}\label{equ16}
    I(W_{K'}; T^{K'}_{J'}|U) \geq \frac{1}{\zeta\xi} \sum_{K'_i\in K',J'_j\in J'} I(W_{K'\backslash K'_i}; T^{K' \backslash K'_i}_{J'\backslash J'_j}|U).
\end{equation}
By using the inequalities (\ref{equ12}) and (\ref{equ16}), one can obtain that 
\begin{align*}
   I(U, W_{K'}; T^{K'}_{J'})  = & I(U; T^{K'}_{J'}) +  I(W_{K'}; T^{K'}_{J'}|U) \\
   \geq &\frac{1}{\zeta\xi} \sum_{K'_i\in K',J'_j\in J'} \big(I(W_{K'\backslash K'_i}; T^{K' \backslash K'_i}_{J'\backslash J'_j}|U) + I(U; T^{K'\backslash K'_i}_{J'\backslash J'_j}) \big) \\
   = & \frac{1}{\zeta\xi} \sum_{K'_i\in K',J'_j\in J'} I(U, W_{K'\backslash K'_i} ;T^{K'\backslash K'_i}_{J'\backslash J'_j} ).
\end{align*}
We further employ the Jensen's inequality on the concave function $g$ and have 
\begin{align}
  &g\Big(\frac{1}{(\zeta+1) (\xi+1)} I(U, W_{K'}; T^{K'}_{J'}) \Big) \nonumber\\
 \geq & g\Big(\frac{1}{(\zeta+1) (\xi+1)}  \Big( \frac{1}{\zeta\xi} \sum_{K'_i\in K',J'_j\in J'} I(U, W_{K'\backslash K'_i} ;T^{K'\backslash K'_i}_{J'\backslash J'_j} ) \Big)\Big) \nonumber\\
    \geq  & \frac{1}{(\zeta+1) (\xi+1)}  \sum_{K'_i\in K',J'_j\in J'} g\Big( \frac{1}{\zeta\xi} I(U, W_{K'\backslash K'_i} ;T^{K'\backslash K'_i}_{J'\backslash J'_j} )  \Big) \label{equ17}
\end{align}
Taking expectation over $K'$ and $J'$ on both sides of (\ref{equ17}), 
\begin{align*}
  & \mathbb{E}_{K'\sim \mathbb{K}', J'\sim \mathbb{J}'} g\Big(\frac{1}{(\zeta+1) (\xi+1)} I(U, W_{K'}; T^{K'}_{J'}) \Big)  \\
     \geq  & \mathbb{E}_{K'\sim \mathbb{K}', J'\sim \mathbb{J}'} \bigg[ \frac{1}{(\zeta+1) (\xi+1)}  \sum_{K'_i\in K',J'_j\in J'} g\Big( \frac{1}{\zeta\xi} I(U, W_{K'\backslash K'_i} ;T^{K'\backslash K'_i}_{J'\backslash J'_j} )  \Big) \bigg] \\
    = &  \frac{1}{C_n^{\zeta} C_m^{\xi}} \sum_{K \in \mathbb{K}, J\in \mathbb{J}}  g\Big( \frac{1}{\zeta\xi} I(U, W_{K} ;T^{K}_{J} )  \Big) \\
    = &  \mathbb{E}_{K\sim \mathbb{K}, J\sim \mathbb{J}} g\Big(\frac{1}{\zeta\xi} I(U, W_{K}; T^{K}_{J})  \Big),
\end{align*}
and this completes the proof.
\end{proof}

\section{Omitted Proofs [CMI Bounds]}

\subsection{Proof of Theorem \ref{theorem3.5}}

\begin{restatetheorem}{\ref{theorem3.5}}[Restate]
    Let $\mathbb{K}$ and $\mathbb{J}$ be random subsets of $[n]$ and $[m]$ with sizes $\zeta$ and $\xi$, respectively. If the loss function $\ell(\cdot,\cdot,\cdot)$ is bounded within $[0,1]$, then
    \begin{equation*}
        \vert \overline{\mathrm{gen}} \vert \leq  \mathbb{E}_{T_{2\mathbb{M}}^{2\mathbb{N}}, K\sim\mathbb{K}, J\sim\mathbb{J} }\sqrt{\frac{2}{\zeta\xi}I^{T_{2\mathbb{M}}^{2\mathbb{N}}}(U, W_K;  \tilde{S}_K,S_J)}.
    \end{equation*}
\end{restatetheorem}

\begin{proof}
    Let us condition on $\mathbb{K}=K$, $\mathbb{J} = J$, and $T_{2\mathbb{M}}^{2\mathbb{N}}$. Further let $\tilde{S}_K = \{\tilde{S}_i\}_{i\in K}\subseteq \tilde{S}_{\mathbb{N}}$, $S_J = \{S_j\}_{j\in J}\subseteq S_{\mathbb{M}}$, and
    \begin{equation*}
        f(u,\omega_K,\tilde{s}_K,s_J) = \frac{1}{\zeta}\sum_{i\in K}\frac{1}{\xi}\sum_{j\in J}  \Big(\ell(u,\omega_{i},\tilde{Z}^{i,\tilde{s}_i}_{j,s_j}) - \ell(u,\omega_{i},\tilde{Z}^{i,\bar{\tilde{s}}_i}_{j,\bar{s}_j}) \Big) .
    \end{equation*}
Let $\tilde{S}'_{\mathbb{N}}$ and $S'_{\mathbb{M}}$ be independent copies of $S_{\mathbb{M}}$ and $\tilde{S}_{\mathbb{N}}$, respectively.  It is noteworthy that each summand of $f(u,\omega_K,\tilde{S}'_K,S'_J)$ is a $1$-sub-gaussian random variable and has zero mean, as $\ell(\cdot,\cdot,\cdot)$ take values in $[0,1]$. Hence, $f(u,\omega_K,\tilde{S}'_K,S'_J)$ is $\frac{1}{\sqrt{\zeta\xi}}$-sub-gaussian with $\mathbb{E}_{\tilde{S}'_{\mathbb{N}},S'_{\mathbb{M}}}[f(u,\omega_K,\tilde{S}'_K,S'_J)]=0$. Following Lemma \ref{lemmaA.7} with $X=(U,W_K)$ and $ Y=(\tilde{S}'_K,S'_J)$, we get that $f(U,W_K,\tilde{S}'_K,S'_J)$ is also zero-mean $\frac{1}{\sqrt{\zeta\xi}}$-sub-gaussian. We further apply Lemma \ref{lemmaA.6} with the choices $X=(U,W_K)$, $ Y=(\tilde{S}_K,S_J)$ and $f(X,Y) = f(U,W_K,\tilde{S}_K,S_J)$, and have
\begin{equation*}
    \bigg\vert  \mathbb{E}_{U,W_\mathbb{N},\tilde{S}_{\mathbb{N}},S_{\mathbb{M}}}\bigg[ \frac{1}{\zeta}\sum_{i\in K}\frac{1}{\xi}\sum_{j\in J} \Big(\ell(U,W_{i},\tilde{Z}^{i,\tilde{S}_i}_{j,S_j}) - \ell(U,W_{i},\tilde{Z}^{i,\bar{\tilde{S}}_i}_{j,\bar{S}_j}) \Big) \bigg] \bigg\vert\leq \sqrt{\frac{2}{\zeta\xi}I^{T_{2\mathbb{M}}^{2\mathbb{N}}}(U, W_K; \tilde{S}_K,S_J)}.
\end{equation*}

Taking expectation over $K$, $J$ and $T_{2\mathbb{M}}^{2\mathbb{N}}$ on both sides and applying Jensen's inequality to swap the order of expectation and absolute value, we obtain 
\begin{align}
    & \bigg\vert \mathbb{E}_{T_{2\mathbb{M}}^{2\mathbb{N}}, K\sim\mathbb{K}, J\sim\mathbb{J} ,U,W_\mathbb{N},\tilde{S}_{\mathbb{N}},S_{\mathbb{M}}}\bigg[ \frac{1}{\zeta}\sum_{i\in K}\frac{1}{\xi}\sum_{j\in J} \Big(\ell(U,W_{i},\tilde{Z}^{i,\tilde{S}_i}_{j,S_j}) - \ell(U,W_{i},\tilde{Z}^{i,\bar{\tilde{S}}_i}_{j,\bar{S}_j}) \Big) \bigg] \bigg\vert \nonumber\\ \leq &  \mathbb{E}_{T_{2\mathbb{M}}^{2\mathbb{N}}, K\sim\mathbb{K}, J\sim\mathbb{J} }\sqrt{\frac{2}{\zeta\xi}I^{T_{2\mathbb{M}}^{2\mathbb{N}}}(U, W_K; \tilde{S}_K,S_J)}, \label{equ18}
\end{align}
which reduces to 
\begin{align}
    &\bigg\vert \mathbb{E}_{T_{2\mathbb{M}}^{2\mathbb{N}},U,W_\mathbb{N},\tilde{S}_{\mathbb{N}},S_{\mathbb{M}}}\bigg[ \frac{1}{nm}\sum_{i,j=1}^{n,m} \Big(\ell(U,W_{i},\tilde{Z}^{i,\tilde{S}_i}_{j,S_j}) - \ell(U,W_{i},\tilde{Z}^{i,\bar{\tilde{S}}_i}_{j,\bar{S}_j}) \Big) \bigg] \bigg\vert \nonumber\\
    \leq & \mathbb{E}_{T_{2\mathbb{M}}^{2\mathbb{N}}, K\sim\mathbb{K}, J\sim\mathbb{J} }\sqrt{\frac{2}{\zeta\xi}I^{T_{2\mathbb{M}}^{2\mathbb{N}}}(U, W_K;  \tilde{S}_K,S_J)},\label{equ19}
\end{align}

We then prove the LSH of the inequality (\ref{equ19}) is equal to the absolute value of the meta-generalization gap. Since $\tilde{S}_{\mathbb{N}} \perp \bar{\tilde{S}}_{\mathbb{N}}$ and $S_{\mathbb{M}} \perp \bar{S}_{\mathbb{M}}$, we have $T_{2\mathbb{M},S_{\mathbb{M}}}^{2\mathbb{N},\tilde{S}_{\mathbb{N}}} \perp T_{2\mathbb{M},\bar{S}_{\mathbb{M}}}^{2\mathbb{N},\tilde{S}_{\mathbb{N}}} \perp T_{2\mathbb{M},S_{\mathbb{M}}}^{2\mathbb{N},\bar{\tilde{S}}_{\mathbb{N}}}\perp T^{2\mathbb{N},\bar{\tilde{S}}_{\mathbb{N}}}_{2\mathbb{M},\bar{S}_{\mathbb{M}}}  $ and $W_i$, $i=1,\ldots,n$ are mutually independent given $U$ and $T_{2\mathbb{M},S_{\mathbb{M}}}^{2\mathbb{N},\tilde{S}_{\mathbb{N}}}$. Then
\begin{align}
    &\mathbb{E}_{T_{2\mathbb{M}}^{2\mathbb{N}},U,W_\mathbb{N},\tilde{S}_{\mathbb{N}},S_{\mathbb{M}}}  \Big[\frac{1}{nm}\sum_{i,j=1}^{n,m} \ell(U,W_{i},\tilde{Z}^{i,\tilde{S}_i}_{j,S_j})\Big] \nonumber\\
    = & \mathbb{E}_{U,W_\mathbb{N}, T_{2\mathbb{M},S_{\mathbb{M}}}^{2\mathbb{N},\tilde{S}_{\mathbb{N}}},  T_{2\mathbb{M},\bar{S}_{\mathbb{M}}}^{2\mathbb{N},\tilde{S}_{\mathbb{N}}}, T_{2\mathbb{M},S_{\mathbb{M}}}^{2\mathbb{N},\bar{\tilde{S}}_{\mathbb{N}}}, T^{2\mathbb{N},\bar{\tilde{S}}_{\mathbb{N}}}_{2\mathbb{M},\bar{S}_{\mathbb{M}}}  }  \Big[ \frac{1}{nm}\sum_{i,j=1}^{n,m}\ell(U,W_{i},\tilde{Z}^{i,\tilde{S}_i}_{j,S_j})\Big] \nonumber\\
    = & \mathbb{E}_{U,W_\mathbb{N},   T_{2\mathbb{M},S_{\mathbb{M}}}^{2\mathbb{N},\tilde{S}_{\mathbb{N}}}} \Big[\frac{1}{nm}\sum_{i,j=1}^{n,m} \ell(U,W_{i},\tilde{Z}^{i,\tilde{S}_i}_{j,S_j})  \Big]\nonumber\\
    = & \mathbb{E}_{U,  T_{2\mathbb{M},S_{\mathbb{M}}}^{2\mathbb{N},\tilde{S}_{\mathbb{N}}}}  \Big[\frac{1}{nm}\sum_{i,j=1}^{n,m} \mathbb{E}_{W_\mathbb{N}|  T_{2\mathbb{M},S_{\mathbb{M}}}^{2\mathbb{N},\tilde{S}_{\mathbb{N}}},U} \ell(U,W_{i},\tilde{Z}^{i,\tilde{S}_i}_{j,S_j}) \Big] \nonumber\\
    = & \mathbb{E}_{U,  T_{2\mathbb{M},S_{\mathbb{M}}}^{2\mathbb{N},\tilde{S}_{\mathbb{N}}}}  \Big[\frac{1}{nm}\sum_{i,j=1}^{n,m} \mathbb{E}_{W_{i}|  T_{2\mathbb{M},S_{\mathbb{M}}}^{i,\tilde{S}_i},U} \ell(U,W_{i},\tilde{Z}^{i,\tilde{S}_i}_{j,S_j})\Big] \nonumber\\
    = & \mathbb{E}_{U, T_{\mathbb{M}}^{\mathbb{N}}}\Big[ \frac{1}{nm}\sum_{i,j=1}^{n,m} \mathbb{E}_{W_{i}|  T_{\mathbb{M}}^{i},U} \ell(U,W_{i},Z^i_j) \Big]\nonumber\\
     = &  \mathbb{E}_{U, T_{\mathbb{M}}^{\mathbb{N}}}[\mathcal{R}(U, T_{\mathbb{M}}^{\mathbb{N}})]. \label{equ20}
\end{align}
Similarly, by using the independence of $T_{2\mathbb{M},S_{\mathbb{M}}}^{2\mathbb{N},\tilde{S}_{\mathbb{N}}}, T_{2\mathbb{M},\bar{S}_{\mathbb{M}}}^{2\mathbb{N},\tilde{S}_{\mathbb{N}}}, T_{2\mathbb{M},\bar{S}_{\mathbb{M}}}^{2\mathbb{N},\bar{\tilde{S}}_{\mathbb{N}}}, T^{2\mathbb{N},\bar{\tilde{S}}_{\mathbb{N}}}_{2\mathbb{M},S_{\mathbb{M}}}$, we have $T^{2\mathbb{N},\bar{\tilde{S}}_{\mathbb{N}}}_{2\mathbb{M},S_{\mathbb{M}}} \perp T_{2\mathbb{M},\bar{S}_{\mathbb{M}}}^{2\mathbb{N},\bar{\tilde{S}}_{\mathbb{N}}} \perp (T_{2\mathbb{M},S_{\mathbb{M}}}^{2\mathbb{N},\tilde{S}_{\mathbb{N}}},U) $ and $W_i$, $i=1,\ldots,n$ are mutually independent given $U$ and $T^{2\mathbb{N},\bar{\tilde{S}}_{\mathbb{N}}}_{2\mathbb{M},S_{\mathbb{M}}}$. We then have 
\begin{align}
    & \mathbb{E}_{T_{2\mathbb{M}}^{2\mathbb{N}},U,W_\mathbb{N},\tilde{S}_{\mathbb{N}},S_{\mathbb{M}}} \Big[\frac{1}{nm}\sum_{i,j=1}^{n,m} \ell(U,W_i,\tilde{Z}^{i,\bar{\tilde{S}}_i}_{j,\bar{S}_j}) \Big] \nonumber\\
    = & \mathbb{E}_{U,W_\mathbb{N}, T_{2\mathbb{M},S_{\mathbb{M}}}^{2\mathbb{N},\tilde{S}_{\mathbb{N}}},  T_{2\mathbb{M},\bar{S}_{\mathbb{M}}}^{2\mathbb{N},\tilde{S}_{\mathbb{N}}}, T_{2\mathbb{M},\bar{S}_{\mathbb{M}}}^{2\mathbb{N},\bar{\tilde{S}}_{\mathbb{N}}}, T^{2\mathbb{N},\bar{\tilde{S}}_{\mathbb{N}}}_{2\mathbb{M},S_{\mathbb{M}}}  }   \Big[ \frac{1}{nm}\sum_{i,j=1}^{n,m} \ell(U,W_i,\tilde{Z}^{i,\bar{\tilde{S}}_i}_{j,\bar{S}_j}) \Big] \nonumber\\
    = &  \mathbb{E}_{U,T_{2\mathbb{M},S_{\mathbb{M}}}^{2\mathbb{N},\tilde{S}_{\mathbb{N}}}} \Big[ \frac{1}{nm}\sum_{i,j=1}^{n,m} \mathbb{E}_{W_\mathbb{N}, T^{2\mathbb{N},\bar{\tilde{S}}_{\mathbb{N}}}_{2\mathbb{M},S_{\mathbb{M}}}|U} \mathbb{E}_{T_{2\mathbb{M},\bar{S}_{\mathbb{M}}}^{2\mathbb{N},\bar{\tilde{S}}_{\mathbb{N}}}}  \ell(U,W_i,\tilde{Z}^{i,\bar{\tilde{S}}_i}_{j,\bar{S}_j}) \Big]  \nonumber\\
    = & \mathbb{E}_{U,T_{2\mathbb{M},S_{\mathbb{M}}}^{2\mathbb{N},\tilde{S}_{\mathbb{N}}}} \Big[ \frac{1}{nm}\sum_{i,j=1}^{n,m}  \mathbb{E}_{ T^{2\mathbb{N},\bar{\tilde{S}}_{\mathbb{N}}}_{2\mathbb{M},S_{\mathbb{M}}}}  \mathbb{E}_{W_\mathbb{N}|U,T^{2\mathbb{N},\bar{\tilde{S}}_{\mathbb{N}}}_{2\mathbb{M},S_{\mathbb{M}}}} \mathbb{E}_{T_{2\mathbb{M},\bar{S}_{\mathbb{M}}}^{2\mathbb{N},\bar{\tilde{S}}_{\mathbb{N}}}}  \ell(U,W_i,\tilde{Z}^{i,\bar{\tilde{S}}_i}_{j,\bar{S}_j}) \Big]  \nonumber\\ 
    = &  \mathbb{E}_{U,T_{2\mathbb{M},S_{\mathbb{M}}}^{2\mathbb{N},\tilde{S}_{\mathbb{N}}}} \Big[ \frac{1}{nm}\sum_{i,j=1}^{n,m}  \mathbb{E}_{ T^{i,\bar{\tilde{S}}_i}_{2\mathbb{M},S_{\mathbb{M}}}}  \mathbb{E}_{W_i|U,T^{i,\bar{\tilde{S}}_i}_{2\mathbb{M},S_{\mathbb{M}}}} \mathbb{E}_{T^{i,\bar{\tilde{S}}_i}_{2\mathbb{M},\bar{S}_{\mathbb{M}}}}  \ell(U,W_i,\tilde{Z}^{i,\bar{\tilde{S}}_i}_{j,\bar{S}_j}) \Big] \nonumber\\  
    = & \mathbb{E}_{U,T_{2\mathbb{M},S_{\mathbb{M}}}^{2\mathbb{N},\tilde{S}_{\mathbb{N}}}} \Big[ \frac{1}{n}\sum_{i=1}^{n} \mathbb{E}_{\tau_i\sim P_{\mathcal{T}}} \mathbb{E}_{ T^{i,\bar{\tilde{S}}_i}_{2\mathbb{M},S_{\mathbb{M}}}\sim P_{Z|\tau_i}}  \mathbb{E}_{W_i|U,T^{i,\bar{\tilde{S}}_i}_{2\mathbb{M},S_{\mathbb{M}}}} \mathbb{E}_{Z'\sim P_{Z|\tau_i}}  \ell(U,W_i,Z') \Big] \nonumber\\ 
    = & \mathbb{E}_{U,T_{2\mathbb{M},S_{\mathbb{M}}}^{2\mathbb{N},\tilde{S}_{\mathbb{N}}}} \big[ \mathbb{E}_{\tau\sim P_{\mathcal{T}}} \mathbb{E}_{T_{\mathbb{M}}\sim P_{Z|\tau}} \mathbb{E}_{W|U,T_{\mathbb{M}}}  \mathbb{E}_{Z'\sim P_{Z|\tau}}  \ell(U,W,Z') \big]  \nonumber\\ 
    = & \mathbb{E}_{U, T_{\mathbb{M}}^{\mathbb{N}}}[\mathcal{R}(U, \mathcal{T})] \label{equ21}
\end{align}
Substituting (\ref{equ20}) and (\ref{equ21}) into (\ref{equ19}), this completes the proof.

\end{proof}

\subsection{Proof of Proposition \ref{proposition2}}

\begin{restateproposition}{\ref{proposition2}}[Restate]
    Let $\zeta\in[n-1]$, $\xi \in[m-1]$, and $\mathbb{K}$ and $\mathbb{J}$ be random subsets of $[n]$ and $[m]$ with sizes $\zeta$ and $\xi$, respectively. Further, let $\mathbb{K}'$ and $\mathbb{J}'$ be random subsets with sizes $\zeta+1$ and $\xi+1$, respectively. If $g:\mathbb{R}\rightarrow\mathbb{R}$ is any non-decreasing concave function, then for $T_{2\mathbb{M}}^{2\mathbb{N}} \sim \{P_{Z|\tau_i}\}_{i=1}^n$ over $\mathcal{Z}^{2n\times 2m}$,
    \begin{align*}
         \mathbb{E}_{K\sim \mathbb{K}, J \sim \mathbb{J}} g\left(\frac{1}{\zeta \xi} I^{T_{2\mathbb{M}}^{2\mathbb{N}}}(U, W_K;  \tilde{S}_K,S_J) \right) \leq \mathbb{E}_{K' \sim \mathbb{K}', J' \sim \mathbb{J}'} g\left(\frac{1}{(\zeta+1)(\xi+1)} I^{T_{2\mathbb{M}}^{2\mathbb{N}}}(U, W_{K'};  \tilde{S}_{K'},S_{J'}) \right).
    \end{align*}
\end{restateproposition}
\begin{proof}
    The proof follows the same procedure as the proof of Proposition \ref{proposition1}, by replacing the MI $I(U,W_{K}; T^{K}_{J})$ with the disintegrated MI $ I^{T_{2\mathbb{M}}^{2\mathbb{N}}}(U,W_{K}; \tilde{S}_K,S_J)$.
\end{proof}

\subsection{Proof of Theorem \ref{theorem3.6}}
\begin{restatetheorem}{\ref{theorem3.6}}[Restate]
        Assume that the loss function $\ell(\cdot,\cdot,\cdot)$ is bounded within $[0,1]$, then
        \begin{equation*}
            d\left(\hat{\mathcal{R}} \Big\Vert \frac{\hat{\mathcal{R}} + \mathcal{R}_{\mathcal{T}}}{2} \right) \leq \frac{1}{nm}\sum_{i,j=1}^{n,m} I(U,W_i; \tilde{S}_i,S_j|T_{2\mathbb{M}}^{2\mathbb{N}}).
        \end{equation*}
Furthermore, in the interpolating setting that $\hat{\mathcal{R}}=0$, we have 
\begin{equation*}
    \mathcal{R}_{\mathcal{T}} \leq \frac{2}{nm}\sum_{i,j=1}^{n,m}  I(U,W_i; \tilde{S}_i,S_j|T_{2\mathbb{M}}^{2\mathbb{N}}).
\end{equation*}
\end{restatetheorem}
\begin{proof} 
    Following the inequalities (\ref{equ20}) and (\ref{equ21}), we can decompose the average empirical and population risks into 
    \begin{align}
         \hat{\mathcal{R}} = \mathbb{E}_{T_{2\mathbb{M}}^{2\mathbb{N}}, K\sim\mathbb{K}, J\sim\mathbb{J} ,U,W_\mathbb{N},\tilde{S}_{\mathbb{N}},S_{\mathbb{M}}} \frac{1}{\zeta}\sum_{i\in K}\frac{1}{\xi}\sum_{j\in J} \ell(U,W_{i},\tilde{Z}^{i,\tilde{S}_i}_{j,S_j}) \label{equ22} \\
        \mathcal{R}_{\mathcal{T}} =\mathbb{E}_{T_{2\mathbb{M}}^{2\mathbb{N}}, K\sim\mathbb{K}, J\sim\mathbb{J} ,U,W_\mathbb{N},\tilde{S}_{\mathbb{N}},S_{\mathbb{M}}} \frac{1}{\zeta}\sum_{i\in K}\frac{1}{\xi}\sum_{j\in J}\ell(U,W_{i},\tilde{Z}^{i,\bar{\tilde{S}}_i}_{j,\bar{S}_j})  \label{equ23}
    \end{align}
Applying Jensen's inequality and the convexity of $d_\gamma(\cdot\Vert \cdot)$, we obtain
\begin{align} 
    & d\left(\hat{\mathcal{R}} \Big\Vert \frac{\hat{\mathcal{R}} + \mathcal{R}_{\mathcal{T}}}{2} \right) \nonumber\\
    = &  \sup_{\gamma} d_\gamma\left(\hat{\mathcal{R}} \Big\Vert \frac{\hat{\mathcal{R}} + \mathcal{R}_{\mathcal{T}}}{2} \right) \nonumber\\ 
     \leq & \sup_{\gamma} \frac{1}{\zeta\xi}  \mathbb{E}_{K\sim\mathbb{K}, J\sim\mathbb{J},T_{2\mathbb{M}}^{2\mathbb{N}},U,W_\mathbb{N},\tilde{S}_{\mathbb{N}},S_{\mathbb{M}}} d_{\gamma}\Bigg( \sum_{i\in K, j\in J} \ell(U,W_{i},\tilde{Z}^{i,\tilde{S}_i}_{j,S_j}) \bigg\Vert \sum_{i\in K, j\in J} \frac{\ell(U,W_{i},\tilde{Z}^{i,\tilde{S}_i}_{j,S_j}) + \ell(U,W_{i},\tilde{Z}^{i,\bar{\tilde{S}}_i}_{j,\bar{S}_j})}{2} \Bigg) \nonumber\\
     = & \sup_{\gamma} \frac{1}{\zeta\xi}  \mathbb{E}_{K\sim\mathbb{K}, J\sim\mathbb{J},T_{2\mathbb{M}}^{2\mathbb{N}}} \mathbb{E}_{U,W_K,\tilde{S}_K,S_J|T_{2\mathbb{M}}^{2\mathbb{N}}} d_{\gamma}\Bigg( \sum_{i\in K, j\in J} \ell(U,W_{i},\tilde{Z}^{i,\tilde{S}_i}_{j,S_j}) \bigg\Vert \sum_{i\in K, j\in J} \frac{\ell(U,W_{i},\tilde{Z}^{i,\tilde{S}_i}_{j,S_j}) + \ell(U,W_{i},\tilde{Z}^{i,\bar{\tilde{S}}_i}_{j,\bar{S}_j})}{2} \Bigg). \label{equ24}
\end{align}
Let $\tilde{S}'_{\mathbb{N}}$ and $S'_{\mathbb{M}}$ be independent copies of $S_{\mathbb{M}}$ and $\tilde{S}_{\mathbb{N}}$. Let us condition on $K,J,T_{2\mathbb{M}}^{2\mathbb{N}}$ and utilize Lemma \ref{lemmaA.5} with $P=P_{U,W_K,\tilde{S}_K,S_J|T_{2\mathbb{M}}^{2\mathbb{N}}}$, $Q = P_{U,W_K|T_{2\mathbb{M}}^{2\mathbb{N}}} P_{\tilde{S}_K,S_J}$, and $f=d_{\gamma}\Big(\sum_{i\in K, j\in J} \ell(U,W_{i},\tilde{Z}^{i,\tilde{S}_i}_{j,S_j}) \Big\Vert \sum_{i\in K, j\in J} \frac{\ell(U,W_{i},\tilde{Z}^{i,\tilde{S}_i}_{j,S_j}) + \ell(U,W_{i},\tilde{Z}^{i,\bar{\tilde{S}}_i}_{j,\bar{S}_j})}{2}\Big)$, we have
\begin{align}
    I^{T_{2\mathbb{M}}^{2\mathbb{N}}}(U,W_K; \tilde{S}_K,S_J) = & D\left( P_{U,W_K,\tilde{S}_K,S_J|T_{2\mathbb{M}}^{2\mathbb{N}}} \Vert P_{U,W_K|T_{2\mathbb{M}}^{2\mathbb{N}}} P_{\tilde{S}_K,S_J} \right) \nonumber\\
    \geq  &\mathbb{E}_{U,W_K,\tilde{S}_K,S_J|T_{2\mathbb{M}}^{2\mathbb{N}}} d_{\gamma}\Bigg( \sum_{i\in K, j\in J} \ell(U,W_{i},\tilde{Z}^{i,\tilde{S}_i}_{j,S_j}) \bigg\Vert \sum_{i\in K, j\in J} \frac{\ell(U,W_{i},\tilde{Z}^{i,\tilde{S}_i}_{j,S_j}) + \ell(U,W_{i},\tilde{Z}^{i,\bar{\tilde{S}}_i}_{j,\bar{S}_j})}{2} \Bigg) \nonumber\\
    & - \log \mathbb{E}_{U,W_K,\tilde{S}'_K,S'_J|T_{2\mathbb{M}}^{2\mathbb{N}}}  e^{d_{\gamma}\Big( \sum_{i\in K, j\in J} \ell\big(U,W_{i},\tilde{Z}^{i,\tilde{S}'_i}_{j,S'_j}\big) \Big\Vert \sum_{i\in K, j\in J} \frac{\ell\big(U,W_{i},\tilde{Z}^{i,\tilde{S}'_i}_{j,S'_j}\big) + \ell\big(U,W_{i},\tilde{Z}^{i,\bar{\tilde{S}}'_i}_{j,\bar{S}^\prime_j}\big)}{2}  \Big)}. \label{equ25}
\end{align}
Notice that 
\begin{align*}
    \mathbb{E}_{\tilde{S}'_K,S'_J}\Big[\sum_{i\in K, j\in J} \ell\big(U,W_{i},\tilde{Z}^{i,\tilde{S}'_i}_{j,S'_j}\big) \Big] =  \sum_{i\in K, j\in J}  \mathbb{E}_{\tilde{S}'_i,S'_j} \big[\ell\big(U,W_{i},\tilde{Z}^{i,\tilde{S}'_i}_{j,S'_j}\big)\big] 
    =  \sum_{i\in K, j\in J} \frac{\ell\big(U,W_{i},\tilde{Z}^{i,\tilde{S}'_i}_{j,S'_j}\big) + \ell\big(U,W_{i},\tilde{Z}^{i,\bar{\tilde{S}}'_i}_{j,\bar{S}'_j}\big)}{2}.
\end{align*}
Utilizing Lemma \ref{lemmaA.8} and the above equation, for any $\gamma>0$, we know that 
\begin{equation}\label{equ26}
    \mathbb{E}_{U,W_K,\tilde{S}'_K,S'_J|T_{2\mathbb{M}}^{2\mathbb{N}}}  e^{d_{\gamma}\Big( \sum_{i\in K, j\in J} \ell\big(U,W_{i},\tilde{Z}^{i,\tilde{S}'_i}_{j,S'_j}\big) \Big\Vert \sum_{i\in K, j\in J} \frac{\ell\big(U,W_{i},\tilde{Z}^{i,\tilde{S}'_i}_{j,S'_j}\big) + \ell\big(U,W_{i},\tilde{Z}^{i,\bar{\tilde{S}}'_i}_{j,\bar{S}^\prime_j}\big)}{2}  \Big)} \leq 1.
\end{equation}
Putting inequality (\ref{equ26}) back into (\ref{equ25}), we have
\begin{align}\label{equ27}
    I^{T_{2\mathbb{M}}^{2\mathbb{N}}}(U,W_K; \tilde{S}_K,S_J) \geq  \mathbb{E}_{U,W_K,\tilde{S}_K,S_J|T_{2\mathbb{M}}^{2\mathbb{N}}} d_{\gamma}\Bigg( \sum_{i\in K, j\in J} \ell(U,W_{i},\tilde{Z}^{i,\tilde{S}_i}_{j,S_j}) \bigg\Vert \sum_{i\in K, j\in J} \frac{\ell(U,W_{i},\tilde{Z}^{i,\tilde{S}_i}_{j,S_j}) + \ell(U,W_{i},\tilde{Z}^{i,\bar{\tilde{S}}_i}_{j,\bar{S}_j})}{2} \Bigg).
\end{align}
Further plugging (\ref{equ27}) back into (\ref{equ24}), we have 
\begin{align*}
    & d\left(\hat{\mathcal{R}} \Big\Vert \frac{\hat{\mathcal{R}} + \mathcal{R}_{\mathcal{T}}}{2} \right) \nonumber\\
     \leq &\sup_{\gamma} \frac{1}{\zeta\xi}  \mathbb{E}_{K\sim\mathbb{K}, J\sim\mathbb{J},T_{2\mathbb{M}}^{2\mathbb{N}}} \mathbb{E}_{U,W_K,\tilde{S}_K,S_J|T_{2\mathbb{M}}^{2\mathbb{N}}} d_{\gamma}\Bigg( \sum_{i\in K, j\in J} \ell(U,W_{i},\tilde{Z}^{i,\tilde{S}_i}_{j,S_j}) \bigg\Vert \sum_{i\in K, j\in J} \frac{\ell(U,W_{i},\tilde{Z}^{i,\tilde{S}_i}_{j,S_j}) + \ell(U,W_{i},\tilde{Z}^{i,\bar{\tilde{S}}_i}_{j,\bar{S}_j})}{2} \Bigg) \nonumber\\
     \leq & \mathbb{E}_{K\sim\mathbb{K}, J\sim\mathbb{J},T_{2\mathbb{M}}^{2\mathbb{N}}} \frac{I^{T_{2\mathbb{M}}^{2\mathbb{N}}}(U,W_K; \tilde{S}_K,S_J)}{\zeta\xi} = \mathbb{E}_{K\sim\mathbb{K}, J\sim\mathbb{J}} \frac{I(U,W_K; \tilde{S}_K,S_J|T_{2\mathbb{M}}^{2\mathbb{N}})}{\zeta\xi}. 
\end{align*}
By taking $\zeta=\xi=1$, 
\begin{equation*}
    d\left(\hat{\mathcal{R}} \Big\Vert \frac{\hat{\mathcal{R}} + \mathcal{R}_{\mathcal{T}}}{2} \right)  \leq \frac{1}{nm}\sum_{i,j=1}^{n,m} I(U,W_i; \tilde{S}_i,S_j|T_{2\mathbb{M}}^{2\mathbb{N}}).
\end{equation*}
When $\hat{\mathcal{R}}=0$, we obtain 
\begin{equation*}
    d\left(\hat{\mathcal{R}} \Big\Vert \frac{\hat{\mathcal{R}} + \mathcal{R}_{\mathcal{T}}}{2} \right) = d\left( 0 \Big\Vert \frac{ \mathcal{R}_{\mathcal{T}}}{2} \right) \geq \frac{ \mathcal{R}_{\mathcal{T}}}{2}. 
\end{equation*}
This completes the proof.

\end{proof}

\subsection{Proof of Theorem \ref{theorem3.8}}
\begin{restatetheorem}{\ref{theorem3.8}}[Restate]
    Assume that the loss function $\ell(\cdot,\cdot,\cdot)$ is bounded within $[0,1]$, then for any $0 < C_2<\log 2$ and $C_1\geq -\frac{\log(2-e^{C_2})}{C_2}-1$,
  \begin{equation*}
    \overline{\mathrm{gen}} \leq C_1 \hat{\mathcal{R}} + \frac{1}{nm}\sum_{i,j=1}^{n,m} \frac{I(U,W_i;\tilde{S}_i,S_j|T_{2\mathbb{M}}^{2\mathbb{N}})}{C_2}.
  \end{equation*}
  In the interpolating setting, i.e., $\hat{\mathcal{R}} = 0$, we have
  \begin{equation*}
    \mathcal{R}_{\mathcal{T}} \leq  \frac{1}{nm}\sum_{i,j=1}^{n,m} \frac{I(U,W_i;\tilde{S}_i,S_j|T_{2\mathbb{M}}^{2\mathbb{N}})}{\log 2}.
  \end{equation*}
\end{restatetheorem}
\begin{proof}
    According to the definition of meta-generalization gap, we have 
    \begin{align}
         \mathcal{R}_{\mathcal{T}}-(1+C_1)\hat{\mathcal{R}} =&  \mathbb{E}_{T_{2\mathbb{M}}^{2\mathbb{N}},U,W_\mathbb{N},\tilde{S}_{\mathbb{N}},S_{\mathbb{M}}}\bigg[ \frac{1}{nm}\sum_{i,j=1}^{n,m} \Big(\ell(U,W_{i},\tilde{Z}^{i,\bar{\tilde{S}}_i}_{j,\bar{S}_j}) - (1+C_1)\ell(U,W_{i},\tilde{Z}^{i,\tilde{S}_i}_{j,S_j}) \Big) \bigg]\nonumber\\
         =& \mathbb{E}_{T_{2\mathbb{M}}^{2\mathbb{N}}} \frac{1}{nm}\sum_{i,j=1}^{n,m} \bigg[\mathbb{E}_{U,W_i,\tilde{S}_i,S_j|T_{2\mathbb{M}}^{2\mathbb{N}}}  \Big(\ell(U,W_{i},\tilde{Z}^{i,\bar{\tilde{S}}_i}_{j,\bar{S}_j}) - (1+C_1)\ell(U,W_{i},\tilde{Z}^{i,\tilde{S}_i}_{j,S_j}) \Big) \bigg]. \nonumber\\
         = & \mathbb{E}_{T_{2\mathbb{M}}^{2\mathbb{N}}} \frac{1}{nm}\sum_{i,j=1}^{n,m} \bigg[\mathbb{E}_{U,W_i,\tilde{S}_i,S_j|T_{2\mathbb{M}}^{2\mathbb{N}}}  \Big(L^{i,\bar{\tilde{S}}_i}_{j,\bar{S}_j} - (1+C_1) L^{i,\tilde{S}_i}_{j,S_j}\Big) \bigg]. 
         \label{equ28}
    \end{align}
    Let $\tilde{S}'_{\mathbb{N}}$ and $S'_{\mathbb{M}}$ be independent copies of $S_{\mathbb{M}}$ and $\tilde{S}_{\mathbb{N}}$, respectively. By leveraging Lemma \ref{lemmaA.5} with $P=P_{U,W_i,\tilde{S}_i,S_j|T_{2\mathbb{M}}^{2\mathbb{N}}}, Q=P_{U,W_i|T_{2\mathbb{M}}^{2\mathbb{N}}}P_{\tilde{S}_i,S_j}$, and $f(U,W_i, \tilde{S}_i,S_j)= L^{i,\bar{\tilde{S}}_i}_{j,\bar{S}_j} - (1+C_1) L^{i,\tilde{S}_i}_{j,S_j}$, we get that  
    \begin{align}
        &I^{T_{2\mathbb{M}}^{2\mathbb{N}}}(U,W_i;\tilde{S}_i,S_j) 
        =  D(P_{U,W_i,\tilde{S}_i,S_j|T_{2\mathbb{M}}^{2\mathbb{N}}} \Vert P_{U,W_i|T_{2\mathbb{M}}^{2\mathbb{N}}}P_{\tilde{S}_i,S_j}) \nonumber\\
        \geq & \sup_{C_2>0}\Big\{\mathbb{E}_{U,W_i,\tilde{S}_i,S_j|T_{2\mathbb{M}}^{2\mathbb{N}}} \Big[C_2 \big(L^{i,\bar{\tilde{S}}_i}_{j,\bar{S}_j} - (1+C_1) L^{i,\tilde{S}_i}_{j,S_j}\big) \Big] - \log \mathbb{E}_{U,W_i,\tilde{S}'_i,S'_j|T_{2\mathbb{M}}^{2\mathbb{N}}} \Big[e^{C_2\big(L^{i,\bar{\tilde{S}}'_i}_{j,\bar{S}'_j} - (1+C_1) L^{i,\tilde{S}'_i}_{j,S'_j}\big)} \Big]   \Big\} \nonumber\\
        = &  \sup_{C_2>0}\Bigg\{\mathbb{E}_{U,W_i,\tilde{S}_i,S_j|T_{2\mathbb{M}}^{2\mathbb{N}}} \Big[C_2 \big(L^{i,\bar{\tilde{S}}_i}_{j,\bar{S}_j} - (1+C_1) L^{i,\tilde{S}_i}_{j,S_j}\big) \Big] \nonumber\\
        &- \log \bigg(\mathbb{I}\{\tilde{S}_i\oplus S_j  = 0\}\frac{\mathbb{E}_{U,W_i|T_{2\mathbb{M}}^{\mathbb{N}}} \Big[ e^{C_2 \big(L^{i,1}_{j,1} - C_2(1+C_1) L^{i,0}_{j,0}\big)}   +  e^{C_2 \big(L^{i,0}_{j,0} - C_2(1+C_1) L^{i,1}_{j,1}\big)}\Big] }{2} \nonumber\\
        &+ \mathbb{I}\{\tilde{S}_i\oplus S_j  = 1\}\frac{\mathbb{E}_{U,W_i|T_{2\mathbb{M}}^{\mathbb{N}}} \Big[ e^{C_2 \big(L^{i,1}_{j,0} - C_2(1+C_1) L^{i,0}_{j,1}\big)}   +  e^{C_2 \big(L^{i,0}_{j,1} - C_2(1+C_1) L^{i,1}_{j,0}\big)}\Big] }{2}\bigg)\Bigg\} \nonumber\\
        = &  \sup_{C_2>0}\bigg\{\mathbb{E}_{U,W_i,\tilde{S}_i,S_j|T_{2\mathbb{M}}^{2\mathbb{N}}} \Big[C_2 \big(L^{i,\bar{\tilde{S}}_i}_{j,\bar{S}_j} - (1+C_1) L^{i,\tilde{S}_i}_{j,S_j}\big) \Big] \nonumber\\ 
        &-  \log \frac{\mathbb{E}_{U,W_i|T_{2\mathbb{M}}^{\mathbb{N}}} \Big[ e^{C_2 \big(L^{\Psi_+}_{i,j} - C_2(1+C_1) L^{\Psi_-}_{i,j}\big)}   +  e^{C_2 \big(L^{\Psi_-}_{i,j}- C_2(1+C_1) L^{\Psi_+}_{i,j}\big)}\Big] }{2} \bigg\} \label{equ29}
    \end{align}
where $\mathbb{I}\{\cdot\}$ is the indicator function. Let $\lambda_{C_1,C_2} = e^{C_2 \big(L^{\Psi_+}_{i,j} - C_2(1+C_1) L^{\Psi_-}_{i,j}\big)}   +  e^{C_2 \big(L^{\Psi_-}_{i,j}- C_2(1+C_1) L^{\Psi_+}_{i,j}\big)}$. We intend to select the values of $C_1,C_2$ such that the $\log$ term of (\ref{equ29}) is guaranteed to less than $0$, that is $\lambda_{C_1,C_2}\leq 2$. By the convexity of exponential function, one could know that the maximum value of this term can be achieved at the endpoints of $L^{\Psi_-}_{i,j} , L^{\Psi_+}_{i,j}\in [0,1]$. It is natural that
\begin{equation*}
    \lambda_{C_1,C_2} = \left\{\begin{array}{ll}
        2 & if \quad L^{\Psi_+}_{i,j}=L^{\Psi_-}_{i,j}=0, \\
        2e^{-C_1C_2}  & if  \quad L^{\Psi_+}_{i,j}=L^{\Psi_-}_{i,j}=1, \\
        e^{C_2} + e^{-C_2(1+C_1)} & if  \quad L^{\Psi_+}_{i,j}=1, L^{\Psi_-}_{i,j}=0 \;\textrm{or} \; L^{\Psi_+}_{i,j}=0, L^{\Psi_-}_{i,j}=1. \\
    \end{array}
    \right.
\end{equation*}
For the case of $L^{\Psi_+}_{i,j}=1$ and $ L^{\Psi_-}_{i,j}=0$ ( or $L^{\Psi_+}_{i,j}=0, L^{\Psi_-}_{i,j}=1$), it suffices to select a large enough $C_1\rightarrow\infty$ such that $\lambda_{C_1,C_2} =  e^{C_2} + e^{-C_2-C_1C_2} \leq 2$, which implies that $C_2\leq \log 2$ and $C_1\geq - \frac{\log(2-e^{C_2})}{C_2}-1$. By the above estimations, we obtain
\begin{align}
    \log \frac{\mathbb{E}_{U,W_i|T_{2\mathbb{M}}^{\mathbb{N}}} \Big[ e^{C_2 \big(L^{\Psi_+}_{i,j} - C_2(1+C_1) L^{\Psi_-}_{i,j}\big)}   +  e^{C_2 \big(L^{\Psi_-}_{i,j}- C_2(1+C_1) L^{\Psi_+}_{i,j}\big)}\Big] }{2} \leq 0. \label{equ30}
\end{align}
Plugging the inequality (\ref{equ30}) back into (\ref{equ29}), 
\begin{align}\label{equ31}
    I^{T_{2\mathbb{M}}^{2\mathbb{N}}}(U,W_i;\tilde{S}_i,S_j)   \geq& \mathbb{E}_{U,W_i,\tilde{S}_i,S_j|T_{2\mathbb{M}}^{2\mathbb{N}}} \Big[C_2 \big(L^{i,\bar{\tilde{S}}_i}_{j,\bar{S}_j} - (1+C_1) L^{i,\tilde{S}_i}_{j,S_j}\big) \Big].
\end{align}
Substituting (\ref{equ31}) into (\ref{equ28}), we have 
\begin{align*} 
  \overline{\mathrm{gen}}  = &  \mathcal{R}_{\mathcal{T}}-(1+C_1)\hat{\mathcal{R}} + C_1 \hat{\mathcal{R}} \\
   \leq & C_1 \hat{\mathcal{R}} +  \frac{1}{nm}\sum_{i,j=1}^{n,m}  \mathbb{E}_{T_{2\mathbb{M}}^{2\mathbb{N}}} \Big[\frac{I^{T_{2\mathbb{M}}^{2\mathbb{N}}}(U,W_i;\tilde{S}_i,S_j)}{C_2}\Big]    \\
   = & C_1 \hat{\mathcal{R}} +    \frac{1}{nm}\sum_{i,j=1}^{n,m} \frac{I(U,W_i;\tilde{S}_i,S_j|T_{2\mathbb{M}}^{2\mathbb{N}})}{C_2} . 
\end{align*}
In the interpolating regime where $\hat{\mathcal{R}}=0$, by letting $C_2\rightarrow \log 2$ and $C_1\rightarrow \infty$, we have 
\begin{align*}
    \mathcal{R}_{\mathcal{T}} \leq \frac{1}{nm}\sum_{i,j=1}^{n,m} \frac{I(U,W_i;\tilde{S}_i,S_j|T_{2\mathbb{M}}^{2\mathbb{N}})}{\log 2}. 
\end{align*}
This completes the proof.
\end{proof}

\section{Omitted Proofs [e-CMI Bounds]}
\subsection{Proof of Theorem \ref{theorem4}}
\begin{restatetheorem}{\ref{theorem4}}
    Assume that the loss function $\ell(\cdot,\cdot,\cdot) \in [0,1]$, then
    \begin{equation*}
        \vert \overline{\mathrm{gen}} \vert \leq \frac{1}{nm}\sum_{i,j=1}^{n,m} \mathbb{E}_{T^{2\mathbb{N}}_{2\mathbb{M}}}\sqrt{2I^{T^{2\mathbb{N}}_{2\mathbb{M}}} (L^i_j; \tilde{S}_i,S_j)}.
    \end{equation*}
\end{restatetheorem}
\begin{proof}
  According to equations (\ref{equ20}) and (\ref{equ21}), we notice that 
    \begin{align}
        \vert \overline{\mathrm{gen}} \vert = &\bigg\vert \mathbb{E}_{U,W_\mathbb{N},T_{2\mathbb{M}}^{2\mathbb{N}},\tilde{S}_{\mathbb{N}},S_{\mathbb{M}}} \Big[ \frac{1}{nm}\sum_{i,j=1}^{n,m} \Big(\ell(U,W_{i},\tilde{Z}^{i,\bar{\tilde{S}}_i}_{j,\bar{S}_j}) - \ell(U,W_{i},\tilde{Z}^{i,\tilde{S}_i}_{j,S_j}) \Big) \Big]  \bigg\vert \nonumber\\
        \leq& \mathbb{E}_{T_{2\mathbb{M}}^{2\mathbb{N}}}  \bigg\vert \mathbb{E}_{U,W_\mathbb{N},\tilde{S}_{\mathbb{N}},S_{\mathbb{M}}|T_{2\mathbb{M}}^{2\mathbb{N}}} \Big[ \frac{1}{nm}\sum_{i,j=1}^{n,m} \Big(\ell(U,W_{i},\tilde{Z}^{i,\bar{\tilde{S}}_i}_{j,\bar{S}_j}) - \ell(U,W_{i},\tilde{Z}^{i,\tilde{S}_i}_{j,S_j}) \Big) \Big]  \bigg\vert \nonumber\\
        \leq &\mathbb{E}_{T_{2\mathbb{M}}^{2\mathbb{N}}}  \frac{1}{nm}\sum_{i,j=1}^{n,m} \bigg\vert \mathbb{E}_{U,W_i,\tilde{S}_i,S_j|T_{2\mathbb{M}}^{2\mathbb{N}}} \Big[ \ell(U,W_{i},\tilde{Z}^{i,\bar{\tilde{S}}_i}_{j,\bar{S}_j}) - \ell(U,W_{i},\tilde{Z}^{i,\tilde{S}_i}_{j,S_j}) \Big]   \bigg\vert \nonumber\\
        \leq &\mathbb{E}_{T_{2\mathbb{M}}^{2\mathbb{N}}}  \frac{1}{nm}\sum_{i,j=1}^{n,m} \Big\vert \mathbb{E}_{L^i_j,\tilde{S}_i,S_j|T_{2\mathbb{M}}^{2\mathbb{N}}} \Big[ L^{i,\bar{\tilde{S}}_i}_{j,\bar{S}_j} - L^{i,\tilde{S}_i}_{j,S_j}\Big]   \Big\vert.  \label{equ32}
    \end{align}
By the assumption of the loss function, we have $L^{i,\bar{\tilde{S}}_i}_{j,\bar{S}_j} - L^{i,\tilde{S}_i}_{j,S_j} \in[-1,1]$, and thus  $L^{i,\bar{\tilde{S}}_i}_{j,\bar{S}_j} - L^{i,\tilde{S}_i}_{j,S_j}$ is $1$-sub-gaussian. Let $\tilde{S}'_{\mathbb{N}}, S'_{\mathbb{M}}$ be independent copies of $\tilde{S}_{\mathbb{N}},S_{\mathbb{M}}$. Applying Lemma \ref{lemmaA.5} with $P=P_{L^i_j,\tilde{S}_i,S_j|T_{2\mathbb{M}}^{2\mathbb{N}}}$, $Q=P_{L^i_j|T_{2\mathbb{M}}^{2\mathbb{N}}}P_{\tilde{S}_i,S_j}$ and $f(L^i_j,\tilde{S}_i,S_j) = L^{i,\bar{\tilde{S}}_i}_{j,\bar{S}_j} - L^{i,\tilde{S}_i}_{j,S_j}$, we have 
\begin{equation}\label{equ33}
    \Big\vert \mathbb{E}_{L^i_j,\tilde{S}_i,S_j|T_{2\mathbb{M}}^{2\mathbb{N}}} \Big[ L^{i,\bar{\tilde{S}}_i}_{j,\bar{S}_j} - L^{i,\tilde{S}_i}_{j,S_j}\Big] - \mathbb{E}_{L^i_j,\tilde{S}'_i,S'_j|T_{2\mathbb{M}}^{2\mathbb{N}}} \Big[ L^{i,\bar{\tilde{S}}'_i}_{j,\bar{S}'_j} - L^{i,\tilde{S}'_i}_{j,S'_j}\Big] \Big\vert \leq \sqrt{2I^{T^{2\mathbb{N}}_{2\mathbb{M}}} (L^i_j; \tilde{S}_i,S_j)}.
\end{equation}
It is easy to prove that $ \mathbb{E}_{L^i_j,\tilde{S}'_i,S'_j|T_{2\mathbb{M}}^{2\mathbb{N}}} \big[ L^{i,\bar{\tilde{S}}'_i}_{j,\bar{S}'_j} - L^{i,\tilde{S}'_i}_{j,S'_j}\big]=0$. Putting the above estimation back into (\ref{equ32}), we obtain 
\begin{equation*}
    \vert \overline{\mathrm{gen}} \vert \leq \frac{1}{nm}\sum_{i,j=1}^{n,m} \mathbb{E}_{T^{2\mathbb{N}}_{2\mathbb{M}}}\sqrt{2I^{T^{2\mathbb{N}}_{2\mathbb{M}}} (L^i_j; \tilde{S}_i,S_j)},
\end{equation*}
and complete the proof.
\end{proof}

\subsection{Proof of Theorem \ref{theorem5}}
\begin{restatetheorem}{\ref{theorem5}}
        Assume that the loss function $\ell(\cdot,\cdot,\cdot) \in [0,1]$, then
        \begin{equation*}
            \vert \overline{\mathrm{gen}} \vert \leq \frac{1}{nm}\sum_{i,j=1}^{n,m}  \sqrt{2I (L^i_j; \tilde{S}_i,S_j)}.
        \end{equation*}
\end{restatetheorem}
\begin{proof}
    In the proof of Theorem \ref{theorem4}, we notice that if we do not move the expectation over $T^{2\mathbb{N}}_{2\mathbb{M}}$ outside of the absolute function and directly take the expectation over $L^i_j$, we will have the opportunity to get rid of the expectation over $T^{2\mathbb{N}}_{2\mathbb{M}}$. By the definition of the meta-generalization gap, we get 
    \begin{align}
        \vert \overline{\mathrm{gen}} \vert = &\bigg\vert \mathbb{E}_{U,W_\mathbb{N},T_{2\mathbb{M}}^{2\mathbb{N}},\tilde{S}_{\mathbb{N}},S_{\mathbb{M}}} \Big[ \frac{1}{nm}\sum_{i,j=1}^{n,m} \Big(\ell(U,W_{i},\tilde{Z}^{i,\bar{\tilde{S}}_i}_{j,\bar{S}_j}) - \ell(U,W_{i},\tilde{Z}^{i,\tilde{S}_i}_{j,S_j}) \Big) \Big]  \bigg\vert \nonumber\\
       = &\Big\vert \frac{1}{nm}\sum_{i,j=1}^{n,m} \mathbb{E}_{T_{2\mathbb{M}}^{2\mathbb{N}},U,W_i,\tilde{S}_i,S_j} \Big[\ell(U,W_{i},\tilde{Z}^{i,\bar{\tilde{S}}_i}_{j,\bar{S}_j}) - \ell(U,W_{i},\tilde{Z}^{i,\tilde{S}_i}_{j,S_j}) \Big]   \Big\vert \nonumber\\
        \leq & \frac{1}{nm}\sum_{i,j=1}^{n,m} \Big\vert \mathbb{E}_{L^i_j,\tilde{S}_i,S_j} \Big[ \ell(U,W_{i},\tilde{Z}^{i,\bar{\tilde{S}}_i}_{j,\bar{S}_j}) - \ell(U,W_{i},\tilde{Z}^{i,\tilde{S}_i}_{j,S_j}) \Big]  \Big\vert \nonumber\\
        = & \frac{1}{nm}\sum_{i,j=1}^{n,m} \Big\vert \mathbb{E}_{L^i_j,\tilde{S}_i,S_j} \Big[L^{i,\bar{\tilde{S}}_i}_{j,\bar{S}_j} - L^{i,\tilde{S}_i}_{j,S_j}  \Big]  \Big\vert .  \label{equ34}
    \end{align}
Analogous proof to the inequality (\ref{equ33}), we can similarly prove that 
\begin{equation*}
    \Big\vert \mathbb{E}_{L^i_j,\tilde{S}_i,S_j} \Big[ L^{i,\bar{\tilde{S}}_i}_{j,\bar{S}_j} - L^{i,\tilde{S}_i}_{j,S_j}\Big] - \mathbb{E}_{L^i_j,\tilde{S}'_i,S'_j}  \Big[ L^{i,\bar{\tilde{S}}'_i}_{j,\bar{S}'_j} - L^{i,\tilde{S}'_i}_{j,S'_j}\Big] \Big\vert \leq \sqrt{2I (L^i_j; \tilde{S}_i,S_j)},
\end{equation*}
and $\mathbb{E}_{L^i_j,\tilde{S}'_i,S'_j}  \big[ L^{i,\bar{\tilde{S}}'_i}_{j,\bar{S}'_j} - L^{i,\tilde{S}'_i}_{j,S'_j}\big] =0$. Substituting the above inequality into (\ref{equ34}) yields
\begin{equation*}
    \vert \overline{\mathrm{gen}} \vert \leq \frac{1}{nm}\sum_{i,j=1}^{n,m}  \sqrt{2I (L^i_j; \tilde{S}_i,S_j)},
\end{equation*}
which completes the proof.
\end{proof}

\subsection{Proof of Theorem \ref{theorem6}}

\begin{restatetheorem}{\ref{theorem6}}
        Assume that the loss function $\ell(\cdot,\cdot,\cdot) \in [0,1]$, then
        \begin{equation*}
            d\left(\hat{\mathcal{R}} \Big\Vert \frac{\hat{\mathcal{R}} + \mathcal{R}_{\mathcal{T}}}{2} \right) \leq \frac{1}{nm}\sum_{i,j=1}^{n,m} I(L^i_j ; \tilde{S}_i, S_j).
        \end{equation*}
\end{restatetheorem}

\begin{proof}
By the definition, we have
\begin{align}
  & d\left(\hat{\mathcal{R}} \Big\Vert \frac{\hat{\mathcal{R}} + \mathcal{R}_{\mathcal{T}}}{2} \right) \nonumber\\
     \leq &  \sup_{\gamma}   \frac{1}{nm}  \mathbb{E}_{U,W_\mathbb{N},T_{2\mathbb{M}}^{2\mathbb{N}},\tilde{S}_{\mathbb{N}},S_{\mathbb{M}}} d_{\gamma}\Bigg( \sum_{i,j=1}^{n,m} \ell(U,W_i,\tilde{Z}^{i,\tilde{S}_i}_{j,S_j}) \bigg\Vert \sum_{i,j=1}^{n,m} \frac{\ell(U,W_i,\tilde{Z}^{i,\tilde{S}_i}_{j,S_j}) + \ell(U,W_i,\tilde{Z}^{i,\bar{\tilde{S}}_i}_{j,\bar{S}_j})}{2} \Bigg) \nonumber\\
     \leq &   \sup_{\gamma}  \frac{1}{nm} \sum_{i,j=1}^{n,m} \mathbb{E}_{U,W_i,\tilde{S}_i,S_j,T_{2\mathbb{M}}^{2\mathbb{N}}} d_{\gamma}\Bigg( \ell(U,W_i,\tilde{Z}^{i,\tilde{S}_i}_{j,S_j}) \bigg\Vert  \frac{\ell(U,W_i,\tilde{Z}^{i,\tilde{S}_i}_{j,S_j}) + \ell(U,W_i,\tilde{Z}^{i,\bar{\tilde{S}}_i}_{j,\bar{S}_j})}{2} \Bigg)  \nonumber\\
     = &   \sup_{\gamma}  \frac{1}{nm} \sum_{i,j=1}^{n,m} \mathbb{E}_{L^i_j ,\tilde{S}_i, S_j} d_{\gamma}\bigg(  L^{i,\tilde{S}_i}_{j,S_j} \bigg\Vert  \frac{L^{i,\tilde{S}_i}_{j,S_j}  +L^{i,\bar{\tilde{S}}_i}_{j,\bar{S}_j}}{2} \bigg). \label{equ35}
\end{align}    
Let $\tilde{S}'_{\mathbb{N}}, S'_{\mathbb{M}}$ be independent copies of $\tilde{S}_{\mathbb{N}},S_{\mathbb{M}}$. Applying Lemma \ref{lemmaA.5} with $P=P_{L^i_j ,\tilde{S}_i, S_j}$, $Q=P_{L_j^i}P_{\tilde{S}_i, S_j}$, and $f(L^i_j ,\tilde{S}_i, S_j)=d_{\gamma}\big( L^{i,\tilde{S}_i}_{j,S_j} \big\Vert  \frac{L^{i,\tilde{S}_i}_{j,S_j}  +L^{i,\bar{\tilde{S}}_i}_{j,\bar{S}_j}}{2} \big)$, we have for any $\gamma>0$
\begin{align}
    I(L^i_j ; \tilde{S}_i, S_j) = & D(P_{L^i_j ,\tilde{S}_i, S_j} \Vert P_{L_j^i}P_{\tilde{S}_i, S_j}) \nonumber\\
    \geq & \mathbb{E}_{L^i_j ,\tilde{S}_i, S_j} d_{\gamma}\bigg(  L^{i,\tilde{S}_i}_{j,S_j} \bigg\Vert  \frac{L^{i,\tilde{S}_i}_{j,S_j}  +L^{i,\bar{\tilde{S}}_i}_{j,\bar{S}_j}}{2} \bigg)- \log \mathbb{E}_{L^i_j ,\tilde{S}'_i, S'_j} \bigg[e^{d_{\gamma}\Big(   L^{i,\tilde{S}'_i}_{j,S'_j} \Big\Vert  \frac{L^{i,\tilde{S}'_i}_{j,S'_j}  +L^{i,\bar{\tilde{S}}'_i}_{j,\bar{S}'_j}}{2} \Big)}\bigg] . \label{equ36}
\end{align}
Notice that $\mathbb{E}_{\tilde{S}'_i, S'_j}[L^{i,\tilde{S}'_i}_{j,S'_j}] = (L^{i,\tilde{S}'_i}_{j,S'_j}  +L^{i,\bar{\tilde{S}}'_i}_{j,\bar{S}'_j})/2$, by Lemma \ref{lemmaA.8}, for any $\gamma\in\mathbb{R}$, we have 
\begin{equation*}
    \mathbb{E}_{L^i_j ,\tilde{S}'_i, S'_j} \bigg[e^{d_{\gamma}\Big(   L^{i,\tilde{S}'_i}_{j,S'_j} \Big\Vert  \frac{L^{i,\tilde{S}'_i}_{j,S'_j}  +L^{i,\bar{\tilde{S}}'_i}_{j,\bar{S}'_j}}{2} \Big)}\bigg] \leq 1.
\end{equation*}
Plugging the above inequality into (\ref{equ36}) implies that 
\begin{equation}\label{equ37}
    \mathbb{E}_{L^i_j ,\tilde{S}_i, S_j} d_{\gamma}\bigg(  L^{i,\tilde{S}_i}_{j,S_j} \bigg\Vert  \frac{L^{i,\tilde{S}_i}_{j,S_j}  +L^{i,\bar{\tilde{S}}_i}_{j,\bar{S}_j}}{2} \bigg) \leq  I(L^i_j ; \tilde{S}_i, S_j).
\end{equation}
Substituting (\ref{equ37}) into (\ref{equ35}), we have 
\begin{equation*}
    d\left(\hat{\mathcal{R}} \Big\Vert \frac{\hat{\mathcal{R}} + \mathcal{R}_{\mathcal{T}}}{2} \right) \leq \frac{1}{nm}\sum_{i,j=1}^{n,m} I(L^i_j ; \tilde{S}_i, S_j),
\end{equation*}
and complete the proof.

\end{proof}
\section{Omitted Proofs [Loss-difference Bounds]}

\subsection{ Proof of Theorem \ref{theorem7}}
\begin{restatetheorem}{\ref{theorem7}}
    Assume that the loss function $\ell(\cdot,\cdot,\cdot) \in [0,1]$, then
    \begin{equation*}
        \vert \overline{\mathrm{gen}} \vert\leq    \frac{1}{nm}\sum_{i,j=1}^{n,m} \sqrt{2I(\Delta_{i,j}^{\Psi};S_j|T_{2\mathbb{M}}^{2\mathbb{N}})}.
    \end{equation*}
\end{restatetheorem}
\begin{proof}
    By the definition of the meta-generalization gap, we have 
    \begin{align}
        \vert \overline{\mathrm{gen}} \vert = &\Big\vert \mathbb{E}_{U,W_\mathbb{N},T_{2\mathbb{M}}^{2\mathbb{N}},\tilde{S}_{\mathbb{N}},S_{\mathbb{M}}} \Big[\frac{1}{nm}\sum_{i,j=1}^{n,m} \Big(\ell(U,W_i,\tilde{Z}^{i,\bar{\tilde{S}}_i}_{j,\bar{S}_j}) - \ell(U,W_i,\tilde{Z}^{i,\tilde{S}_i}_{j,S_j})\Big) \Big]   \Big\vert \nonumber\\
        \leq & \mathbb{E}_{T_{2\mathbb{M}}^{2\mathbb{N}}}\Big\vert \mathbb{E}_{L_{2\mathbb{M}}^{2\mathbb{N}},\tilde{S}_{\mathbb{N}},S_{\mathbb{M}}|T_{2\mathbb{M}}^{2\mathbb{N}}} \Big[\frac{1}{nm}\sum_{i,j=1}^{n,m} \Big(L^{i,\bar{\tilde{S}}_i}_{j,\bar{S}_j} - L^{i,\tilde{S}_i}_{j,S_j}\Big) \Big]     \Big\vert \nonumber\\
        \leq &\frac{1}{nm}\sum_{i,j=1}^{n,m}  \mathbb{E}_{T_{2\mathbb{M}}^{2\mathbb{N}}}\Big\vert \mathbb{E}_{L^i_j,\tilde{S}_i,S_j|T_{2\mathbb{M}}^{2\mathbb{N}}} \big[\mathbb{I}\{\tilde{S}_i\oplus S_j = 0\} (-1)^{S_j} (L^{i,1}_{j,1}-L^{i,0}_{j,0}) \nonumber\\
        & + \mathbb{I}\{\tilde{S}_i\oplus S_j = 1\} (-1)^{S_j} (L^{i,0}_{j,1}-L^{i,1}_{j,0}) \big] \Big\vert \nonumber\\
        = & \frac{1}{nm}\sum_{i,j=1}^{n,m}  \mathbb{E}_{T_{2\mathbb{M}}^{2\mathbb{N}}}\Big\vert \mathbb{E}_{L_{i,j}^{\Psi},\tilde{S}_i,S_j|T_{2\mathbb{M}}^{2\mathbb{N}}} \Big[  (-1)^{S_j} \Big(L_{i,j}^{\Psi_+} - L_{i,j}^{\Psi_-}\Big)\Big] \Big\vert    \nonumber\\
        = &   \frac{1}{nm}\sum_{i,j=1}^{n,m} \mathbb{E}_{T_{2\mathbb{M}}^{2\mathbb{N}}} \Big\vert \mathbb{E}_{\Delta_{i,j}^{\Psi},S_j|T_{2\mathbb{M}}^{2\mathbb{N}}} \big[ (-1)^{S_j} \Delta_{i,j}^{\Psi} \big]   \Big\vert, \label{equ38}
    \end{align}
where $\mathbb{I}\{\cdot\}$ is the indicator function and the inequalities follows from the Jensen's inequality on the absolute function. Since the loss function takes values in $[0,1]$, $\Delta_{i,j}^{\Psi} \in[-1,+1]$ and thus $(-1)^{S'_j} \Delta_{i,j}^{\Psi}$ is $1$-sub-gaussian. Let $S'_{\mathbb{M}}$ be an independent copy of $S_{\mathbb{M}}$. Applying Lemma \ref{lemmaA.5} with $P=P_{\Delta_{i,j}^{\Psi},S_j|T_{2\mathbb{M}}^{2\mathbb{N}}}$, $Q = P_{\Delta_{i,j}^{\Psi}|T_{2\mathbb{M}}^{2\mathbb{N}}} P_{S_j}$ and $f(\Delta_{i,j}^{\Psi},S_j) =(-1)^{S_j} \Delta_{i,j}^{\Psi}$, we have 
\begin{align*}
   \Big\vert \mathbb{E}_{\Delta_{i,j}^{\Psi},S_j|T_{2\mathbb{M}}^{2\mathbb{N}}} \big[(-1)^{S_j} \Delta_{i,j}^{\Psi}\big]  - \mathbb{E}_{\Delta_{i,j}^{\Psi},S'_j|T_{2\mathbb{M}}^{2\mathbb{N}}} \big[(-1)^{S'_j} \Delta_{i,j}^{\Psi}\big] \Big\vert 
    = &  \Big\vert \mathbb{E}_{\Delta_{i,j}^{\Psi},S_j|T_{2\mathbb{M}}^{2\mathbb{N}}} \big[(-1)^{S_j} \Delta_{i,j}^{\Psi}\big] \Big\vert \\
     \leq & \sqrt{2I^{T_{2\mathbb{M}}^{2\mathbb{N}}}(\Delta_{i,j}^{\Psi};S_j)},
\end{align*}
by $\mathbb{E}_{\Delta_{i,j}^{\Psi},S'_j|T_{2\mathbb{M}}^{2\mathbb{N}}} \big[(-1)^{S'_j} \Delta_{i,j}^{\Psi}\big] = 0$. By substituting the above inequality into (\ref{equ38}), we get that
\begin{align*}
    \vert \overline{\mathrm{gen}} \vert \leq &  \frac{1}{nm}\sum_{i,j=1}^{n,m} \mathbb{E}_{T_{2\mathbb{M}}^{2\mathbb{N}}} \Big\vert \mathbb{E}_{\Delta_{i,j}^{\Psi},S_j|T_{2\mathbb{M}}^{2\mathbb{N}}} \big[ (-1)^{S_j} \Delta_{i,j}^{\Psi} \big]   \Big\vert 
    \leq  \frac{1}{nm}\sum_{i,j=1}^{n,m} \mathbb{E}_{T_{2\mathbb{M}}^{2\mathbb{N}}}  \sqrt{2I^{T_{2\mathbb{M}}^{2\mathbb{N}}}(\Delta_{i,j}^{\Psi};S_j)} \\
    \leq &  \frac{1}{nm}\sum_{i,j=1}^{n,m} \sqrt{2I(\Delta_{i,j}^{\Psi};S_j|T_{2\mathbb{M}}^{2\mathbb{N}})}.
\end{align*}

This completes the proof.
\end{proof}

\subsection{ Proof of Theorem \ref{theorem8}}
\begin{restatetheorem}{\ref{theorem8}}
        Assume that the loss function $\ell(\cdot,\cdot,\cdot) \in [0,1]$, then
        \begin{equation*}
            \vert \overline{\mathrm{gen}} \vert \leq \frac{1}{nm}\sum_{i,j=1}^{n,m} \sqrt{2I(\Delta_{i,j}^{\Psi};S_j)}.
        \end{equation*}
\end{restatetheorem}
\begin{proof}
    Notice that 
    \begin{align}
        \vert \overline{\mathrm{gen}} \vert = &\Big\vert \mathbb{E}_{U,W_\mathbb{N},T_{2\mathbb{M}}^{2\mathbb{N}},\tilde{S}_{\mathbb{N}},S_{\mathbb{M}}} \Big[\frac{1}{nm}\sum_{i,j=1}^{n,m} \Big(\ell(U,W_i,\tilde{Z}^{i,\bar{\tilde{S}}_i}_{j,\bar{S}_j}) - \ell(U,W_i,\tilde{Z}^{i,\tilde{S}_i}_{j,S_j})\Big) \Big]   \Big\vert \nonumber\\
        = & \Big\vert \mathbb{E}_{L_{2\mathbb{M}}^{2\mathbb{N}},\tilde{S}_{\mathbb{N}},S_{\mathbb{M}}}\Big[\frac{1}{nm}\sum_{i,j=1}^{n,m} \Big(L^{i,\bar{\tilde{S}}_i}_{j,\bar{S}_j} - L^{i,\tilde{S}_i}_{j,S_j}\Big) \Big]   \Big\vert \nonumber\\
        \leq &  \frac{1}{nm}\sum_{i,j=1}^{n,m}\Big\vert \mathbb{E}_{L^i_j,\tilde{S}_i,S_j} \Big[ L^{i,\bar{\tilde{S}}_i}_{j,\bar{S}_j} - L^{i,\tilde{S}_i}_{j,S_j} \Big]   \Big\vert \nonumber\\
        = & \frac{1}{nm}\sum_{i,j=1}^{n,m}\Big\vert \mathbb{E}_{L_{i,j}^{\Psi},\tilde{S}_i,S_j}  \big[ (-1)^{S_j} \Big(L_{i,j}^{\Psi_+} - L_{i,j}^{\Psi_-}\Big) \big] \Big\vert \nonumber\\
        = & \frac{1}{nm}\sum_{i,j=1}^{n,m} \Big\vert \mathbb{E}_{\Delta_{i,j}^{\Psi},S_j} \big[ (-1)^{S_j} \Delta_{i,j}^{\Psi} \big]   \Big\vert .   \label{equ39}
    \end{align}
For any $i\in[n]$ , $j\in[m]$, it is clear that $\Delta_{i,j}^{\Psi} \in [-1, +1]$, and $ (-1)^{S_j} \Delta_{i,j}^{\Psi}$ are $1$-sub-gaussian. Let $ S'_{\mathbb{M}}$ be independent copy of $S_{\mathbb{M}}$. Notice that $\mathbb{E}_{\Delta_{i,j}^{\Psi},S'_j} \big[ (-1)^{S'_j} \Delta_{i,j}^{\Psi} \big] = 0$.  Applying Lemma \ref{lemmaA.5} with $P=P_{\Delta_{i,j}^{\Psi},S_j}$, $Q = P_{\Delta_{i,j}^{\Psi}} P_{S_j}$ and $f(\Delta_{i,j}^{\Psi}, S_j) =(-1)^{S_j} \Delta_{i,j}^{\Psi} $, we have
\begin{align}\label{equ40}
    \Big\vert \mathbb{E}_{\Delta_{i,j}^{\Psi},S_j} \big[ (-1)^{S_j} \Delta_{i,j}^{\Psi} \big] -\mathbb{E}_{\Delta_{i,j}^{\Psi},S'_j} \big[ (-1)^{S'_j} \Delta_{i,j}^{\Psi} \big] \Big\vert \
    = \Big\vert \mathbb{E}_{\Delta_{i,j}^{\Psi},S_j} \big[ (-1)^{S_j} \Delta_{i,j}^{\Psi} \big] \Big\vert 
    \leq \sqrt{2I(\Delta_{i,j}^{\Psi};S_j)}
\end{align}
Substituting (\ref{equ40}) into (\ref{equ39}), we obtain
\begin{equation*}
    \vert \overline{\mathrm{gen}} \vert \leq \frac{1}{nm}\sum_{i,j=1}^{n,m} \sqrt{2I(\Delta_{i,j}^{\Psi};S_j)},
\end{equation*} 
which completes the proof.

\end{proof}

\subsection{ Proof of Theorem \ref{theorem9}}

\begin{restatetheorem}{\ref{theorem9}}
        Assume that the loss function $\ell(\cdot,\cdot,\cdot) \in \{0, 1 \}$. In the interpolating setting, i.e., $\hat{\mathcal{R}} = 0$, then
        \begin{equation*}
            \mathcal{R}_{\mathcal{T}} = \frac{1}{nm}\sum_{i,j=1}^{n,m} \frac{I(\Delta_{i,j}^{\Psi};S_j) }{\log 2} = \frac{1}{nm}\sum_{i,j=1}^{n,m} \frac{I(L_{i,j}^{\Psi}; S_j)}{\log 2} \leq \frac{1}{nm}\sum_{i,j=1}^{n,m} \frac{2I(L_{i,j}^{\Psi_+}; S_j)}{\log 2}.
        \end{equation*}
\end{restatetheorem}
\begin{proof}
    In the interpolating setting with binary losses, i.e., $\ell(\cdot,\cdot,\cdot)\in\{0,1\}$ and $\hat{\mathcal{R}}=0$, we have 
    \begin{align}
        \mathcal{R}_{\mathcal{T}}  = & \frac{1}{nm}\sum_{i,j=1}^{n,m} \mathbb{E}_{L^i_{j}, \tilde{S}_i,S_j}\big[ L^{i,\bar{\tilde{S}}_i}_{j,\bar{S}_j}\big] \nonumber \\
        = & \frac{1}{nm}\sum_{i,j=1}^{n,m} \frac{1}{4} \Big(\mathbb{E}_{L^{i,1}_{j,1}  |\tilde{S}_i = 0, S_j=0} [L^{i,1}_{j,1}]+ \mathbb{E}_{L^{i,0}_{j,0}  |\tilde{S}_i = 1, S_j=1} [L^{i,0}_{j,0} ] \nonumber \\
        & +  \mathbb{E}_{L^{i,0}_{j,1}  |\tilde{S}_i = 1, S_j=0} [L^{i,0}_{j,1} ]+ \mathbb{E}_{L^{i,1}_{j,0}  |\tilde{S}_i = 0, S_j=1} [L^{i,1}_{j,0} ] \Big)\nonumber \\
        = & \frac{1}{nm}\sum_{i,j=1}^{n,m} \frac{P(\Delta_{i,j}^{\Psi}= 1|S_j=0)+ P(\Delta_{i,j}^{\Psi} = -1|S_j=1)}{2}. \label{equ41}
     \end{align}
Notice that the distribution of $L^{i,\bar{\tilde{S}}_i}_{j,\bar{S}_j}$ and $L^{i,\tilde{S}_i}_{j,S_j}$ should be symmetric regardless of the value $\tilde{S}_i,S_j$.  In other words, we have that $P(\Delta_{i,j}^{\Psi}= 1|S_j=0) = P(\Delta_{i,j}^{\Psi} = -1|S_j=1)$, $P(\Delta_{i,j}^{\Psi}= 0|S_j=0)=P(\Delta_{i,j}^{\Psi} = 0|S_j=1)$, and $P(\Delta_{i,j}^{\Psi}= -1|S_j=0) =  P(\Delta_{i,j}^{\Psi} = 1|S_j=1) = 0$. Let $P(\Delta_{i,j}^{\Psi}= 1|S_j=0)  = \alpha_{i,j}$ and $P(\Delta_{i,j}^{\Psi}= 0|S_j=0)  = 1- \alpha_{i,j}$, then
\begin{align*}
    I(\Delta_{i,j}^{\Psi}; S_j) = &  H(\Delta_{i,j}^{\Psi}) - H(\Delta_{i,j}^{\Psi}|S_j) \\
     =& - \frac{\alpha_{i,j}}{2} \log\big(\frac{\alpha_{i,j}}{2} \big) - (1-\alpha_{i,j}) \log\big(1-\alpha_{i,j} \big)  - \frac{\alpha_{i,j}}{2} \log\big(\frac{\alpha_{i,j}}{2} \big) \\
     &+ \alpha_{i,j} \log\big(\alpha_{i,j} \big) + (1-\alpha_{i,j}) \log\big(1-\alpha_{i,j} \big) \\
     = & \alpha_{i,j} \log\big(\alpha_{i,j} \big) -\alpha_{i,j} \log\big(\frac{\alpha_{i,j}}{2} \big)\\
      = & \alpha_{i,j} \log 2.
\end{align*}
Putting $P(\Delta_{i,j}^{\Psi}= 1|S_j=0) = \alpha_{i,j}  = \frac{I(\Delta_{i,j}^{\Psi}; S_j)}{\log 2} $ back into (\ref{equ41}), we have 
\begin{align}
    \mathcal{R}_{\mathcal{T}}  = & \frac{1}{nm}\sum_{i,j=1}^{n,m} \frac{P(\Delta_{i,j}^{\Psi}= 1|S_j=0)+ P(\Delta_{i,j}^{\Psi} = -1|S_j=1)}{2} \nonumber\\
    = & \frac{1}{nm}\sum_{i,j=1}^{n,m} \alpha_{i,j} =   \frac{1}{nm}\sum_{i,j=1}^{n,m}  \frac{I(\Delta_{i,j}^{\Psi}; S_j)}{\log 2}. \label{equ42}
\end{align}
By assuming $\hat{\mathcal{R}}=0$, we know that $P(L_{i,j}^{\Psi_+}=1,L_{i,j}^{\Psi_-}=1)=0$. Therefore, there exists a bijection between $\Delta_{i,j}^{\Psi}$ and $L_{i,j}^{\Psi}$: $\Delta_{i,j}^{\Psi} = 0 \leftrightarrow L_{i,j}^{\Psi}= \{0,0\}$, $\Delta_{i,j}^{\Psi}  = 1 \leftrightarrow L_{i,j}^{\Psi}= \{1,0\}$, and $\Delta_{i,j}^{\Psi}= -1 \leftrightarrow L_{i,j}^{\Psi}= \{0,1\}$. By the data-processing inequality, we know that $I(\Delta_{i,j}^{\Psi};S_j) = I(L_{i,j}^{\Psi}; S_j)$ and
\begin{align*}
  & I(L_{i,j}^{\Psi_+}; S_j) \\
     = & H(L_{i,j}^{\Psi_+}) - H(L_{i,j}^{\Psi_+} | S_j)  \\
    = & -\frac{\alpha_{i,j}}{2} \log(\frac{\alpha_{i,j}}{2}) - \big( 1- \frac{\alpha_{i,j}}{2}\big) \log\big( 1- \frac{\alpha_{i,j}}{2}\big)  +  \frac{\alpha_{i,j}}{2} \log(\alpha_{i,j}) + \frac{1-\alpha_{i,j}}{2} \log\big( 1- \alpha_{i,j}\big) \\
    \geq & -\frac{\alpha_{i,j}}{2} \log(\frac{\alpha_{i,j}}{2}) + \frac{\alpha_{i,j}}{2} \log(\alpha_{i,j}) = \frac{\alpha_{i,j}}{2}  \log 2,
\end{align*}
where the last inequality is due to Jensen'inequality on the convex function $f(x)=(1-x)\log(1-x)$ such that $\frac{f(0)+f(\alpha_{i,j})}{2}\geq f(\frac{\alpha_{i,j}}{2})$. Therefore, we get that for any $i\in[n],j\in[m]$
\begin{align}\label{equ43}
    I(\Delta_{i,j}^{\Psi};S_j) = I(L_{i,j}^{\Psi}; S_j) \leq 2 I(L_{i,j}^{\Psi_+}; S_j) .
\end{align}
Combining (\ref{equ42}) and (\ref{equ43}), this completes the proof.

\end{proof}

\subsection{ Proof of Theorem \ref{theorem10}}

\begin{restatetheorem}{\ref{theorem10}}
    Assume that the loss function $\ell(\cdot,\cdot,\cdot)\in [0,1]$. For any $0 < C_2 < \frac{\log 2}{2}$ and $C_1\geq -\frac{\log(2-e^{C_2})}{2C_2}-1$, 
    \begin{equation*}
        \overline{\mathrm{gen}} \leq C_1 \hat{\mathcal{R}} + \frac{1}{nm}\sum_{i,j=1}^{n,m} \frac{\min\{I(L_{i,j}^{\Psi};S_j ), 2 I(L_{i,j}^{\Psi_+};S_j)\}}{C_2}.
    \end{equation*}
    In the interpolating setting, i.e., $\hat{\mathcal{R}} = 0$, we further have 
    \begin{equation*}
        \mathcal{R}_{\mathcal{T}} \leq \frac{1}{nm}\sum_{i,j=1}^{n,m} \frac{2\min\{I(L_{i,j}^{\Psi};S_j ), 2 I(L_{i,j}^{\Psi_+};S_j)\}}{\log 2}.
    \end{equation*}
\end{restatetheorem}

\begin{proof}
    Note that 
    \begin{align}
        \mathcal{R}_{\mathcal{T}} - (1+C_1)\hat{\mathcal{R}}
        = & \frac{1}{nm}\sum_{i,j=1}^{n,m}  \mathbb{E}_{L_j^i,\tilde{S}_i,S_j} \big[L^{i,\bar{\tilde{S}}_i}_{j,\bar{S}_j} - (1+C_1)L^{i,\tilde{S}_i}_{j,S_j}\big] \nonumber\\
        = & \frac{1}{nm}\sum_{i,j=1}^{n,m}  \mathbb{E}_{L_j^i,\tilde{S}_i,S_j} \bigg[\Big(1+\frac{C_1}{2}\Big)\Big(L^{i,\bar{\tilde{S}}_i}_{j,\bar{S}_j}- L^{i,\tilde{S}_i}_{j,S_j}\Big) - \frac{C_1}{2}L^{i,\bar{\tilde{S}}_i}_{j,\bar{S}_j}- \frac{C_1}{2}L^{i,\tilde{S}_i}_{j,S_j}\bigg] \nonumber\\
        = & \frac{1}{2nm}\sum_{i,j=1}^{n,m}   \Big(\mathbb{E}_{L^{\Psi}_{i,j},S_j}\Big[(-1)^{S_j}\big(2+C_1\big) L_{i,j}^{\Psi_+} - C_1 L_{i,j}^{\Psi_+} \Big] \nonumber\\
        & +  \mathbb{E}_{L^{\Psi}_{i,j},S_j} \Big[ -(-1)^{S_j}\big(2+C_1\big) L_{i,j}^{\Psi_-} - C_1 L_{i,j}^{\Psi_-}\Big] \Big). \label{equ44}
    \end{align}
Since the distribution of the training or test loss  is invariant to supersample variables, the distributions $P_{L_{i,j}^{\Psi_+},S_j}$ and $P_{L_{i,j}^{\Psi_-},S_j}$ have symmetric property, namely, $P_{L_{i,j}^{\Psi_+}|S_j=1} = P_{L_{i,j}^{\Psi_-}|S_j=0}$ and $P_{L_{i,j}^{\Psi_-}|S_j=1} = P_{L_{i,j}^{\Psi_+}|S_j=0}$. Hence, $P_{L_{i,j}^{\Psi_+}}=P_{L_{i,j}^{\Psi_-}}$, which implies that $\mathbb{E}_{L_{i,j}^{\Psi_+}}[C_1 L_{i,j}^{\Psi_+}] = \mathbb{E}_{L_{i,j}^{\Psi_-}}[C_1 L_{i,j}^{\Psi_-}] $ and  
\begin{align*}
    \mathbb{E}_{L^{\Psi}_{i,j},S_j} \Big[(-1)^{S_j}\big(2+C_1\big) L_{i,j}^{\Psi_+} \Big] 
    = & \frac{\mathbb{E}_{L_{i,j}^{\Psi_+}|S_j=0} \Big[\big(2+C_1\big) L_{i,j}^{\Psi_+} \Big] - \mathbb{E}_{L_{i,j}^{\Psi_+}|S_j=1} \Big[\big(2+C_1\big) L_{i,j}^{\Psi_+} \Big]}{2} \\
    = & \frac{\mathbb{E}_{L_{i,j}^{\Psi_-}|S_j=1} \Big[\big(2+C_1\big) L_{i,j}^{\Psi_-}\Big] - \mathbb{E}_{L_{i,j}^{\Psi_-}|S_j=0} \Big[\big(2+C_1\big) L_{i,j}^{\Psi_-} \Big]}{2} \\
    = & -\mathbb{E}_{L^{\Psi}_{i,j},S_j} \Big[(-1)^{S_j}\big(2+C_1\big)L_{i,j}^{\Psi_-} \Big].
\end{align*}
Plugging the above into (\ref{equ44}), we get that 
\begin{align}
    \mathcal{R}_{\mathcal{T}} - (1+C_1)\hat{\mathcal{R}} =  \frac{1}{nm}\sum_{i,j=1}^{n,m}  \mathbb{E}_{L_{i,j}^{\Psi_+},S_j} \Big[(-1)^{S_j}\big(2+C_1\big)L_{i,j}^{\Psi_+} - C_1 L_{i,j}^{\Psi_+} \Big]. \label{equ45}
\end{align}
Let $S'_{\mathbb{M}}$ be an independent copy of $S_{\mathbb{M}}$. Further leveraging Lemma \ref{lemmaA.5} with $P = P_{L_{i,j}^{\Psi_+},S_j}$, $Q = P_{L_{i,j}^{\Psi_+}}P_{S_j}$ and $f(L_{i,j}^{\Psi_+},S_j) = (-1)^{S_j}\big(2+C_1\big)L_{i,j}^{\Psi_+} - C_1 L_{i,j}^{\Psi_+}$, we then have 
\begin{align}
    I(L_{i,j}^{\Psi_+};S_j) = & D(P_{L_{i,j}^{\Psi_+},S_j} \Vert P_{L_{i,j}^{\Psi_+}}P_{S_j} ) \nonumber\\
     \geq & \sup_{C_2>0} \Big\{\mathbb{E}_{L_{i,j}^{\Psi_+},S_j} \Big[  (-1)^{S_j}C_2\big(2+C_1\big)L_{i,j}^{\Psi_+}- C_2C_1 L_{i,j}^{\Psi_+} \Big]  \nonumber\\
     & - \log \mathbb{E}_{L_{i,j}^{\Psi_+},S'_j} \Big[e^{(-1)^{S'_j}C_2\big(2+C_1\big) L_{i,j}^{\Psi_+} - C_2C_1 L_{i,j}^{\Psi_+}} \Big]  \Big\} \nonumber\\
     = & \sup_{C_2>0} \bigg\{\mathbb{E}_{L_{i,j}^{\Psi_+},S_j} \Big[  (-1)^{S_j}C_2\big(2+C_1\big) L_{i,j}^{\Psi_+} - C_2C_1 L_{i,j}^{\Psi_+} \Big] \nonumber\\
     & - \log \frac{\mathbb{E}_{L_{i,j}^{\Psi_+},S'_j} \Big[e^{2C_2 L_{i,j}^{\Psi_+}} + e^{-2C_2(C_1+1) L_{i,j}^{\Psi_+}} \Big]}{2}  \bigg\}  \label{equ46}
\end{align}
Let $\lambda'_{C_1,C_2} = e^{2C_2 L_{i,j}^{\Psi_+}} + e^{-2C_2(C_1+1) L_{i,j}^{\Psi_+}}$. We intend to select the values of $C_1$ and $C_2$ such that the $\log$ term in the above inequality can be less than zero, i.e., $\lambda'_{C_1,C_2} \leq 2$. Since $\lambda'_{C_1,C_2} $ is convex function w.r.t $L_{i,j}^{\Psi_+}$, the maximum value of $\lambda'_{C_1,C_2} $ can be achieved at the endpoints of $L_{i,j}^{\Psi_+}\in[0,1]$. When $L_{i,j}^{\Psi_+} = 0$, we have $\lambda'_{C_1,C_2} = 2$. When $L_{i,j}^{\Psi_+} = 1$, we can select a large enough $C_1$ such that $\lambda'_{C_1,C_2} \leq 2$, which yields $C_2\leq \frac{\log 2}{2}$ and $C_1\geq - \frac{\log(2-e^{2C_2})}{2C_2}-1$. By the above conditions and the inequality (\ref{equ46}), we get that 
\begin{equation*}
    \mathbb{E}_{L_{i,j}^{\Psi_+},S'_j} \Big[e^{2C_2 L_{i,j}^{\Psi_+}} + e^{-2C_2(C_1+1) L_{i,j}^{\Psi_+}} \Big] \leq 2
\end{equation*}
and
\begin{equation}\label{equ47}
    \mathbb{E}_{L_{i,j}^{\Psi_+},S_j} \Big[  (-1)^{S_j}C_2\big(2+C_1\big)L_{i,j}^{\Psi_+} - C_2C_1 L_{i,j}^{\Psi_+} \Big]  \leq I(L_{i,j}^{\Psi_+};S_j).
\end{equation}
Plugging the inequality (\ref{equ47}) into (\ref{equ45}) and employing (\ref{equ43}), we have 
\begin{align*}
    \overline{\mathrm{gen}}= & \mathcal{R}_{\mathcal{T}} - (1+C_1)\hat{\mathcal{R}} + C_1 \hat{\mathcal{R}}\\
    \leq & C_1 \hat{\mathcal{R}} + \frac{1}{nm}\sum_{i,j=1}^{n,m}  \frac{I(L_{i,j}^{\Psi_+};S_j)}{C_2} \\
    \leq & C_1 \hat{\mathcal{R}} + \frac{1}{nm}\sum_{i,j=1}^{n,m}  \frac{\min\{I(L_{i,j}^{\Psi};S_j),2I(L_{i,j}^{\Psi_+};S_j)\}}{C_2}.
\end{align*}
In the interpolating setting, i.e., $\hat{\mathcal{R}} = 0$, by setting $C_2\rightarrow \frac{\log 2}{2}$ and $C_1\rightarrow\infty$, we get that 
\begin{equation*}
    \mathcal{R}_{\mathcal{T}} \leq \frac{1}{nm}\sum_{i,j=1}^{n,m} \frac{2\min\{I(L_{i,j}^{\Psi};S_j), 2 I(L_{i,j}^{\Psi_+};S_j)\}}{\log 2}.
\end{equation*}
This completes the proof.

\end{proof}

\subsection{ Proof of Theorem \ref{theorem12}}
\begin{restatetheorem}{\ref{theorem12}}
    Assume that the loss function $\ell(\cdot,\cdot,\cdot)\in\{0,1\}$ and $\gamma\in(0,1)$. Then, for any $0 < C_2 < \frac{\log 2}{2}$ and $C_1\geq -\frac{\log(2-e^{C_2})}{2C_2\gamma^2}-\frac{1}{\gamma^2}$, we have 
    \begin{equation*}
        \overline{\mathrm{gen}} \leq  C_1 V(\gamma) +  \frac{1}{nm}\sum_{i,j=1}^{n,m}\frac{\min\{I(L_{i,j}^{\Psi};S_j ), 2 I(L_{i,j}^{\Psi_+};S_j)\}}{C_2}.
    \end{equation*}
\end{restatetheorem}

\begin{proof}
    According to the definition of $\gamma$-variance and the property of zero-one loss, we have
\begin{align*}
    V(\gamma) = & \mathbb{E}_{U,T_{\mathbb{M}}^{\mathbb{N}}}\Big[\frac{1}{nm}\sum_{i,j=1}^{nm} \mathbb{E}_{W_\mathbb{N}|U,T_{\mathbb{M}}^{\mathbb{N}}} \big(\ell(U,W_i,Z^i_j) - (1+\gamma) \mathcal{R}(U, T_{\mathbb{M}}^{\mathbb{N}})\big)^2\Big] \\
    = & \mathbb{E}_{U,T_{\mathbb{M}}^{\mathbb{N}}} \Big[\frac{1}{nm}\sum_{i,j=1}^{nm} \mathbb{E}_{W_i|U,T_\mathbb{M}^i} \Big(\ell^2(U,W_i,Z^i_j) - 2(1+\gamma)\ell(U,W_i,Z^i_j) \mathcal{R}(U, T_{\mathbb{M}}^{\mathbb{N}}) \\
    &+ (1+\gamma)^2 \mathcal{R}^2(U, T_{\mathbb{M}}^{\mathbb{N}}) \Big) \Big] \\
    = & \mathbb{E}_{U,T_{\mathbb{M}}^{\mathbb{N}}} \Big[\frac{1}{nm}\sum_{i,j=1}^{nm} \mathbb{E}_{W_i|U,T_\mathbb{M}^i} \ell(U,W_i,Z^i_j) \Big]  - 2(1+\gamma) \mathbb{E}_{U,T_{\mathbb{M}}^{\mathbb{N}}} \big[ \mathcal{R}^2(U, T_{\mathbb{M}}^{\mathbb{N}})\big] \\
    & + (1+\gamma)^2 \mathbb{E}_{U,T_{\mathbb{M}}^{\mathbb{N}}} \big[\mathcal{R}^2(U, T_{\mathbb{M}}^{\mathbb{N}})]  \\
    = & \hat{\mathcal{R}} - (1-\gamma^2) \mathbb{E}_{U,T_{\mathbb{M}}^{\mathbb{N}}} \big[ \mathcal{R}^2(U, T_{\mathbb{M}}^{\mathbb{N}}) \big]. 
\end{align*}
Recall that $\ell(\cdot,\cdot,\cdot)\in\{0,1\}$, we get $\mathcal{R}(U, T_{\mathbb{M}}^{\mathbb{N}})\in[0,1]$, $\mathcal{R}^2(U, T_{\mathbb{M}}^{\mathbb{N}})\leq \mathcal{R}(U, T_{\mathbb{M}}^{\mathbb{N}})$, and 
\begin{align*}
    \overline{\mathrm{gen}} -C_1 V(\gamma) =  &\mathcal{R}_\mathcal{T} - \hat{\mathcal{R}} -C_1 V(\gamma) \\
     = &\mathcal{R}_\mathcal{T} - \hat{\mathcal{R}}  - C_1 \hat{\mathcal{R}} + C_1 (1-\gamma^2)\mathbb{E}_{U,T_{\mathbb{M}}^{\mathbb{N}}} \big[ \mathcal{R}^2(U, T_{\mathbb{M}}^{\mathbb{N}}) \big] \\
    \leq &\mathcal{R}_\mathcal{T} - \hat{\mathcal{R}}  - C_1 \hat{\mathcal{R}} + C_1 (1-\gamma^2)\mathbb{E}_{U,T_{\mathbb{M}}^{\mathbb{N}}} \big[ \mathcal{R}(U, T_{\mathbb{M}}^{\mathbb{N}}) \big] \\
    = &\mathcal{R}_\mathcal{T} - (1+C_1 \gamma^2) \hat{\mathcal{R}}.
\end{align*}
Analogous proof to Theorem \ref{theorem10} with $C_1 = C_1 \gamma^2$ and $C_2=C_2$, we obtain
\begin{equation*}
    \overline{\mathrm{gen}} -C_1 V(\gamma) \leq \frac{1}{nm}\sum_{i,j=1}^{n,m} \frac{\min\{I(L_{i,j}^{\Psi};S_j ), 2 I(L_{i,j}^{\Psi_+};S_j)\}}{C_2},
\end{equation*}
under the conditions that $0 < C_2 < \frac{\log 2}{2}$ and $C_1\geq -\frac{\log(2-e^{C_2})}{2C_2\gamma^2}-\frac{1}{\gamma^2}$. This completes the proof.

\end{proof}

\section{Omitted Proofs [Algorithm-based Bounds]}

\subsection{ Proof of Theorem \ref{theorem4.1}}

\begin{restatetheorem}{\ref{theorem4.1}}
    For the algorithm output $(U, W_{\mathbb{N}})$ after $T$ iterations, we have 
    \begin{equation*}
        I(U,W_{\mathbb{N}}; T_{\mathbb{M}}^{\mathbb{N}}) \leq  \sum_{t=1}^T \frac{1}{2} \log \Big\vert \frac{\eta_t^2}{\sigma^2_t} \mathbb{E}_{U^{t-1},W^{t-1}_{\mathbb{N}}} [\Sigma_t]  + \mathbf{1}_{(\vert I_t \vert +1 ) d   } \Big\vert,
    \end{equation*}
where $\Sigma_t = \mathrm{Cov}_{T_{\mathbb{M}}^{\mathbb{N}},J_{I_t}}[G^t]$.
\end{restatetheorem}
\begin{proof}
    For any $t\in[T]$, by applying the data-processing inequality on the Markov chain $ T_{\mathbb{M}}^{\mathbb{N}}  \rightarrow  ((U^{T-1}, W^{T-1}_{I_T}), \eta_T G^T + N_{T}) $ $\rightarrow (U^{T-1}, W^{T-1}_{I_T}) + \eta_T G^T + N_{T} $ $\rightarrow (U, W_{\mathbb{N}})$, we have 
    \begin{align*}
        I(U, W_{\mathbb{N}} ; T_{\mathbb{M}}^{\mathbb{N}}) \leq & I\big((U^{T-1}, W^{T-1}_{I_T}) + \eta_T G^T + N_{T} ; T_{\mathbb{M}}^{\mathbb{N}} \big) \\
        \leq & I\big((U^{T-1}, W^{T-1}_{I_T}),  \eta_T G^T + N_{T} ; T_{\mathbb{M}}^{\mathbb{N}}\big) \\
        = & I(U^{T-1}, W^{T-1}_{I_T} ; T_{\mathbb{M}}^{\mathbb{N}}) + I(\eta_T G^T + N_{T} ; T_{\mathbb{M}}^{\mathbb{N}}| U^{T-1}, W^{T-1}_{I_T}) \\
        & \cdots \\
        \leq & \sum_{t=1}^T I(\eta_t G^t + N_{t} ; T_{\mathbb{M}}^{\mathbb{N}}| U^{t-1}, W^{t-1}_{I_t}).
    \end{align*} 
    Since $N_{t}$ is independent of $T_{\mathbb{M}}^{\mathbb{N}},J_{I_t}$, we have 
    \begin{align*}
        \mathrm{Cov}_{T_{\mathbb{M}}^{\mathbb{N}},J_{I_t},N_{t}} [\eta_t G^t + N_{t}]= \mathrm{Cov}_{T_{\mathbb{M}}^{\mathbb{N}},J_{I_t}}[\eta_t G^t] + \mathrm{Cov}_{N_{t}}[N_{t}] = \eta^2_t \Sigma_{t} + \sigma^2_t  \mathbf{1}_{(\vert I_t \vert +1 ) d   }.
    \end{align*}
    By applying Lemma \ref{lemmaA.9} with $\Sigma = \eta^2_t \Sigma_{t} + \sigma^2_t  \mathbf{1}_{(\vert I_t \vert +1 ) d   }$, we have 
    \begin{align*}
        & I^{U^{t-1}, W^{t-1}_{I_t}}(\eta_t G^t + N_{t} ; T_{\mathbb{M}}^{\mathbb{N}}) \\
        = & H(\eta_t G^t + N_{t} | U^{t-1}, W^{t-1}_{I_t}) - H(\eta_t G^t + N_{t} |T_{\mathbb{M}}^{\mathbb{N}}, U^{t-1}, W^{t-1}_{I_t}) \\
        \leq & H(\eta_t G^t + N_{t} | U^{t-1}, W^{t-1}_{I_t}) - H(\eta_t G^t + N_{t} |T_{\mathbb{M}}^{\mathbb{N}},J_{I_t}, U^{t-1}, W^{t-1}_{I_t}) \\
        =& H(\eta_t G^t + N_{t} | U^{t-1}, W^{t-1}_{I_t}) - H( N_{t} ) \\
        \leq & \frac{(\vert I_t \vert +1 ) d  }{2} \log(2\pi e) + \frac{1}{2} \log  \Big\vert  \eta^2_t \Sigma_{t} + \sigma^2_t  \mathbf{1}_{(\vert I_t \vert +1 ) d   } \Big\vert - \frac{(\vert I_t \vert +1 ) d  }{2} \log(2\pi e \sigma_t^2) \\
        \leq & \frac{1}{2}  \log \Big\vert  \frac{\eta^2_t}{\sigma_t^2} \Sigma_{t} +  \mathbf{1}_{(\vert I_t \vert +1 ) d   } \Big\vert.
    \end{align*}
    Combining the above estimation and applying Jensen's inequality on the concave log-determinant function, we have 
    \begin{align*}
        I(U, W_{\mathbb{N}} ; T_{\mathbb{M}}^{\mathbb{N}}) \leq &  \sum_{t=1}^T I(\eta_t G^t + N_{t} ; T_{\mathbb{M}}^{\mathbb{N}}| U^{t-1}, W^{t-1}_{I_t}) \\
        = & \sum_{t=1}^T \mathbb{E}_{U^{t-1}, W^{t-1}_{I_t}} \Big[I^{U^{t-1}, W^{t-1}_{I_t}}(\eta_t G^t + N_{t} ; T_{\mathbb{M}}^{\mathbb{N}}) \Big] \\
        \leq & \sum_{t=1}^T \mathbb{E}_{U^{t-1}, W^{t-1}_{I_t}} \Big[  \frac{1}{2}  \log \Big\vert  \frac{\eta^2_t}{\sigma_t^2} \Sigma_{t} + \mathbf{1}_{(\vert I_t \vert +1 ) d   }\Big\vert  \Big] \\
        \leq & \sum_{t=1}^T  \frac{1}{2}  \log \Big\vert  \frac{\eta^2_t}{\sigma_t^2} \mathbb{E}_{U^{t-1}, W^{t-1}_{I_t}} [\Sigma_{t}] + \mathbf{1}_{(\vert I_t \vert +1 ) d   }\Big\vert.
    \end{align*}
Substituting the upper bound above into Theorem \ref{theorem3.1} with $\zeta =n, \xi= m$ yields the desirable generalization bound for meta-learning with joint in-task training and test dataset.
\end{proof}

\subsection{ Proof of Theorem \ref{theorem4.2}}

Before proving Theorem \ref{theorem4.2}, we first establish the MI-based generalization bound for MAML algorithm. Let the meta-dataset $T_{\mathbb{M}}^{\mathbb{N}}$ be randomly divided into inner training and test datasets, denoted by $T_{\mathbb{M},\mathrm{tr}}^{\mathbb{N}}$ and $T_{\mathbb{M},\mathrm{te}}^{\mathbb{N}}$ respectively. Further let $\vert T_{\mathbb{M},\mathrm{tr}}^{\mathbb{N}}\vert = m^{\mathrm{tr}}$, $\vert T_{\mathbb{M},\mathrm{te}}^{\mathbb{N}}\vert = m^{\mathrm{te}}$ and $m^{\mathrm{tr}} + m^{\mathrm{te}} =m$. The empirical meta-risk can then be rewritten by

\begin{equation*}
    \tilde{\mathcal{R}}(U, T_{\mathbb{M}}^{\mathbb{N}}) = \frac{1}{n}\sum_{i=1}^n \mathbb{E}_{P_{W_i|U, T_{\mathbb{M},\mathrm{tr}}^{i}}} \big[  \frac{1}{m^{\mathrm{te}}} \sum_{j=1}^{m^{\mathrm{te}}} \ell(U, W_i,Z^i_{j})\big].
\end{equation*}
The following theorem provides the generalization bound for meta-learning with in-task sample partitioning:
\begin{theorem}[Input-output MI bound] \label{theoremG.1}
        Let $\mathbb{K}$ and $\mathbb{J}=\{\mathbb{J}_{\mathrm{tr}},\mathbb{J}_{\mathrm{te}}\}$ be random subsets of $[n]$ and $[m]$ with sizes $\zeta$ and $\xi$, respectively, where $\vert \mathbb{J}_{\mathrm{tr}}\vert + \vert \mathbb{J}_{\mathrm{tr}}\vert =\xi^{\mathrm{tr}}+\xi^{\mathrm{te}} =\xi $. Assume that $\ell(u,\omega,Z)$, where $Z\sim P_{Z|T}, T\sim P_{\mathcal{T}}$ is $\sigma$-sub-gaussian for all $u\in\mathcal{U}, \omega\in\mathcal{W}$, then
        \begin{equation*}
           \vert \overline{\mathrm{gen}} \vert \leq  \mathbb{E}_{T_{\mathbb{M},\mathrm{tr}}^{\mathbb{N}}, K\sim \mathbb{K}, J \sim \mathbb{J}} \sqrt{\frac{2\sigma^2}{\zeta \xi^{\mathrm{te}}} I^{T_{J_\mathrm{tr}}^{K}}(U,W_{K}; T^{K}_{J_\mathrm{te}})}.
        \end{equation*}
    When $\zeta=n,\xi =m$, we have 
    \begin{equation*}
        \vert \overline{\mathrm{gen}} \vert \leq  \mathbb{E}_{T_{\mathbb{M},\mathrm{tr}}^{\mathbb{N}}} \sqrt{\frac{2\sigma^2}{nm^{\mathrm{te}}} I^{T_{\mathbb{M},\mathrm{tr}}^{\mathbb{N}}}(U,W_{\mathbb{N}}; T_{\mathbb{M},\mathrm{te}}^{\mathbb{N}})}\leq \sqrt{\frac{2\sigma^2}{nm^{\mathrm{te}}} I(U,W_{\mathbb{N}}; T_{\mathbb{M},\mathrm{te}}^{\mathbb{N}}|T_{\mathbb{M},\mathrm{tr}}^{\mathbb{N}})}.
     \end{equation*}
\end{theorem}

\begin{proof}
  Analogous to the proof of Theorem \ref{theorem3.1}, let random subsets $\mathbb{K},\mathbb{J}$ be fixed to $\mathbb{K}=K,\mathbb{J}=J$ with size $\zeta$ and $\xi$, where $J=\{J_{\mathrm{tr}},J_{\mathrm{te}}\}$ and $\vert J_{\mathrm{tr}}\vert + \vert J_{\mathrm{tr}}\vert =\xi^{\mathrm{tr}}+\xi^{\mathrm{te}} =\xi$. Let $T^{K}_{J_{\mathrm{tr}}}=\{T^i_{J_{\mathrm{tr}}}\}_{i\in K}, T^{K}_{J_{\mathrm{te}}}=\{T^i_{J_{\mathrm{te}}}\}_{i\in K} \subseteq T^{\mathbb{N}}_{\mathbb{M}}$ where $T^i_{J_{\mathrm{tr}}} = \{Z^i_j\}_{j\in J_{\mathrm{tr}}}, T^i_{J_{\mathrm{te}}} = \{Z^i_j\}_{j\in J_{\mathrm{te}}}$, $W_K = \{W_i\}_{i\in K}$, and 
    \begin{equation*}
        f(U,W_K, T^{K}_{J_{\mathrm{te}}}) = \frac{1}{\zeta} \sum_{i\in K} \frac{1}{\xi^{\mathrm{te}}}\sum_{j\in J_{\mathrm{te}}}\ell(U, W_i, Z^i_j).
    \end{equation*}
    Let $\bar{U}$ and $\bar{T}^{\mathbb{N}}_{\mathbb{M}}$ be independent copy of $U$ and $T^{\mathbb{N}}_{\mathbb{M}}$. Let us condition on $T^{K}_{J_{\mathrm{tr}}}$. Applying Lemma \ref{lemmaA.5} with $P= P_{\tau_K,T^{K}_{J_{\mathrm{te}}}}P_{U|T^{K}_{J_{\mathrm{te}}}}P_{W_{K}|U,T^{K}_{J_{\mathrm{tr}}}}$, $Q=  P_{\tau_K,T^{K}_{J_{\mathrm{te}}}} P_{\bar{U}}P_{W|U,T^{K}_{J_{\mathrm{tr}}}}P_{\bar{T}^{K}_{J_{\mathrm{te}}}|\tau_K}$, and $f = f(U,W_K, T^{K}_{J_{\mathrm{te}}})$, we get that 
    \begin{align}
        &  D( P_{\tau_K,T^{K}_{J_{\mathrm{te}}}}P_{U|T^{K}_{J_{\mathrm{te}}}}P_{W_{K}|U,T^{K}_{J_{\mathrm{tr}}}} \Vert  P_{\tau_K,T^{K}_{J_{\mathrm{te}}}} P_{\bar{U}}P_{W|U,T^{K}_{J_{\mathrm{tr}}}}P_{\bar{T}^{K}_{J_{\mathrm{te}}}|\tau_K})\nonumber\\
        \geq & \sup_{\lambda}\Big\{\lambda\Big(\mathbb{E}_{\tau_K,T^{K}_{J_{\mathrm{te}}},U,W_{K}} [f(U,W_K, T^{K}_{J_{\mathrm{te}}})] 
         - \mathbb{E}_{\tau_K,T^{K}_{J_{\mathrm{te}}}}\mathbb{E}_{\bar{U}, \bar{T}^{K}_{J_{\mathrm{te}}}|\tau_K}\mathbb{E}_{W|U,T^{K}_{J_{\mathrm{tr}}}} [f(\bar{U},W_K, \bar{T}^{K}_{J_{\mathrm{te}}})] \Big) \nonumber\\
        & -  \log\mathbb{E}_{\tau_K,T^{K}_{J_{\mathrm{te}}}}\mathbb{E}_{\bar{U}, \bar{T}^{K}_{J_{\mathrm{te}}}|\tau_K}\mathbb{E}_{W|U,T^{K}_{J_{\mathrm{tr}}}} \Big[e^{\lambda \big(f(\bar{U},W_K, \bar{T}^{K}_{J_{\mathrm{te}}}) - \mathbb{E}[f(\bar{U},W_K, \bar{T}^{K}_{J_{\mathrm{te}}})]\big)} \Big] \Big\}, \label{equ51}
    \end{align}
    where 
    \begin{align*}
        & D( P_{\tau_K,T^{K}_{J_{\mathrm{te}}}}P_{U|T^{K}_{J_{\mathrm{te}}}}P_{W_{K}|U,T^{K}_{J_{\mathrm{tr}}}} \Vert  P_{\tau_K,T^{K}_{J_{\mathrm{te}}}} P_{\bar{U}}P_{W|U,T^{K}_{J_{\mathrm{tr}}}}P_{\bar{T}^{K}_{J_{\mathrm{te}}}|\tau_K}) \nonumber\\
         =& D(P_{\tau_K,T^{K}_{J_{\mathrm{te}}}}P_{U|T^{K}_{J_{\mathrm{te}}}} \Vert P_{U}P_{\bar{T}^{K}_{J_{\mathrm{te}}}|\tau_K} ) + D(P_{W_K|U,T^{K}_{J_{\mathrm{tr}}}}\Vert P_{W_K|U,\tau_K} | P_{\tau_K,T^{K}_{J_{\mathrm{te}}}}P_{U}) \nonumber\\
        =& I^{T^{K}_{J_{\mathrm{tr}}}}(U;T^{K}_{J_{\mathrm{te}}}) + I(W_K;T^{K}_{J_{\mathrm{te}}}|\tau_K,U)
    \end{align*}
    By the sub-Gaussian property of the loss function, it is clear that the random variable $f(\bar{U},W_K, \bar{T}^{K}_{J_{\mathrm{te}}})$ is $\frac{\sigma}{\sqrt{\zeta \xi^{\mathrm{te}}}}$-sub-gaussian, which implies that 
    \begin{align*}
        \log\mathbb{E}_{\tau_K,T^{K}_{J_{\mathrm{te}}}}\mathbb{E}_{\bar{U}, \bar{T}^{K}_{J_{\mathrm{te}}}|\tau_K}\mathbb{E}_{W|U,T^{K}_{J_{\mathrm{tr}}}} \Big[e^{\lambda \big(f(\bar{U},W_K, \bar{T}^{K}_{J_{\mathrm{te}}}) - \mathbb{E}[f(\bar{U},W_K, \bar{T}^{K}_{J_{\mathrm{te}}})]\big)} \Big] \leq \frac{\lambda^2\sigma^2}{2\zeta \xi^{\mathrm{te}}}.
    \end{align*}
    Putting the above back into (\ref{equ51}), we have
    \begin{align*}
        & I^{T^{K}_{J_{\mathrm{tr}}}}(U;T^{K}_{J_{\mathrm{te}}}) + I(W_K;T^{K}_{J_{\mathrm{te}}}|\tau_K,U) \\
        \geq& \sup_{\lambda}\Big\{\lambda\Big(\mathbb{E}_{\tau_K,T^{K}_{J_{\mathrm{te}}}} [f(U,W_K, T^{K}_{J_{\mathrm{te}}})] - \mathbb{E}_{\tau_K,T^{K}_{J_{\mathrm{te}}}}\mathbb{E}_{\bar{U}, \bar{T}^{K}_{J_{\mathrm{te}}}|\tau_K}\mathbb{E}_{W|U,T^{K}_{J_{\mathrm{tr}}}} [f(\bar{U},W_K, \bar{T}^{K}_{J_{\mathrm{te}}})] \Big) - \frac{\lambda^2\sigma^2}{2\zeta \xi^{\mathrm{te}}} \Big\}. 
    \end{align*}
    Solving $\lambda$ to maximize this inequality, we obtain that 
    \begin{align*}
        &\Big\vert\mathbb{E}_{\tau_K,T^{K}_{J_{\mathrm{te}}},U,W_{K}} [f(U,W_K, T^{K}_{J_{\mathrm{te}}})] - \mathbb{E}_{\tau_K,T^{K}_{J_{\mathrm{te}}}}\mathbb{E}_{\bar{U}, \bar{T}^{K}_{J_{\mathrm{te}}}|\tau_K}\mathbb{E}_{W|U,T^{K}_{J_{\mathrm{tr}}}} [f(\bar{U},W_K, \bar{T}^{K}_{J_{\mathrm{te}}})] \Big\vert \\
        \leq & \sqrt{\frac{2\sigma^2}{\zeta \xi} \Big(I^{T^{K}_{J_{\mathrm{tr}}}}(U;T^{K}_{J_{\mathrm{te}}}) + I(W_K;T^{K}_{J_{\mathrm{te}}}|\tau_K,U)\Big)} \\
        \leq & \sqrt{\frac{2\sigma^2}{\zeta \xi} I^{T^{K}_{J_{\mathrm{tr}}}}(U,W_K;T^{K}_{J_{\mathrm{te}}})}.
    \end{align*}
 Taking expectation over $T_{\mathbb{M},\mathrm{tr}}^{\mathbb{N}}$ and $K,J$ on both sides, and then applying Jensen's inequality on the absolute value function, we have 
    \begin{align}
        &\Big\vert\mathbb{E}_{\tau_\mathbb{N},T_{\mathbb{M}}^{\mathbb{N}},U,W_\mathbb{N}} [f(U,W_\mathbb{N}, T^{\mathbb{N}}_{\mathbb{M},{\mathrm{te}}})] - \mathbb{E}_{\tau_\mathbb{N},T_{\mathbb{M}}^{\mathbb{N}}}\mathbb{E}_{\bar{U}, \bar{T}^{\mathbb{N}}_{\mathbb{M}_{\mathrm{te}}}|\tau_\mathbb{N}}\mathbb{E}_{W|U,T^{\mathbb{N}}_{\mathbb{M},{\mathrm{tr}}}} [f(\bar{U},W_\mathbb{N}, \bar{T}^{\mathbb{N}}_{\mathbb{M},{\mathrm{te}}})] \Big\vert \nonumber \\
        \leq  &\mathbb{E}_{T_{\mathbb{M},\mathrm{tr}}^{\mathbb{N}},K\sim \mathbb{K}, J\sim\mathbb{J}} \sqrt{\frac{2\sigma^2I^{T^{K}_{J_{\mathrm{tr}}}}(U,W_K; T^{K}_{J_{\mathrm{te}}})}{\zeta\xi^{\mathrm{te}}}}. \label{b39}
    \end{align}
 Similar to the proof of the inequality (\ref{equ7}), it is easy to prove that the LSH of the above inequality is equivalent to the absolute value of the meta-generalization gap. This completes the proof.

\end{proof}

\begin{restatetheorem}{\ref{theorem4.2}}
    Let $T_{\mathbb{M}}^{\mathbb{N}}=\{T_{\mathbb{M},\mathrm{tr}}^{\mathbb{N}}, T_{\mathbb{M},\mathrm{te}}^{\mathbb{N}}\}$ consist of separate within-task training and test datasets, where $\vert T_{\mathbb{M},\mathrm{tr}}^{\mathbb{N}}\vert = m^{\mathrm{tr}}$, $\vert T_{\mathbb{M},\mathrm{te}}^{\mathbb{N}}\vert = m^{\mathrm{te}}$ and $m^{\mathrm{tr}} + m^{\mathrm{te}} =m$.   Assume that $\ell(\cdot,\cdot,\cdot)\in [0,1]$, then 
    \begin{align*}
        \vert \overline{\mathrm{gen}} \vert \leq  \frac{1}{\sqrt{nm^{\mathrm{te}}}} \sqrt{\sum_{t=1}^T \log \Big\vert \frac{\beta_t^2}{\sigma^2_t} \mathbb{E}_{U^{t-1}} [\Sigma^{\mathrm{tr}}_t]  + \mathbf{1}_{\vert I_t \vert d   } \Big\vert + \log \Big\vert \frac{\eta_t^2}{\sigma^2_t} \mathbb{E}_{U^{t-1}, W^t_{I_t}} [\Sigma^{\mathrm{te}}_t]  + \mathbf{1}_{d} \Big\vert},
     \end{align*}
     where  $\Sigma^{\mathrm{te}}_t = \mathrm{Cov}_{T_{\mathbb{M},\mathrm{te}}^{\mathbb{N}}, J_{I_t}^{\mathrm{te}}}[-\frac{1}{\vert I_t\vert} \sum_{i\in I_t} \tilde{G}^{\mathrm{te},t}_{W_{i}}]$ and $\Sigma^{\mathrm{tr}}_t = \mathrm{Cov}_{T_{\mathbb{M},\mathrm{tr}}^{\mathbb{N}}, J_{I_t}^{\mathrm{tr}}}[\tilde{G}^{\mathrm{tr},t}_{W_{I_t}}]$.
\end{restatetheorem}

\begin{proof}
    Let the meta-dataset $T_{\mathbb{M}}^{\mathbb{N}}$ be randomly divided into inner training and test datasets, denoted by $ T_{\mathbb{M},\mathrm{tr}}^{\mathbb{N}}$ and $T_{\mathbb{M},\mathrm{te}}^{\mathbb{N}}$ respectively, and let $\tilde{N}_{I_t}^t=\{\tilde{N}_t\}^{\vert I_t\vert}$ be the isotropic Gaussian noise for $t\in[T]$. Leveraging the Markov structure, we have 
    \begin{align*}
        T_{\mathbb{M},\mathrm{tr}}^{\mathbb{N}}\rightarrow  J_{I_T}^{\mathrm{tr}} \rightarrow  (W^{T-1}_{I_T} , \beta_T \tilde{G}_{W_{I_T}}^{\mathrm{tr},T} + \tilde{N}_{I_T}^T)  \rightarrow W^{T}_{I_T}  \\
        T_{\mathbb{M},\mathrm{te}}^{\mathbb{N}} \rightarrow J_{I_T}^{\mathrm{te}} \rightarrow (U^{T-1} , \eta_T \tilde{G}_{U}^T + \tilde{N}_T) \rightarrow U^T,  \\
        T_{\mathbb{M},\mathrm{tr}}^{\mathbb{N}} \rightarrow  (W_{\mathbb{N}} , U) \leftarrow T_{\mathbb{M},\mathrm{te}}^{\mathbb{N}}.
    \end{align*}
By applying the data-processing inequality on the above Markov chain, we obtain that 
\begin{align}
    I(U,W_{\mathbb{N}}; T_{\mathbb{M},\mathrm{te}}^{\mathbb{N}}| T_{\mathbb{M},\mathrm{tr}}^{\mathbb{N}})  \leq &  I\big(U^{T-1} + \eta_T \tilde{G}_{U}^T + \tilde{N}_T ,W^{T}_{I_T}; T_{\mathbb{M},\mathrm{te}}^{\mathbb{N}} | T_{\mathbb{M},\mathrm{tr}}^{\mathbb{N}}   \big) \nonumber\\
    \leq & I\big( U^{T-1} , \eta_T \tilde{G}_{U}^T + \tilde{N}_T , W^{T}_{I_T}; T_{\mathbb{M},\mathrm{te}}^{\mathbb{N}}   |T_{\mathbb{M},\mathrm{tr}}^{\mathbb{N}} \big) \nonumber\\ 
    = & I \big(U^{T-1} , W^{T}_{I_T}; T_{\mathbb{M},\mathrm{te}}^{\mathbb{N}} | T_{\mathbb{M},\mathrm{tr}}^{\mathbb{N}} \big) + I(\eta_T \tilde{G}_{U}^T + \tilde{N}_T ; T_{\mathbb{M},\mathrm{te}}^{\mathbb{N}} | U^{T-1} , W^{T}_{I_T}, T_{\mathbb{M},\mathrm{tr}}^{\mathbb{N}}) \nonumber\\
    = & I\big(U^{T-1} , W^{T-1}_{I_{T}} + \beta_{T} \tilde{G}_{W_{I_T}}^{\mathrm{tr},T} +\tilde{N}^T_{I_T} ; T_{\mathbb{M},\mathrm{te}}^{\mathbb{N}} | T_{\mathbb{M},\mathrm{tr}}^{\mathbb{N}} \big)\nonumber\\
    & + I\big(\eta_T \tilde{G}_{U}^T + \tilde{N}_T ; T_{\mathbb{M},\mathrm{te}}^{\mathbb{N}} | U^{T-1} , W^{T}_{I_T}, T_{\mathbb{M},\mathrm{tr}}^{\mathbb{N}} \big) \nonumber\\
    \leq & I\big(U^{T-1} , W^{T-1}_{I_T} ; T_{\mathbb{M},\mathrm{te}}^{\mathbb{N}} |T_{\mathbb{M},\mathrm{tr}}^{\mathbb{N}} \big) \nonumber\\
    & + I\big( \beta_{T} \tilde{G}_{W_{I_T}}^{\mathrm{tr},T} +\tilde{N}^T_{I_T} ; T_{\mathbb{M},\mathrm{te}}^{\mathbb{N}} | U^{T-1} , W^{T-1}_{I_T}, T_{\mathbb{M},\mathrm{tr}}^{\mathbb{N}} \big) \nonumber\\
    &+ I \big(\eta_T \tilde{G}_{U}^T + \tilde{N}_T ; T_{\mathbb{M},\mathrm{te}}^{\mathbb{N}} | U^{T-1} , W^{T}_{I_T}| T_{\mathbb{M},\mathrm{tr}}^{\mathbb{N}} \big) \nonumber\\
    & \cdots \nonumber\\
    \leq & \sum_{t=1}^T \Big(I \big( \beta_{t} \tilde{G}_{W_{I_t}}^{\mathrm{tr},t} +\tilde{N}^t_{I_t} ; T_{\mathbb{M},\mathrm{te}}^{\mathbb{N}} | U^{t-1} , W^{t-1}_{I_{t}}, T_{\mathbb{M},\mathrm{tr}}^{\mathbb{N}}   \big) \nonumber\\
    & + I \big(\eta_t \tilde{G}_{U}^t + \tilde{N}_t ; T_{\mathbb{M},\mathrm{te}}^{\mathbb{N}} | U^{t-1} , W^{t}_{I_t}, T_{\mathbb{M},\mathrm{tr}}^{\mathbb{N}} \big) \Big). \label{equ48}
\end{align}
Since $\tilde{N}_t$ is independent of $T_{\mathbb{M}}^{\mathbb{N}}$ and $J_{I_T}$, we then have
\begin{equation*}
    \mathrm{Cov}_{T_{\mathbb{M},\mathrm{te}}^{\mathbb{N}}, J_{I_T}^{\mathrm{te}}, \tilde{N}_{t}}[\eta_t \tilde{G}_{U}^t + \tilde{N}_t ] = \eta_t^2 \Sigma^{\mathrm{te}}_t + \sigma_t^2 \mathbf{1}_d,
\end{equation*}
\begin{equation*}
    \mathrm{Cov}_{T_{\mathbb{M},\mathrm{tr}}^{\mathbb{N}}, J_{I_T}^{\mathrm{tr}}, \tilde{N}^t_{I_t}}[\beta_{t} \tilde{G}_{W_{I_T}}^{\mathrm{tr},T} +\tilde{N}^t_{I_t} ] = \beta_t^2 \Sigma^{\mathrm{tr}}_t + \sigma_t^2 \mathbf{1}_{\vert  J_t\vert d}.
\end{equation*}

Further using Lemma \ref{lemmaA.9}, we get 
\begin{align} 
    &I^{U^{t-1} , W^{t-1}_{I_{t}}, T_{\mathbb{M},\mathrm{tr}}^{\mathbb{N}} } \big( \beta_{t} \tilde{G}_{W_{I_t}}^{\mathrm{tr},t} +\tilde{N}^t_{I_t} ; T_{\mathbb{M},\mathrm{te}}^{\mathbb{N}} \big)     \nonumber\\
    = & H(\beta_{t} \tilde{G}_{W_{I_t}}^{\mathrm{tr},t}  + \tilde{N}^t_{I_t} | U^{t-1} , W^{t-1}_{I_{t}} , T_{\mathbb{M},\mathrm{tr}}^{\mathbb{N}}) - H(\beta_{t} \tilde{G}_{W_{I_t}}^{\mathrm{tr},t}  + \tilde{N}^t_{I_t} |   W^{t-1}_{I_{t}} , U^{t-1} , T_{\mathbb{M},\mathrm{tr}}^{\mathbb{N}}, T_{\mathbb{M},\mathrm{te}}^{\mathbb{N}}) \nonumber\\
    \leq & H(\beta_{t} \tilde{G}_{W_{I_t}}^{\mathrm{tr},t}  + \tilde{N}^t_{I_t} | U^{t-1} , W^{t-1}_{I_{t}} , T_{\mathbb{M},\mathrm{tr}}^{\mathbb{N}}) -  H(\beta_{t} \tilde{G}_{W_{I_t}}^{\mathrm{tr},t}  + \tilde{N}^t_{I_t} |   W^{t-1}_{I_{t}} , U^{t-1} , T_{\mathbb{M},\mathrm{tr}}^{\mathbb{N}}, T_{\mathbb{M},\mathrm{te}}^{\mathbb{N}}, J_{I_T}^{\mathrm{tr}}) \nonumber\\
    = &H(\beta_{t} \tilde{G}_{W_{I_t}}^{\mathrm{tr},t}  + \tilde{N}^t_{I_t} | U^{t-1} , W^{t-1}_{I_{t}} , T_{\mathbb{M},\mathrm{tr}}^{\mathbb{N}})  -  H( \tilde{N}^t_{I_t}) \nonumber\\
    \leq & \frac{\vert J_t\vert d}{2} \log(2\pi e) + \frac{1}{2}\log \big\vert \beta_t^2 \Sigma^{\mathrm{tr}}_t + \sigma_t^2 \mathbf{1}_{\vert  J_t \vert d} \big\vert - \frac{\vert J_t\vert d}{2} \log(2 \pi e \sigma_t^2) \nonumber\\
    \leq & \frac{1}{2}\log \bigg\vert \frac{\beta_t^2}{\sigma_t^2} \Sigma^{\mathrm{tr}}_t +  \mathbf{1}_{\vert  J_t\vert d} \bigg\vert, \label{equ49}
\end{align}
and 
\begin{align}
  & I^{U^{t-1} , W^{t}_{I_t}, T_{\mathbb{M},\mathrm{tr}}^{\mathbb{N}} } \big(\eta_t \tilde{G}_{U}^t + \tilde{N}_t ; T_{\mathbb{M},\mathrm{te}}^{\mathbb{N}} \big) \nonumber\\
    \leq & H(\eta_t \tilde{G}_{U}^t + \tilde{N}_t| U^{t-1} , W^{t}_{I_t}, T_{\mathbb{M},\mathrm{tr}}^{\mathbb{N}}) - H(\eta_t \tilde{G}_{U}^t + \tilde{N}_t| U^{t-1} , W^{t}_{I_t}, T_{\mathbb{M},\mathrm{tr}}^{\mathbb{N}}, T_{\mathbb{M},\mathrm{te}}^{\mathbb{N}}, J_{I_t}^{\mathrm{te}}) \nonumber\\
    = & H(\eta_t \tilde{G}_{U}^t + \tilde{N}_t| U^{t-1} , W^{t}_{I_t}, T_{\mathbb{M},\mathrm{tr}}^{\mathbb{N}}) - H(\tilde{N}_t) \nonumber\\
    \leq & \frac{ d}{2} \log(2\pi e) + \frac{1}{2}\log \big\vert \eta_t^2 \Sigma^{\mathrm{te}}_t + \sigma_t^2 \mathbf{1}_{ d} \big\vert - \frac{ d}{2} \log(2 \pi e \sigma_t^2) \nonumber\\
    \leq & \frac{1}{2}\log \bigg\vert \frac{\eta_t^2}{\sigma_t^2} \Sigma^{\mathrm{te}}_t +  \mathbf{1}_{d} \bigg\vert. \label{equ50}
\end{align}
Plugging (\ref{equ49}) and (\ref{equ50}) into (\ref{equ48}), we have 

\begin{align*}
  & I(U,W_{\mathbb{N}}; T_{\mathbb{M},\mathrm{te}}^{\mathbb{N}}| T_{\mathbb{M},\mathrm{tr}}^{\mathbb{N}})  \\
    \leq & \sum_{t=1}^T \Big(I \big( \beta_{t} \tilde{G}_{W_{I_t}}^{\mathrm{tr},t} +\tilde{N}^t_{I_t} ; T_{\mathbb{M},\mathrm{te}}^{\mathbb{N}} | U^{t-1} , W^{t-1}_{I_{t}}, T_{\mathbb{M},\mathrm{tr}}^{\mathbb{N}}   \big)  + I \big(\eta_t \tilde{G}_{U}^t + \tilde{N}_t ; T_{\mathbb{M},\mathrm{te}}^{\mathbb{N}} | U^{t-1} , W^{t}_{I_t}, T_{\mathbb{M},\mathrm{tr}}^{\mathbb{N}} \big) \Big) \\
    = & \sum_{t=1}^T \Big(\mathbb{E}_{U^{t-1} , W^{t-1}_{I_{t}}, T_{\mathbb{M},\mathrm{tr}}^{\mathbb{N}}  }\Big[I^{U^{t-1} , W^{t-1}_{I_{t}}, T_{\mathbb{M},\mathrm{tr}}^{\mathbb{N}}  } \big( \beta_{t} \tilde{G}_{W_{I_t}}^{\mathrm{tr},t} +\tilde{N}^t_{I_t} ; T_{\mathbb{M},\mathrm{te}}^{\mathbb{N}} \big) \Big]   \\
    & + \mathbb{E}_{U^{t-1} , W^{t}_{I_t}, T_{\mathbb{M},\mathrm{tr}}^{\mathbb{N}}} \Big[ I^{U^{t-1} , W^{t}_{I_t}, T_{\mathbb{M},\mathrm{tr}}^{\mathbb{N}}} \big(\eta_t \tilde{G}_{U}^t + \tilde{N}_t ; T_{\mathbb{M},\mathrm{te}}^{\mathbb{N}} \big) \Big] \Big)\\
    \leq &  \sum_{t=1}^T \mathbb{E}_{U^{t-1} , W^{t-1}_{I_{t}}, T_{\mathbb{M},\mathrm{tr}}^{\mathbb{N}}  } \bigg[\frac{1}{2}\log \bigg\vert \frac{\beta_t^2}{\sigma_t^2} \Sigma^{\mathrm{tr}}_t +  \mathbf{1}_{\vert I_t\vert d} \bigg\vert \bigg]  + \sum_{t=1}^T \mathbb{E}_{U^{t-1} , W^{t}_{I_t}, T_{\mathbb{M},\mathrm{tr}}^{\mathbb{N}}} \bigg[\frac{1}{2}\log \bigg\vert \frac{\eta_t^2}{\sigma_t^2} \Sigma^{\mathrm{te}}_t +  \mathbf{1}_{d} \bigg\vert \bigg] \\
    \leq & \sum_{t=1}^T  \bigg[\frac{1}{2}\log \bigg\vert \frac{\beta_t^2}{\sigma_t^2} \mathbb{E}_{U^{t-1} } [\Sigma^{\mathrm{tr}}_t] +  \mathbf{1}_{\vert  I_t\vert d} \bigg\vert \bigg]  + \sum_{t=1}^T  \bigg[\frac{1}{2}\log \bigg\vert \frac{\eta_t^2}{\sigma_t^2} \mathbb{E}_{U^{t-1} , W^t_{I_t}}[\Sigma^{\mathrm{te}}_t] +  \mathbf{1}_{d} \bigg\vert \bigg]. 
\end{align*}
Substituting the above inequality into Theorem \ref{theoremG.1} with $\zeta=n,\xi =m$, this completes the proof.

\end{proof}

\vskip 0.2in
\bibliography{reference}

\end{document}